\documentclass{article}
\PassOptionsToPackage{numbers, compress}{natbib}

\usepackage[final]{neurips_2021}

\usepackage{microtype}
\usepackage{graphicx}
\usepackage{subcaption}
\usepackage{booktabs} 
\usepackage{amsmath}
\usepackage{amsthm}
\usepackage{amssymb}
\usepackage{amsfonts}
\usepackage{enumitem}
\usepackage{multirow}
\usepackage{wrapfig}
\usepackage{hyperref}

\usepackage{textcomp}
\usepackage[utf8]{inputenc} 
\usepackage[T1]{fontenc} 
\usepackage{hyperref}  
\usepackage{url}      
\usepackage{booktabs}  
\usepackage{amsfonts}  
\usepackage{nicefrac}  
\usepackage{microtype} 
\usepackage{xcolor}    

\title{Bayesian Optimization with High-Dimensional Outputs}

\DeclareMathOperator*{\argmax}{arg\,max}
\newtheorem{proposition}{Proposition}
\newtheorem{lemma}{Lemma}

\newcommand{\Ktt}{\mathcal{K}_{\text{joint}}}

\author{
  Wesley J. Maddox \\
  New York University \\
  \texttt{wjm363@nyu.edu}
  \And 
  Maximilian Balandat \\
  Facebook \\
  \texttt{balandat@fb.com}
  \And
  Andrew Gordon Wilson \\
  New York University \\
  \texttt{andrewgw@cims.nyu.edu}
  \And Eytan Bakshy \\
  Facebook \\
    \texttt{eytan@fb.com}
}

\begin{document}

\maketitle

\begin{abstract}
    Bayesian Optimization is a sample-efficient black-box optimization procedure that is typically applied to problems with a small number of independent objectives. However, in practice we often wish to optimize objectives defined over many correlated outcomes (or “tasks”). For example, network operators may want to optimize the coverage of a cell tower network across a dense grid of locations. Similarly, engineers may seek to balance the performance of a robot across dozens of different environments via constrained or robust optimization. However, the Gaussian Process (GP) models typically used as probabilistic surrogates for multi-task Bayesian Optimization scale poorly with the number of outcomes, which greatly limitis their applicability. We devise an efficient technique for exact multi-task GP sampling that combines exploiting Kronecker structure in the covariance matrices with Matheron’s identity, allowing us to perform Bayesian Optimization using exact multi-task GP models with tens of thousands of correlated outputs. In doing so, we achieve substantial improvements in sample efficiency compared to existing approaches that only model aggregate functions of the outcomes. We demonstrate how this unlocks a new class of applications for Bayesian Optimization across a range of tasks in science and engineering, including optimizing interference patterns of an optical interferometer with more than 65,000 outputs. 
\end{abstract}

\section{Introduction}
Many problems in science and engineering involve reasoning about multiple, correlated outputs.
For example, cell towers broadcast signal across an area, and thus signal strength is spatially correlated. In randomized experiments, treatment effects on multiple outcomes are naturally correlated due to shared causal mechanisms. 
Without further knowledge of the internal mechanisms (i.e., in a ``black-box'' setting), Multi-task Gaussian processes (MTGPs) are a natural model for these types of problems as they model the relationship between each output (or ``task'') while maintaining the gold standard predictive capability and uncertainty quantification of Gaussian processes (GPs). 
Many downstream analyses require more of the model than just prediction; they also involve sampling from the posterior distribution to estimate quantities of interest.
For instance, we may be interested in the performance of a complex stock trading strategy that requires modeling different stock prices jointly, and want to characterize its conditional value at risk (CVaR) \citep{rockafellar2000optimization}, which generally requires Monte Carlo (MC) estimation strategies \citep{cakmak2020risk}.
Or, we want to use MTGPs in Bayesian Optimization (BO), a method for sample-efficient optimization of black-box functions. 
Many state of the art BO approaches use MC acquisition functions \citep{wilson2018maxbo,astudillo_bayesian_2019,balandat_botorch_2020}, which require sampling from the posterior distribution over new candidate data points.

Drawing posterior samples from MTGPs means sampling over all tasks and all new data points, which typically scales \textit{multiplicatively} in the number of tasks ($t$) and test data points ($n$), e.g. like $\mathcal{O}(n^3 t^3)$ \citep{shah16correlated,bonilla_multi-task_2007}.
For problems with more than a few tasks,
posterior sampling thus quickly becomes intractable due to the size of the posterior covariance matrix. 
This is especially problematic in the case of many real-world problems that can have hundreds or thousands of correlated outputs that should be modelled jointly in order to achieve the best performance. 

\begin{wrapfigure}{r}{0.5\textwidth}
    \vspace{-0.5cm}
	\centering
	\includegraphics[width=\linewidth,clip,clip,trim=0cm 12.75cm 0cm 0cm]{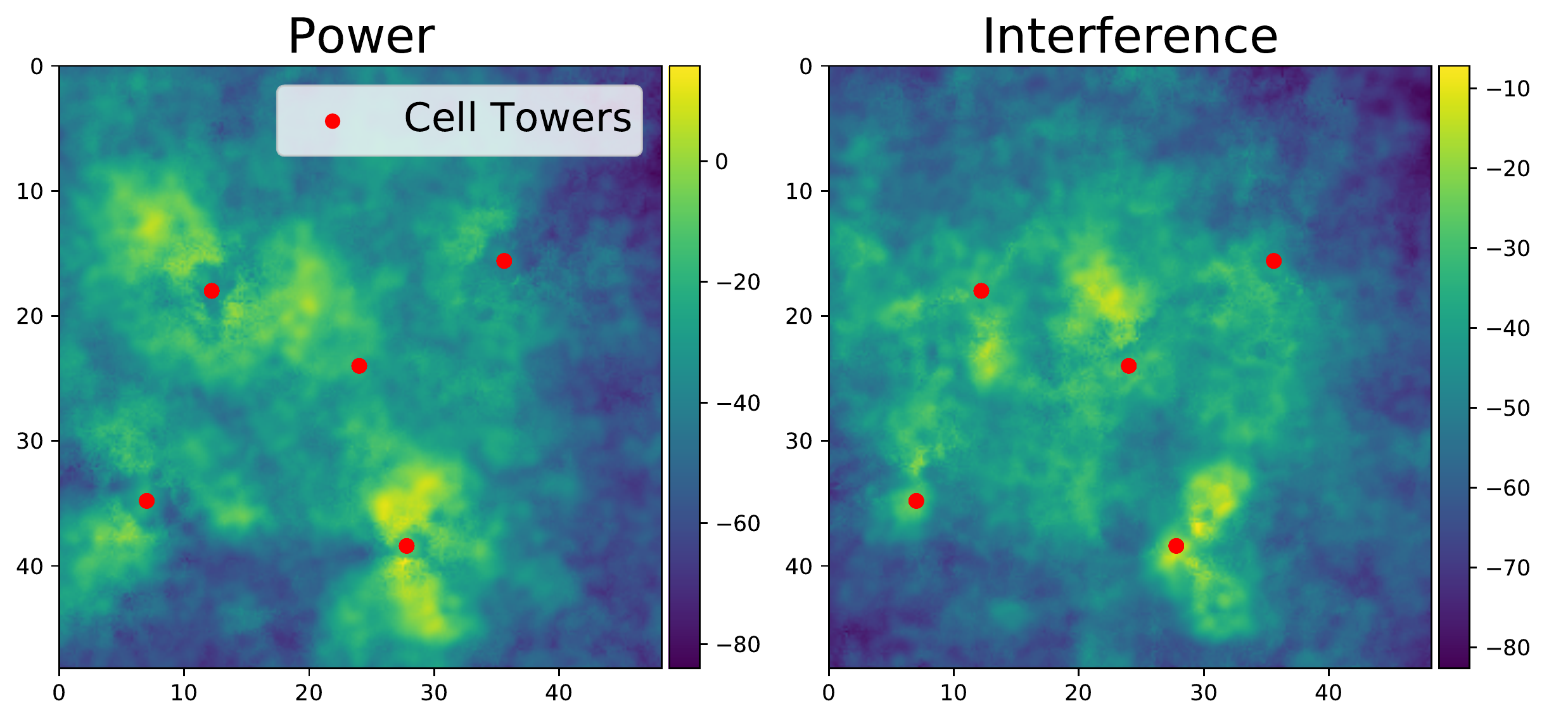}
	\caption{Map of radio signal power and interference for fixed locations of cell towers (red dots). Outcomes vary smoothly with respect to the towers' down-tilt angle and transmission power. Our goal is to optimize statistics of these maps as to maximize the overall signal coverage across an area while minimizing interference. 
	}
	\label{fig:celltower_schematic}
	\vspace{-0.25cm}
\end{wrapfigure}

For instance, the cell tower signal maps in Figure \ref{fig:celltower_schematic} each contain 2,500 outputs (pixels). In this problem, we aim to jointly tune the down-tilt angle and transmission power of the antennas on each cell tower (locations shown in red) to optimize a global coverage quality metric, which is a known function of power and interference at each location \citep{dreifuerst2020optimizing}. Since simulating the power and interference maps given a parameterization is computationally costly, traditionally one might apply BO to optimize the aggregate metric. At its core, this problem is a composite BO problem \citep{astudillo_bayesian_2019,balandat_botorch_2020}, so we expect an approach that models the constituent outcomes at each pixel individually to achieve higher sample efficiency. However, modelling each pixel using existing approaches used for BO is completely intractable in this setting, as we would have to train and sample from a MTGP with 5,000 tasks.

To remedy the poor computational scaling with the number of tasks, we exploit Matheron's rule for sampling from GP posterior distributions \citep{chiles2009geostatistics,wilson_efficiently_2020}.
We derive an efficient method for MTGP sampling that exploits Kronecker structure inherent to the posterior covariance matrices, 
thereby reducing the complexity of sampling from the posterior to become effectively \textit{additive} in the combination of tasks of data points, i.e. $\mathcal{O}(n^3 + t^3),$ as compared to $\mathcal{O}(n^3 t^3).$
Our implementation of Matheron's rule draws from the \emph{exact} posterior distribution and does not require random features or inducing points, unlike decoupled sampling \citep{wilson_efficiently_2020}.
More specifically, our contributions are as follows:
\begin{itemize}[itemsep=3pt,topsep=1pt,leftmargin=15pt]
	\item We propose an \emph{exact} sampling method for multi-task Gaussian processes that has additive time costs in the combination of tasks and data points, rather than multiplicative (Section \ref{sec:matheron}).
	\item We demonstrate empirically how large-scale sampling from MTGPs can aid in challenging multi-objective, constrained, and contextual Bayesian Optimization problems (Section \ref{sec:apps}).
	\item We introduce a method for efficient posterior sampling for the High-Order Gaussian Process (HOGP) model \citep{zhe_scalable_2019}, allowing it to be used for Bayesian Optimization (Section \ref{sec:hogp}). 
	This advance allows us to more efficiently perform BO on high-dimensional outputs such as images --- including optimizing PDEs, optimizing placements of cell towers for cell coverage, and tuning the mirrors of an optical interferometer which optimizes over 65,000 tasks jointly (Section \ref{sec:comp_bo_app}).
\end{itemize}

The rest of the paper is organized as follows: First, in Section \ref{sec:background} we review GPs, MTGPs, and sampling procedures from the posterior in both GPs and MTGPs. In Section \ref{sec:matheron}, we review Matheron's rule for sampling from GP posteriors and explain how to employ it for efficient sampling from MTGP models including the HOGP model. In Section \ref{sec:apps}, we illustrate the utility of our method on a wide suite of problems ranging from constrained BO to the first demonstration of large scale composite BO with the HOGP. 
Please see Appendix \ref{app:limitations} for discussion of the limitations and broader impacts of our work.
Our code is fully integrated into BoTorch, see \url{https://botorch.org/tutorials/composite_bo_with_hogp} and \url{https://botorch.org/tutorials/composite_mtbo} for tutorials.

\section{Background}\label{sec:background}

\subsection{Bayesian Optimization}\label{sec:bo_intro}
In Bayesian Optimization (BO), the goal is to minimize an expensive-to-evaluate black-box function, i.e., finding $\min_{x \in \mathcal{X}} f(x),$ by constructing a \emph{surrogate model} to emulate that function.
Gaussian processes (GPs) are often used as surrogates due to their flexibility and well-calibrated uncertainty estimates. 
BO optimizes an \emph{acquisition function} defined on the predictive distribution of the surrogate model to select the next point(s) to evaluate on the true function. 
These acquisition functions are often written as intractable integrals that are typically evaluated using Monte Carlo (MC) integration \citep{wilson2018maxbo, astudillo_bayesian_2019}.
MC acquisition functions rely on posterior samples from the surrogate model, which should support fast sampling capabilities for efficient optimization \citep{balandat_botorch_2020}.
BO has been applied throughout machine learning, engineering, and the sciences, 
and many extensions to the setting described above exist.
We focus on multi-task BO (MTBO), where $f(x)$ is composed of several correlated tasks \citep{swersky_multi-task_2013,chowdhury2021no}.

There are many sub-classes of MTBO problems: constrained BO uses surrogate models to optimize an objective subject to black-box constraints \citep{gardner2014bayesian,eriksson2020scalable,gelbart2014bayesian}, contextual BO models a single function that varies across different contexts or environments \citep{krause2011contextual,char2019offline,feng_high-dimensional_2020}, multi-objective BO aims to explore a Pareto frontier across several objectives \citep{khan2002multi,knowles2006parego,emmerich2005single,emmerich2011hypervolume,daulton_differentiable_2020}, and composite BO considers the setting of a differentiable objective function defined on the outputs of a vector-valued black-box function   \citep{uhrenholt2019efficient,astudillo_bayesian_2019,balandat_botorch_2020}.
In all of these problems, the setting is similar: several outputs are modelled by the surrogate, whether the output is a constraint, another objective, or a separate context. As the outputs are often correlated, multi-task Gaussian processes, which model the relationships between the outputs in a data-efficient manner, are a natural and common modeling choice.

\subsection{Gaussian Processes}\label{sec:gp_intro}
\textbf{Single Output Gaussian Processes:}
We briefly review single output GPs, see \citet{rasmussen_gaussian_2008} for a more detailed introduction.
We assume that $y = f(x) + \varepsilon$, $f \sim \mathcal{GP}(0, k_\theta(x,x'))$, and $\varepsilon \sim \mathcal{N}(0, \sigma^2)$, where $f$ is the noiseless latent function and $y$ are noisy observations of $f$  with standard deviation $\sigma$. 
$k_\theta(x,x')$ is the kernel with hyperparameters $\theta$ (we will drop the dependence on $\theta$ for simplicity); we use $K_{AB} := k_\theta(A, B)$ to refer to the evaluated kernel function on data points $A$ and $B,$ a matrix which has size $|A| \times |B|.$
The predictive distributions over new data points, $x_{\text{test}},$ is given by the conditional distribution of the Gaussian distribution.
That is, $p(f(x_{\text{test}}) | \mathcal{D}, \theta) = \mathcal{N}(\textcolor{blue}{\mu_{f | \mathcal{D}}^*}, \textcolor{blue}{\Sigma_{f | \mathcal{D}}^*})$, where $\mathcal{D}:=\{X, \mathbf y\}$ is the training dataset of size $n = |X|$ and 
\begin{align}
	\textcolor{blue}{\mu_{f | \mathcal{D}}^*} &= K_{x_{\text{test}} X} (K_{\text{train}} + \sigma^2 I )^{-1} \mathbf y \label{main:eq:exact_pred_mean}, \\
	\textcolor{blue}{\Sigma_{f | \mathcal{D}}^*} &= K_{x_{\text{test}} x_{\text{test}}} - K_{x_{\text{test}} X} (K_{\text{train}} + \sigma^2 I )^{-1}K_{X x_{\text{test}}}, \label{main:eq:exact_pred_var}
\end{align}
with $K_{\text{train}} := K_{XX}$.
For simplicity, we will drop the subscripts $f|\mathcal{D}$ in all future statements. 
Computing the predictive mean $\textcolor{blue}{\mu^*}$ and variance $\textcolor{blue}{\Sigma^*}$ requires $\mathcal{O}(n^3)$ time and $\mathcal{O}(n^2)$ space when using Cholesky decompositions for the linear solves \citep{rasmussen_gaussian_2008}. 
Sampling is usually performed by 
\begin{align}
	f(\mathbf x_{\text{test}}) | (Y=y) = \textcolor{blue}{\mu^*} + (\textcolor{blue}{\Sigma^*})^{1/2} z,
	\label{eq:dist_sampling}
\end{align}
where $z \sim \mathcal{N}(0, I).$
Computing $s$ samples at $n_{\text{test}}$ test points from the predictive distribution costs $\mathcal{O}(n^3 + s n_{\text{test}}^2+ n^2 n_{\text{test}} + n_{\text{test}}^3),$ computed by adding up the cost of all of the matrix vector multiplications (MVMs) and matrix solves.
For fixed $\mathbf x_{\text{test}}$ we can incur the cubic terms only once by re-using Cholesky factorizations of
 $K_{\text{train}} + \sigma^2 I$ and $\textcolor{blue}{\Sigma^*}$ for each sample.

To reduce the time complexity, we can replace all matrix solves with $r<n$ steps of conjugate gradients (CG) and the Cholesky decompositions with rank $r<n$ Lanczos decompositions (an approach called LOVE~\citep{pleiss_constant-time_2018}). 
These change the major scaling from $n^3$ down to $rn^2$ and the overall time complexity to $\mathcal{O}(rn^2 + s r n_{\text{test}}+ r n n_{\text{test}} + r n_{\text{test}}^2)$ \citep{pleiss_constant-time_2018,gardner_gpytorch_2018}.
In general, $r \ll n$ is used and is usually accurate to nearly numerical precision \citep{gardner_gpytorch_2018}.
We provide additional details in Appendix \ref{app:rel_work}.

\textbf{Multi-Output Gaussian Processes:}
One straightforward way of modelling multiple outputs is to consider each output as an independent GP, modelled in batch together, either with shared or independent hyperparameters \citep{rasmussen_gaussian_2008,gardner_gpytorch_2018,eriksson2020scalable}. However, there are two major drawbacks to this approach: (i) the model is not able to model correlations between the outputs, and (ii) if there are many outputs then  maintaining a separate model for each can result in high memory usage and slow inference times. 
To remedy these issues, \citet{higdon2008computer} propose the PCA-GP, using principal component analysis (PCA) to project the outputs to a low-dimensional subspace and then use batch GPs to model the  lower-dimensional outputs.

We consider \textbf{multi-task Gaussian processes} (MTGPs) with the intrinsic co-regionalization model (ICM), which considers the relationship between responses as the product of the data and task features \citep{goovaerts1997geostatistics,bonilla_multi-task_2007,alvarez2012kernels}. We focus on this model due to its popularity and simplicity, leaving a similar derivation of the linear model of co-regionalization to Appendix \ref{appdx:extension_LMC}.
Given inputs $x$ and $x'$ belonging to tasks $i$ and $j$, respectively, the covariance under the ICM is $k([x,i], [x', j]) = k_D(x,x') k_t(i,j)$. 
Given $n$ data points $X$ with the $n$ associated task indices $\mathcal{I}$, the covariance is a Hadamard product of the data kernel and the task kernel, $K_{\text{train}} = K_{XX} \odot K_{\mathcal{I}\mathcal{I}},$ and the response $\mathbf y$ is still of size $n$. We term this implementation of multi-task GPs ``Hadamard MTGPs'' in our experiments. In general, there is no easily exploitable structure in this model. 

If we observe each task at each data point (the so-called  \emph{block design} case), the covariance matrix becomes Kronecker structured, e.g. $K_{\text{train}} = K_{XX}\otimes K_{T},$ where $K_{XX}$ is the data covariance matrix and $K_T$ is the $t \times t$ task covariance matrix between tasks \citep{bonilla_multi-task_2007},  and we now have $nt$ scalar responses.
To simplify our exposition, we assume that $K_T$ is full-rank (this is not required as we can use pseudo-inverses in place of inverses).
Thus, the GP prior is
$\text{vec}(\mathbf y) \sim \mathcal{N}(0, K_{XX} \otimes K_T),$
where $\mathbf y$ is a matrix of size $n \times t,$ and  $\text{vec}(\mathbf y)$ is a vector of shape $nt.$
The GP predictive distribution is given by $p(f^* | x_{\text{test}}, \mathcal{D}) = \mathcal{N}\left(\textcolor{blue}{\mu^*}, \textcolor{blue}{\Sigma^*} \right),$ where
\begin{align}
	\textcolor{blue}{\mu^*} &= (K_{x_{\text{test}}, X} \otimes K_T)(K_{XX} \otimes K_T + \sigma^2 I_{nT})^{-1} \text{vec}(\mathbf y), \nonumber  \\
	\textcolor{blue}{\Sigma^*} &= (K_{x_{\text{test}}, x_{\text{test}}} \otimes K_T) -
	(K_{x_{\text{test}}, X} \otimes K_T)(K_{XX} \otimes K_T + \sigma^2 I_{nT})^{-1} (K_{x_{\text{test}}, X}^\top \otimes K_T). \label{eq:mt_posterior}
\end{align}
The kernel matrix on the training data, $K_{XX} \otimes K_T$, is of size $nt \times nt$, which under standard (Cholesky-based) approaches yields inference cubic and multiplicative in $n$ and $t$, that is $\mathcal{O}(n^3 t^3)$. However, the Kronecker structure can be exploited to compute the posterior mean and variance in $\mathcal{O}(nt(n+t) + n^3 + t^3),$ which is dominated by the individual cubic terms \citep{rougier2008efficient,stegle2011efficient}. 

Sampling from the posterior distribution in  \eqref{eq:mt_posterior} produces additional computational challenges as we must compute a root (e.g. Cholesky) decomposition of $\textcolor{blue}{\Sigma^*},$
which naively costs $\mathcal{O}((n_{\text{test}}t)^3)$ plus an additional 
Cholesky decomposition of $(K_{XX} \otimes K_T + \sigma^2 I)^{-1},$ which similarly costs $\mathcal{O}((nt)^3)$ time \citep{bonilla_multi-task_2007}.
Thus, the time complexity of drawing $s$ samples is \emph{multiplicative} in $n$ and $t,$
$\mathcal{O}((nt)^3 + (n_{\text{test}}t)^3 + s((nt)^2 + (n_{\text{test}}t)^2)).$
Using CG and LOVE reduces the complexity; see Appendix \ref{app:love}.

\textbf{High-Order Gaussian Processes:}
Recently, \citet{zhe_scalable_2019} proposed the high-order Gaussian process (HOGP) model, a MTGP designed for matrix- and tensor-valued outputs.
Given outputs $\mathbf y \in \mathbb{R}^{d_1 \times \cdots \times d_k}$ (e.g. a matrix or tensor), the covariance is the product of the data dimension and each output index ($i_l$ and $i'_l$ respectively) $k([x, i_1, \cdots, i_k], [x', i'_1, \cdots, i'_k]) = k(x,x') k(v_{i_1}, v'_{i'_1}) \cdots  k(v_{i_k}, v'_{i'_k}),$ where $i_1, \cdots, i_k$ are the indices for the output tensor and $v_1, \cdots, v_k$ are latent parameters that are optimized along with the kernel hyper-parameters.
Thus, $K_T$ in the MTGP framework is replaced by a chain of Kronecker products, so that the GP prior is
$\text{vec}(\mathbf y) \sim \mathcal{N}(0, K_{XX} \otimes K_2 \otimes \cdots \otimes K_k).$
Exploiting the Kronecker structure, computation of posterior mean, posterior variance and hyper-parameter learning takes
$\mathcal{O}(n^3 + \sum_{i=2}^k d_i^3 + n \prod_{i=1}^k d_i)$
time as $d_1 = n$. 
\citet{zhe_scalable_2019} demonstrate that the HOGP outperforms PCA-GPs \citep{higdon2008computer}, among other high-dimensional output GP methods, with respect to prediction accuracy. However, their experiments measure only the error of the predictive mean, rather than quantities that use the predictive variance (such as the negative log likelihood or calibration), and they do not provide a way to sample from the posterior distribution.

\subsection{Matheron's Rule for Single-Task Gaussian Processes}\label{sec:matheron_background}
\begin{table*}[t!]
	\centering
	\caption{Time complexities for posterior sampling in single-output, multi-task, and high-order Gaussian Process (HOGP) models. \textcolor{blue}{Time complexities shown in blue} are our contributions that have not yet been considered by the literature. Standard sampling from MTGPs scales multiplicatively in the combination of the number of tasks, $t,$ and the number of data points, $n,$ while using Matheron's rule reduces the combination to effectively become additive in these components. 
	}
	\begin{small}
		\begin{tabular}{c|c|c}
			\toprule Model & \multicolumn{1}{c}{Distributional (Standard) (Eq. \ref{eq:dist_sampling})} & \multicolumn{1}{|c}{\textbf{With Matheron's rule  (Eq. \ref{eq:matheron_gp})}} \\\toprule
			Single-Output & $\mathcal{O}(n^3 + n_{\text{test}}^3)$& $\mathcal{O}(n^3 + n_{\text{test}}^3)$  \\ \midrule
			Multi-Task & $\mathcal{O}((n^3 + n_{\text{test}}^3)t^3)$
			 &
			$\textcolor{blue}{\mathcal{O}((n^3 + n_{\text{test}}^3) + t^3)}$  \\\midrule
			HOGP &$\mathcal{O}((n^3 + n_{\text{test}}^3)\prod_{i=2}^d d_i^3)$ & 
			$\textcolor{blue}{\mathcal{O}((n^3 + n_{\text{test}}^3) + \sum_{i=2}^k d_i^3)}$ \\\bottomrule
		\end{tabular}
	\end{small}
	\label{tab:compute}
\end{table*}

Matheron's rule provides an alternative way of sampling from GP posterior distributions: rather than decomposing the GP predictive covariance for sampling, one can jointly sample from the prior distribution over both train and test points and then update the training samples using the observed data. 
Matheron's rule is well known in the geostatistics literature where it is used for ``prediction by conditional simulation" \citep{chiles2005prediction,chiles2009geostatistics}.
\citet{wilson_efficiently_2020} revitalized interest in Matheron's rule within the machine learning community by developing a decoupled sampling approach that exploits Matheron's rule to use both random Fourier features (RFFs) and inducing point approaches in the context of sampling from the posterior in single task GPs.
\citet{wilson2020pathwise} extended decoupled sampling to use combinations of RFFs, inducing points, and iterative methods. 
They applied these approaches to approximate Thompson sampling, simulations of dynamical systems, and deep GPs.

Matheron's rule states that if two random variables, $X$ and $Y$, are jointly Gaussian, then 
\begin{align*}
	X | (Y = y) \overset{d}{=} X + \text{Cov}(X, Y)\text{Cov}(Y,Y)^{-1}(y - Y),
\end{align*}
where $\text{Cov}(a, b)$ is the covariance of $a$ and $b$ and $\overset{d}{=}$ denotes equality in distribution \citep{doucet_note_2010,wilson_efficiently_2020}.
The identity can easily be shown by computing the mean and variance of both sides.
Following \citet{wilson_efficiently_2020}, we can use this identity to draw posterior samples 
\begin{align}
	f^* | (Y + \epsilon = y) &\overset{d}{=} f^* + 
	K_{x_{\text{test}X}} 
	(K_{XX} + \sigma^2 I)^{-1}(y - Y - \epsilon),
	\label{eq:matheron_gp}
\end{align}
where $\epsilon \sim \mathcal{N}(0, \sigma^2 I),$ $f^*$ is the random variable and $f^* | (Y +\epsilon= y)$ is the conditional random variable we are interested in drawing samples from. 
To implement this procedure, we first draw samples from the joint prior (of size $n + n_{\text{test}}$):
\begin{align}
	(f, Y) \sim \mathcal{N}\left(0, \left(\begin{array}{cc}
		K_{XX} & K_{x_{\text{test}}X}  \\
		K_{Xx_{\text{test}}} & K_{x_{\text{test}}x_{\text{test}}}
	\end{array} \right)\right). 
	\label{eq:joint_prior}
\end{align}
For shorthand, we denote the joint covariance matrix in~\eqref{eq:joint_prior} by $\Ktt$. 
We next sample $\bar{\epsilon} \sim \mathcal{N}(0, \sigma^2 I)$ and compute:
$\bar{f} = f + K_{x_{\text{test}} X}(K_{XX} + \sigma^2 I)^{-1}(y - Y - \bar\epsilon)$. 
The solve is against a matrix of size $n \times n$ so that $\bar{f}$ is our realization of the random variable $f^* | (Y +\epsilon = y).$
The total time requirements are then $\mathcal{O}((n + n_{\text{test}})^3 + n^3),$ which is slightly slower than $\mathcal{O}(n_{\text{test}}^3 + n^3).$
Thus, sampling from the single-task GP posterior using~\eqref{eq:matheron_gp} is slower than the standard method based on~\eqref{main:eq:exact_pred_mean} and~\eqref{main:eq:exact_pred_var}.

\section{Matheron's Rule for Multi-Task Sampling }\label{sec:matheron}
While the time complexity of Matheron-based posterior sampling is inferior for single-task GPs, the opposite holds true for multi-task GPs, provided that one exploits the special structure in the covariance matrices. 
In the following, we demonstrate the advantages in using Matheron-based MTGP sampling in Section \ref{sec:extension}, extend it to the HOGP \citep{zhe_scalable_2019} in Section \ref{sec:hogp}, and identify further advantages of it for Bayesian Optimization in Section \ref{sec:adv_bo}. This approach allows for further pre-computations than distributional sampling and that it maintains the same convergence results.

\subsection{Extending Matheron's Rule for Multi-task GPs}\label{sec:extension}
The primary bottleneck that we wish to avoid is the multiplicative scaling in $n$ (or $n_{\text{test}}$) and $t.$
Ideally, we hope to achieve posterior sampling that is additive in $n$ and $t$, similar to how Kronecker matrix vector multiplication is additive in its components.
Unlike in the single task case, Matheron's rule can substantially reduce the time and memory complexity sampling for MTGPs. For brevity we limit our presentation here to the core ideas, please see Appendix~\ref{app:methods} for further discussion of the implementation, as well as Appendix~\ref{app:kron} for a description of the Kronecker identities we exploit.

The covariance $\Ktt$ in~\eqref{eq:joint_prior} is structured as the Kronecker product of the joint test-train covariance matrix and the inter-task covariance; that is,
$\Ktt = K_{(X,x_{\text{test}}),(X,x_{\text{test}})} \otimes K_T,$ where $K_{(X,x_{\text{test}}),(X,x_{\text{test}})}$ appends $n_{\text{test}}$ rows and columns to the joint training data covariance matrix $K_{XX}$.
To sample from the prior distribution, we need to compute a root decomposition of $\Ktt$.

We assume that we have pre-computed $RR^\top = K_{XX}$ with either a Cholesky decompostion ($\mathcal{O}(n^3)$ time) or a Lanczos decomposition ($\mathcal{O}(rn^2)$ time). 
Then, we update $R$ to compute $\tilde R \tilde R^\top \approx K_{(x_{\text{test}}, X),(x_{\text{test}}, X)}$. 
\citet{chevalier2015fast} used a similar strategy to update samples in the context of single task kriging.
Following \citet[Prop. 2]{jiang_efficient_2020}, computing $\tilde R$ from $R$ costs 
 $\mathcal{O}(r n_{\text{test}} n + r n_{\text{test}}^2 )$ time, dominated by the $r n_{\text{test}} n$ terms for small $n_{\text{test}}$.
 Using a Cholesky decomposition, this is $\mathcal{O}(n_{\text{test}} n^2)$ time following the same procedure.
We then have
\begin{align}
    \Ktt = K_{(x_{\text{test}}X),(x_{\text{test}}X)} \otimes K_T = (\tilde R \otimes L) (\tilde R \otimes L)^\top,
\end{align}
where $LL^\top = K_T.$
To use the root to sample from the joint random variables, $(f,Y),$ we need to compute only $(f,Y) = (\tilde R^\top \otimes L^\top)z, $ where $z\sim \mathcal{N}(0, I_{nt});$ this computation is a Kronecker matrix vector multiplication (MVM), which costs just $\mathcal{O}(nt(n+t))$ time.
Thus, the overall cost of sampling to compute the joint random variables is $\mathcal{O}(n^3 + t^3 + nt(n+t) + n_{\text{test}} n^2)$ time, reduced to 
$\mathcal{O}(rn^2 + rt^2 + rnt + r^2 t + rn_{\text{test}}n )$
if using Lanczos decompositions throughout.
$L$ only needs to be computed once, so samples at new test points only require re-computing $\tilde R$ and further MVMs.

Computing the solve $w = (K_{XX} \otimes K_T + \sigma^2 I_{nT})^{-1}(y - Y - \epsilon)$ takes $\mathcal{O}(n^3 + t^3 + nt(n+t))$ time using eigen-decompositions and Kronecker MVMs. The cost of the eigen-decomposition dominates the Kronecker MVM cost so the major cost in this is $\mathcal{O}(n^3 + t^3)$. 
Finally, there is one more Kronecker MVM $(K_{x_{\text{test}},X} \otimes K_T) w$ which takes $\mathcal{O}(n t^2 + n_\text{test} n t)$ time.
We then only need to reshape $\bar f$ to be of the correct size. 

Therefore, the total time complexity of using Matheron's rule to sample from the MTGP posterior is $\mathcal{O}(n^3 + t^3)$ for small $n_\text{test}$, as the cubic terms dominate due to the Cholesky and eigen-decompositions.
We show the improvements from using Matheron's rule in Table~\ref{tab:compute}, where the dominating terms are 
now \emph{additive in the combination of the tasks and the number of data points, rather than multiplicative.}
Memory complexities, which are also much reduced, are provided in Table~\ref{tab:memory} in Appendix~\ref{app:methods}. 
Finally, we emphasize that sampling in this manner is \emph{exact} up to floating point precision as the solves we use are all exact by virture of being computed with eigen-decompositions.

\subsection{Extension to HOGP}\label{sec:hogp}
The HOGP model can be seen as a special case of the general procedure for sampling MTGPs.
We replace $K_T$ with kernels across each tensor dimension of the response, so that $K_T = \otimes_{i=2}^k K_i$. Therefore, the time complexities for sampling go from a cubic dependence, $n^3 \prod_{i=2}^k d_i^3$,  to a $(n \prod_{i=1}^k d_i ) + (n^3 + \sum_{i=2}^k d_i^3)$ dependence. The latter will usually be dominated by the additive terms for $k \lesssim 5$, as generally $n$ is much larger than the tensor sizes, hence their product will generally be less than $n^2$.
Prior to this work, sampling from the posterior was infeasible for all but the smallest HOGP models due to the time complexity.
By using Matheron's rule, we can sample from HOGP posterior distributions over large tensors, unlocking the model for a much wider range of applications such as BO with composite objectives computed on high-dimensional simulator outputs.

\subsection{Usage in Bayesian Optimization}\label{sec:adv_bo}
The primary usage of efficient multi-task posterior sampling that we consider is that of Bayesian optimization.
Here, we want to use and optimize Monte Carlo acquisition functions that require many samples from the posterior of the surrogate model.

At a high level, Bayesian optimization loops look like the following procedure that we repeat after initializing several random points. 
First, we fit MTGPs by maximizing the log marginal likelihood (see Appendix \ref{app:mtgps}) 
Then, we draw many posterior samples from the MTGP in the context of computing MC acquisition functions, e.g. \eqref{eq:app:convergence:acq_mc}.
We use gradient-based optimization to find the points $x$ which optimize the acquisition $\hat\alpha_N(x; y).$
After choosing these points, $x_{\text{cand}},$ we then query the function, returning $y_{\text{cand}}$ and updating our data with these points.
Finally, we continue back to the top of the loop, and re-train the MTGP.

\paragraph{Convergence Results:}

The convergence guarantees for Sample Average Approximation (SAA) of MC acquisition functions from \citet{balandat_botorch_2020} still apply if posterior samples are drawn via Matheron's rule.  Consider acquisition functions of the form $\alpha(x; y) = \mathbb{E} \bigl[ h(f(x)) \mid Y=y\bigr]$ with $h: \mathbb{R}^{n_{\text{test}} \times t} \rightarrow \mathbb{R}$, and their MC approximation $\hat\alpha_N(x; y) := \frac{1}{N} \sum_{i=1}^N  a(g(f^*_i)))$, where $f^*_i$ are i.i.d. samples drawn from the model posterior at $x \in \mathbb{R}^{n_\text{test}}$ using Matheron's rule. Then, under sufficient regularity conditions, the optimizer $\argmax_{x} \hat\alpha_N(x; y)$ converges to the true optimizer $\argmax_{x} \alpha(x; y)$ almost surely as $N\rightarrow \infty$. For a more formal statement and proof of this result see Appendix~\ref{app:convergence}.

\section{Experiments}\label{sec:apps}
We first demonstrate the computational efficiencies gained by posterior sampling using Matheron's rule. While this contribution is much more broadly useful, here we focus on the benefits it provides for BO.   
Namely, we show that accounting for correlations between outcomes in the model improves performance in multi-objective and large-scale constrained optimization problems. 
Finally, we perform composite BO with tens of thousands of tasks using HOGPs with Matheron-based sampling. 
Additional experiments on contextual policy optimization are given in Appendix~\ref{app:cbo}.
All plotted lines represent the mean over repeated trials with shading represent two standard deviations of the mean across trials.

\subsection{Drawing Samples from Multi-Task Models}

\begin{figure*}[t!]
	\centering
	\begin{subfigure}{0.24\textwidth}
		\centering
		\includegraphics[width=\linewidth]{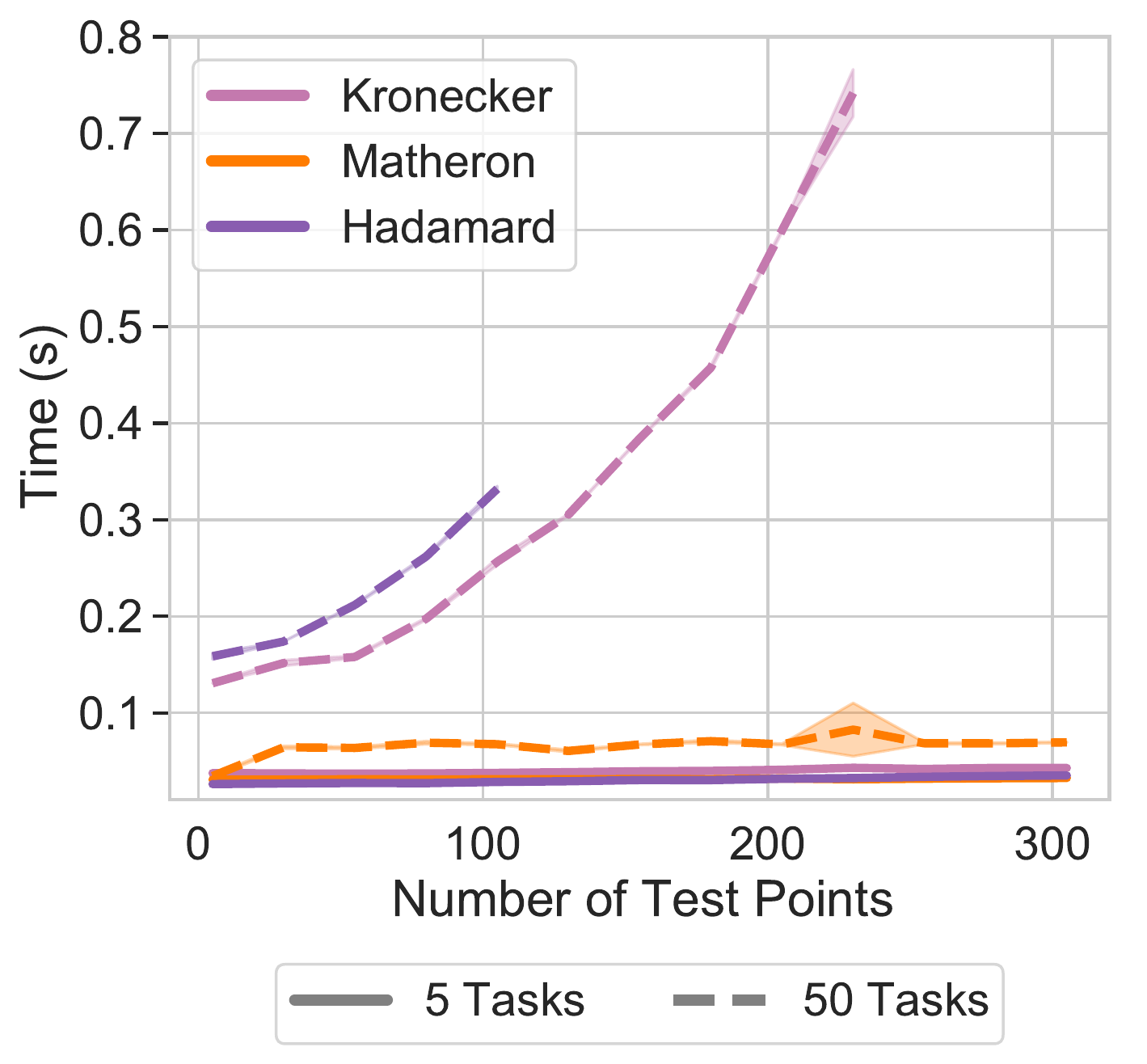}
		\caption{Test points, GPU}
		\label{fig:mt_speedup_test}
	\end{subfigure}
	\begin{subfigure}{0.24\textwidth}
		\centering
		\includegraphics[width=\linewidth]{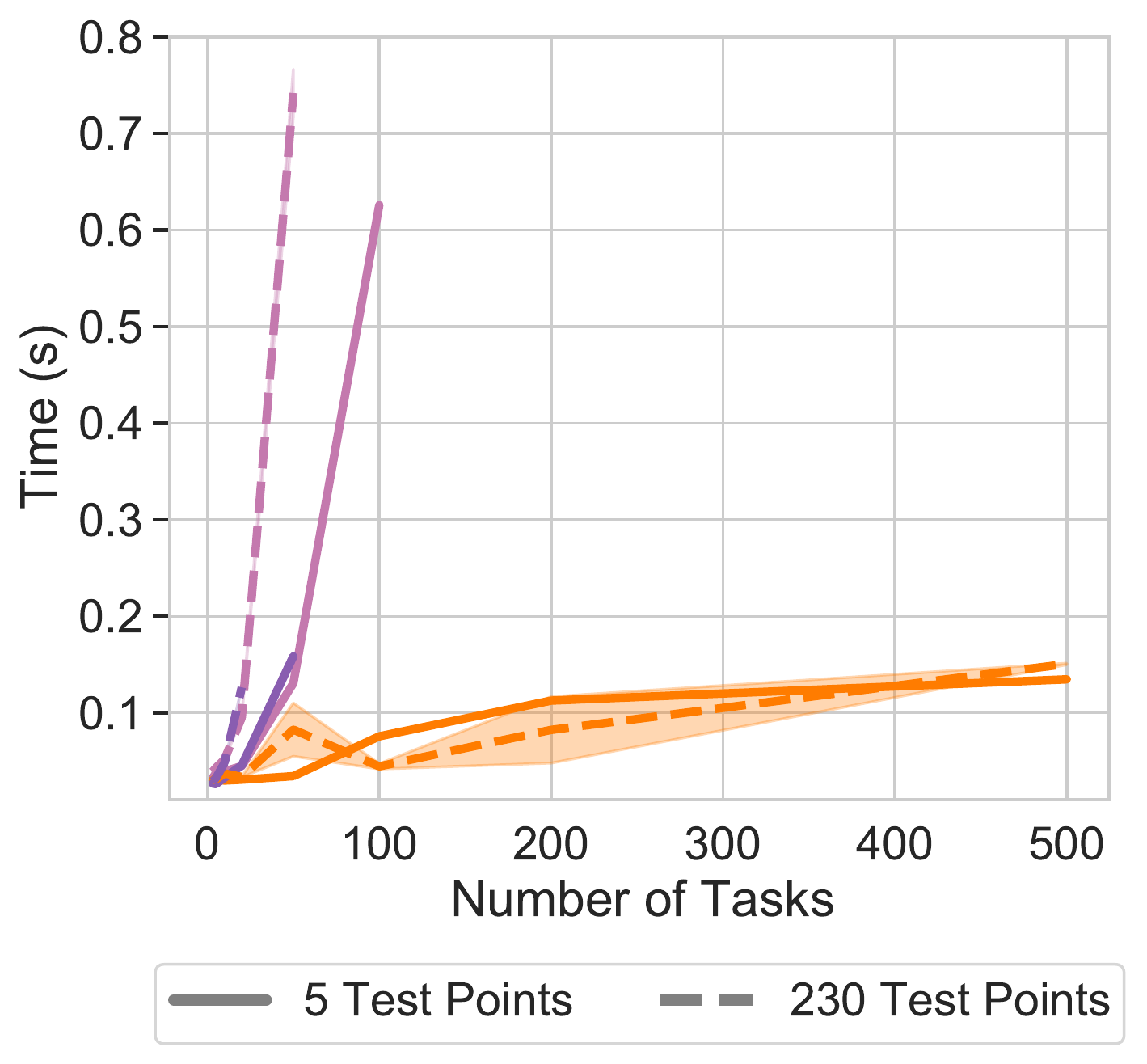}
		\caption{Tasks, GPU}
		\label{fig:mt_speedup_tasks}
	\end{subfigure}
	\begin{subfigure}{0.24\textwidth}
		\centering
		\includegraphics[width=\linewidth]{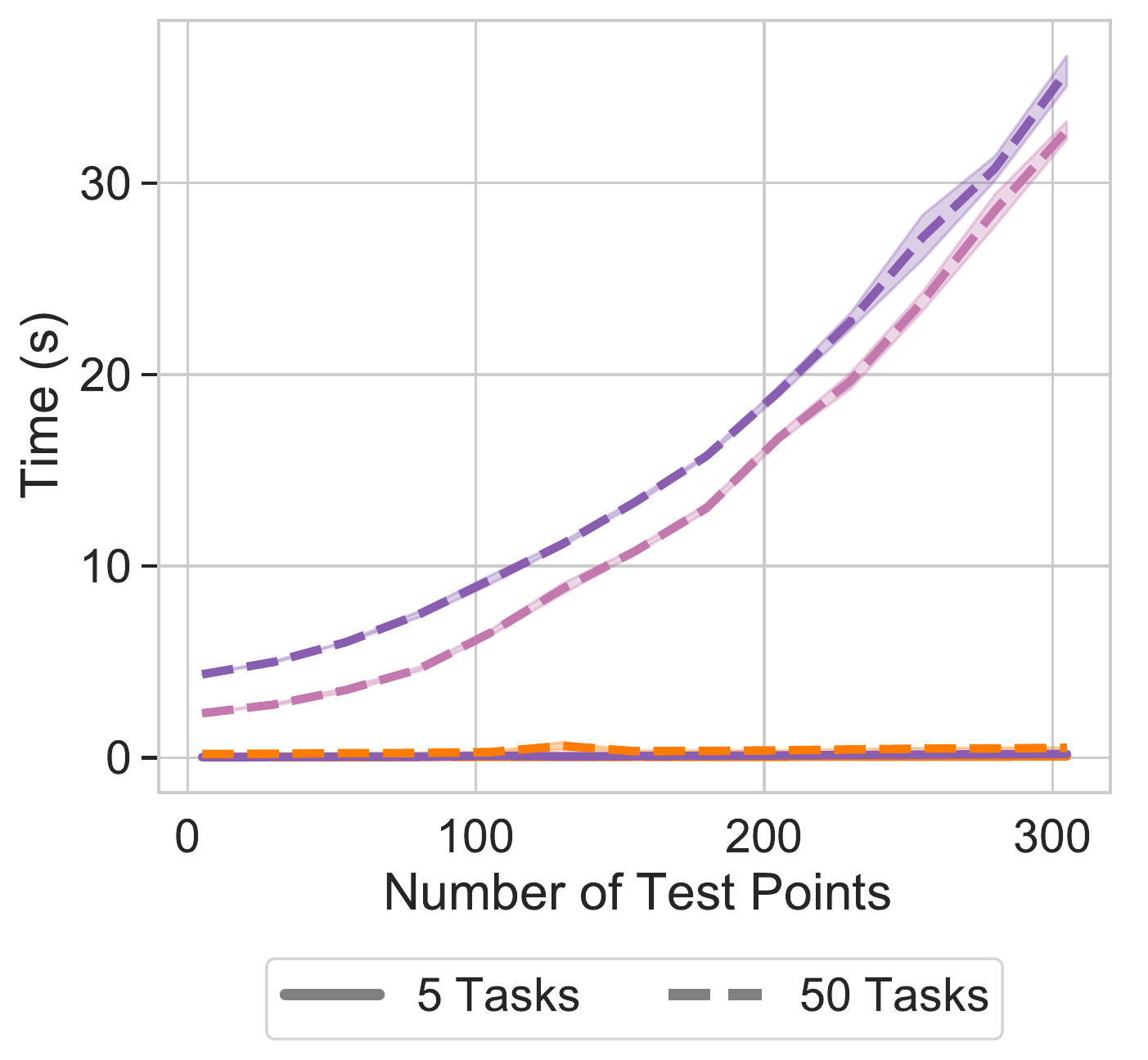}
		\caption{Test points, CPU}
		\label{fig:mt_speedup_test_cpu}
	\end{subfigure}
	\begin{subfigure}{0.24\textwidth}
		\centering
		\includegraphics[width=\linewidth]{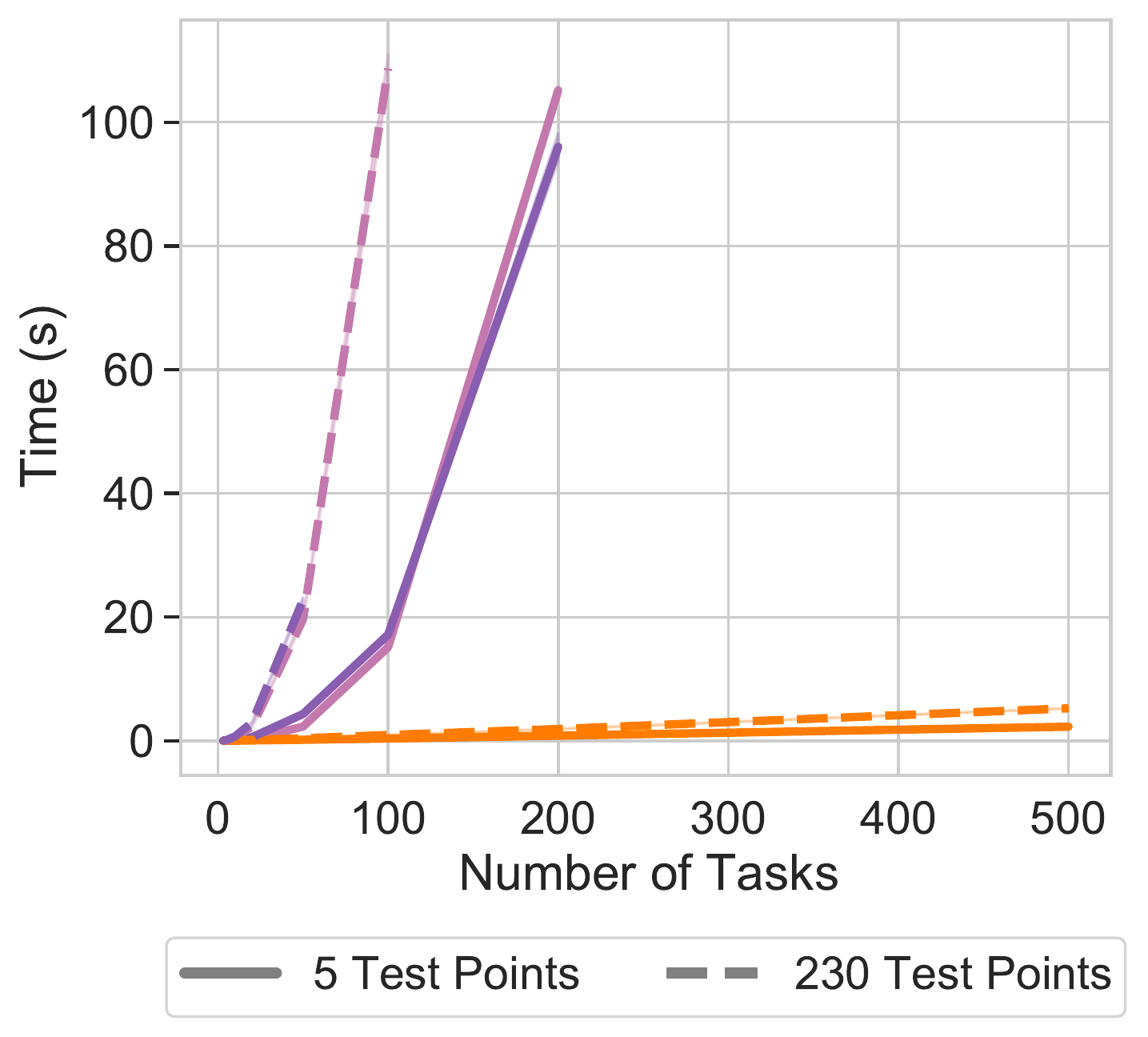}
		\caption{Tasks, CPU}
		\label{fig:mt_speedup_tasks_cpu}
	\end{subfigure}
	
	\caption{
		Timings for distributional sampling with Hadamard and Kronecker MTGPs as well as Matheron's rule sampling for a MTGPs as the number of test points vary for fixed tasks \textbf{(a,c)} and as the number of tasks vary for fixed test points \textbf{(b,d)} on a single Tesla V100 GPU \textbf{(a,b)} and on a single CPU \textbf{(c,d)}.
		The multiplicative scaling of the number of tasks and data points creates significant timing and memory overhead, causing the Kronecker and Hadamard implementations to run out of memory very quickly for all but the smallest numbers of tasks and data points, whereas sampling using Matheron's rule is efficient even in the many-task large-data regime. The plots show mean and two standard errors over $10$ trials on the GPU, and $6$ trials on the CPU.
	}
	\label{fig:mt_speedup}  
\end{figure*}
\begin{wrapfigure}{r}{0.5\textwidth}
	\centering
	\begin{subfigure}{0.49\textwidth}
		\centering
		\includegraphics[width=\linewidth,clip,clip,trim=0cm 0cm 0cm 8cm]{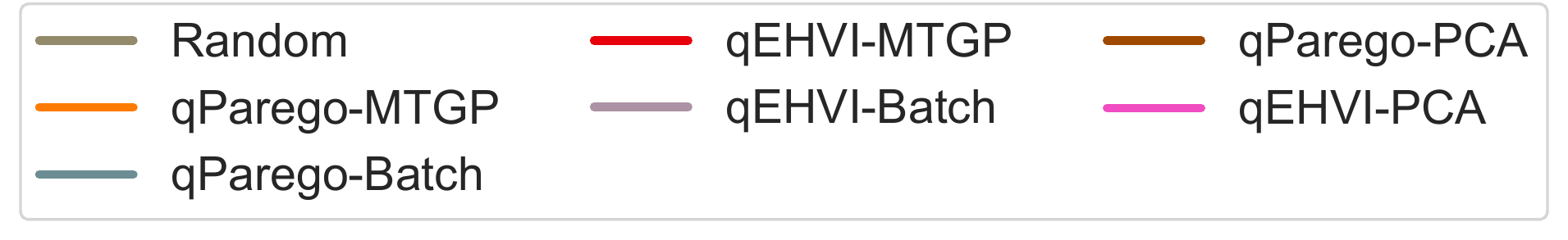}
	\end{subfigure}
	
	\begin{subfigure}{0.23\textwidth}
		\centering
		\includegraphics[width=\linewidth]{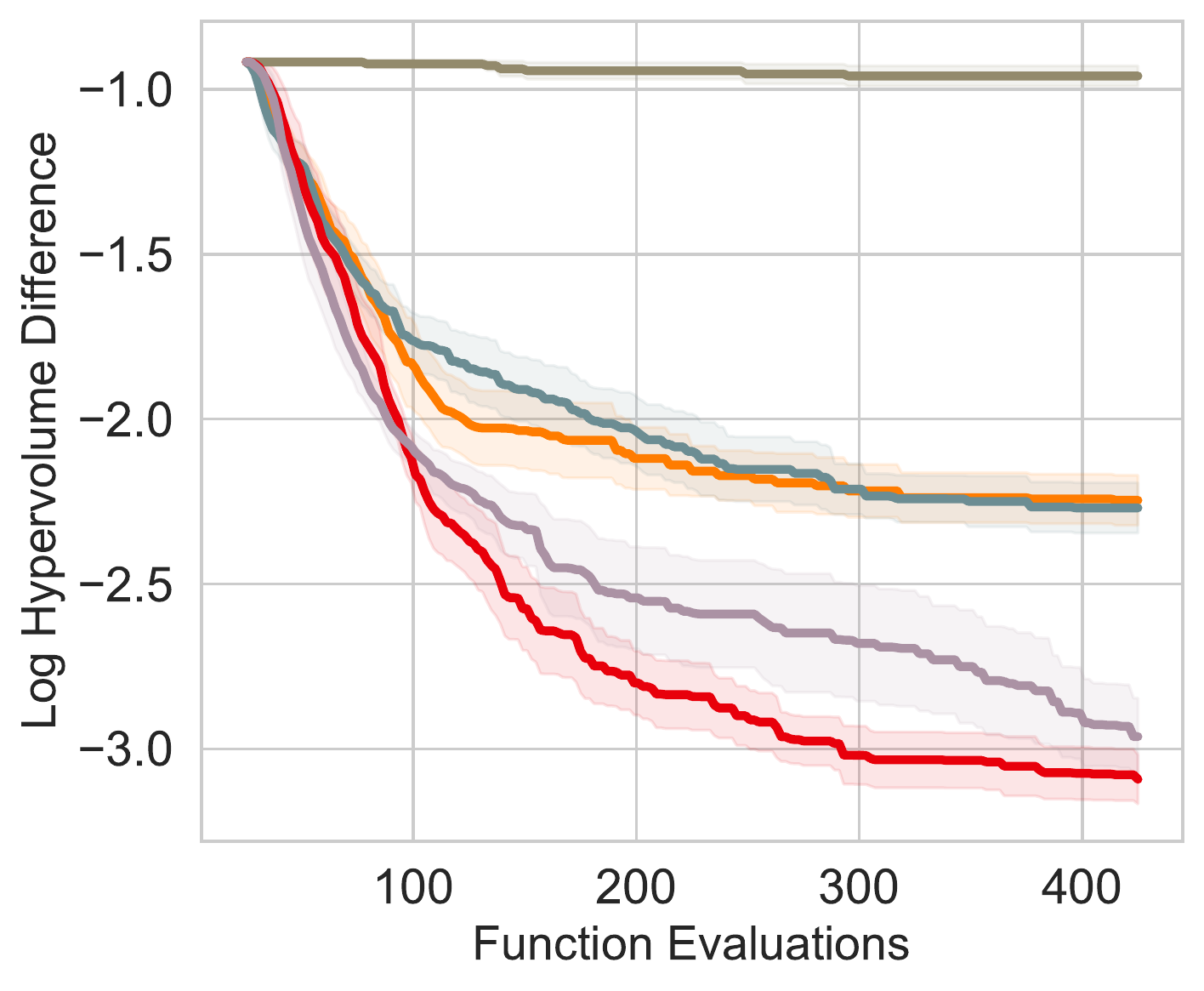}
		\caption{C2DTLZ2, $q=2$}
		\label{fig:c2dtl}
	\end{subfigure}
	\begin{subfigure}{0.23\textwidth}
		\centering
		\includegraphics[width=\linewidth]{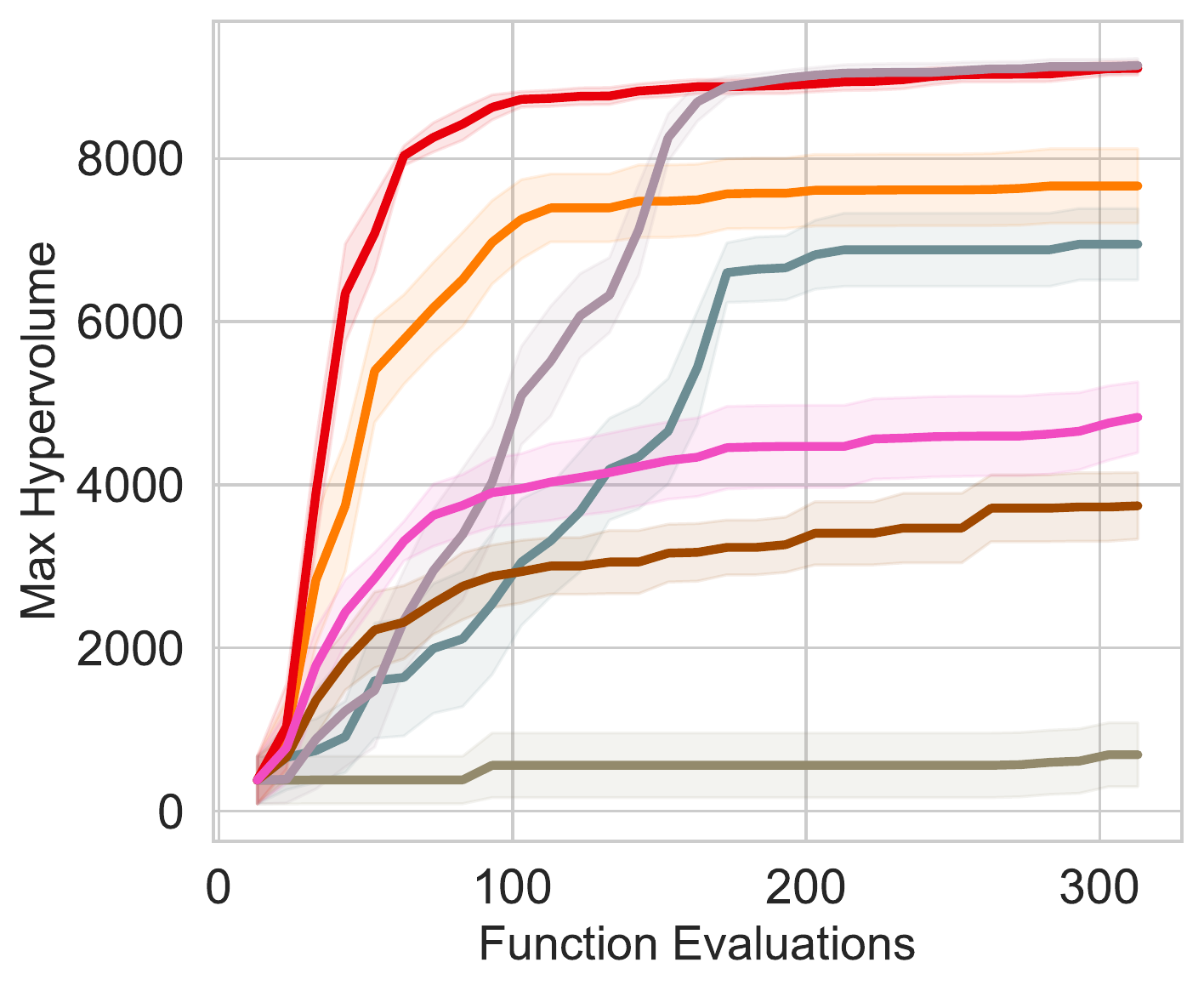}
		\caption{OSY, $q=10$}
		\label{fig:osy}
	\end{subfigure}
	\caption{
		Constrained multi-objective Bayesian Optimization tasks. MTGPs outperform batch models in both the \textbf{(a)} small batch ($q = 2$) and the \textbf{(b)} large batch ($q=10$) setting. The latter was previously computationally infeasible for MTGPs. 
	}
	\label{fig:con_mobo}
	\vspace{-0.2cm}
\end{wrapfigure}

To demonstrate the performance gains of sampling using Matheron's rule, we first vary the number of test points and  tasks for a fixed number of samples and training points.
Following \citet{feng_high-dimensional_2020}, we consider a multi-task version of the Hartmann-$6$ function, generated by splitting the response surface into tasks using slices of the last dimension.
In Figures \ref{fig:mt_speedup_test} and \ref{fig:mt_speedup_test_cpu}, we vary the number of test points for $5$ and $50$ tasks, demonstrating that Matheron's rule sampling is faster and more memory efficient than distributional sampling with either Kronecker or Hadamard MTGPs. 
In Figures \ref{fig:mt_speedup_tasks} and \ref{fig:mt_speedup_tasks_cpu}, we vary the number of tasks for fixed test points, again finding that distributional sampling is only competitive for $5$ tasks on the GPU.
See Appendix \ref{app:further_exps} for sampling with LOVE predictive covariances.

\subsection{Multi-Objective Bayesian Optimization}

We next consider constrained multi-objective BO, where the goal is to find the \emph{Pareto Frontier}, i.e., the set of objective values for which no objective can be improved without deteriorating another  while satisfying all constraints.
To measure the quality of the Pareto Frontier, we compute the hypervolume (HV) of the non-dominated objectives~\citep{yang17}. 
The optimization problem is made more difficult by the presence of black-box constraints (and hence additional outcomes) that must be modelled. 
We use MC batch versions of the ParEGO and EHVI acquisition functions (qParEGO and qEHVI)  \citep{daulton_differentiable_2020}.
To generate new candidate points, qParEGO maximizes expected improvement using a random Chebyshev scalarization of the objectives, while qEHVI maximizes expected hypervolume improvement.
As far as we are aware, \citet{shah16correlated} are the only authors to investigate the use of MTGPs in combination with multi-objective optimization, but only consider 2-4 tasks in the sequential setting (generating one point at a time).  
Here, we use full rank inter-task covariances with LKJ priors \citep{lewandowski2009generating} which we find to work well even in the low data regime.
Our Matheron-based sampling scales to large batches and tasks, and is more sample-efficient on both tasks.

We compare Matheron sampled MTGPs to batch independent MTGPs on the C2DTLZ2 \citep{deb2019constrained} (2 objectives, 1 constraint for a total of 3 modelled tasks) and OSY \citep{osyczka1995new} (2 objectives, 6 constraints for a total of 8 modelled tasks) test problems. On OSY, we also compare to PCA GPs \citep{higdon2008computer} due to the larger number of outputs. 
Following \citet{daulton_differentiable_2020} we use both qParEGO and qEHVI with $q = 2,$ for C2DTLZ2 and optimize for $200$ iterations.
In Figure \ref{fig:c2dtl}, we see that the MTGPs outperform their batch counterparts by coming closer to the known optimal HV.
On OSY, in Figure \ref{fig:osy}, we plot the maximum HV achieved for each method, using a batch size of $q=10,$ optimizing for $30$ iterations, where again we see that the MTGPs significantly outperform their batch counterparts as well as the PCA GPs, which stagnate quickly.

\subsection{Scalable Constrained Bayesian Optimization}\label{sec:comp_bo_app}

We next extend scalable constrained Bayesian Optimization \citep[SCBO,][]{eriksson2020scalable}, a state of the art algorithm for constrained BO in high-dimensional problems, to use MTGPs instead of independent GPs. 
In constrained BO, the goal is to minimize the objective, $f,$ subject to black box constraints, $c_i;$ e.g.,
\begin{align}
	\arg \min_x f(x) \hspace{0.25cm}\text{s.t.}\hspace{0.25cm} c_i(x) \leq 0, \hspace{0.25cm}\forall i \in \{1,\cdots,m\}.
\end{align}
SCBO is a method based on trust regions and uses batched independent GPs to model the outcome and transformed constraints.
We compare to their results on their two largest problems --- the $12$-dimensional lunar lander and the $124$-dimensional MOPTA08 problem, using the same benchmark hyper-parameters. 
To extend their approach, we replace the independent GPs with a single MTGP model with a full rank ICM kernel over objectives and constraints.

\begin{figure}[t]
	\centering
	\captionsetup[subfigure]{justification=centering}
	\begin{subfigure}{0.23\textwidth}
		\centering
		\includegraphics[width=\linewidth]{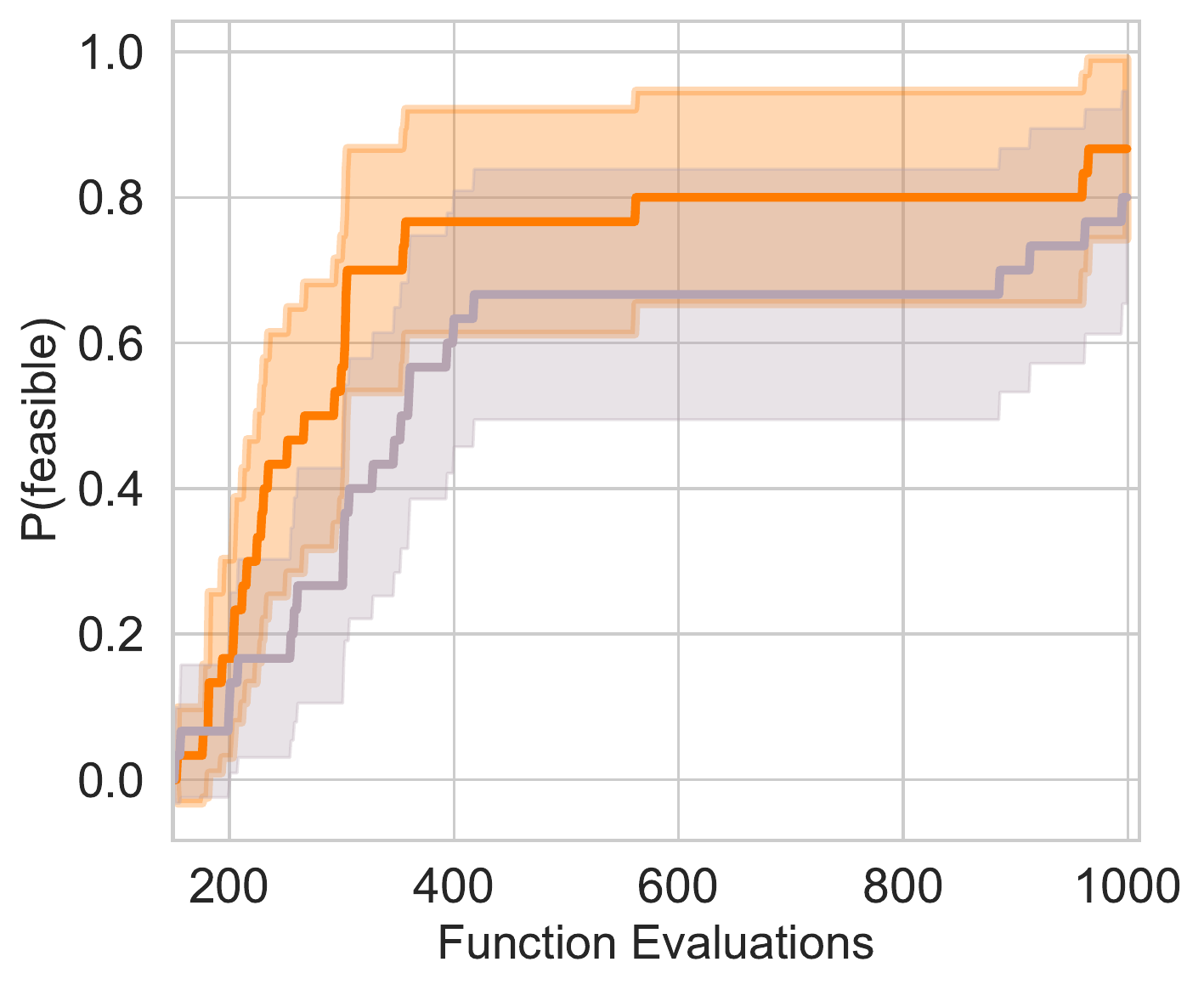}
		\caption{Feasibility, Lunar Lander, $m=50$}
		\label{fig:scbo_prob_feas_50}
	\end{subfigure}
	\begin{subfigure}{0.23\textwidth}
		\centering
		\includegraphics[width=\linewidth]{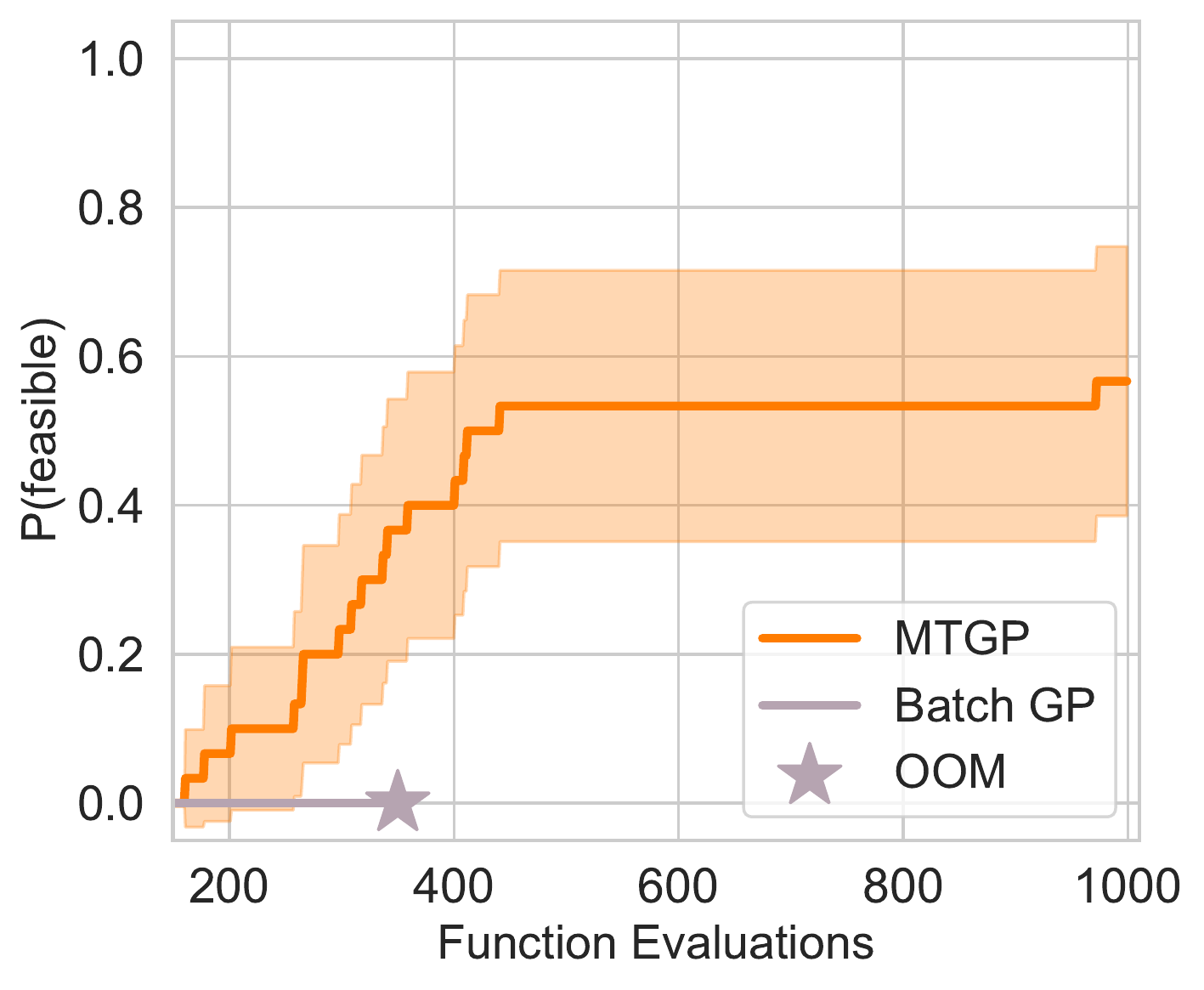}
		\caption{Feasibility, Lunar Lander, $m=100$}
		\label{fig:scbo_prob_feas_100}
	\end{subfigure}  
	\begin{subfigure}{0.23\textwidth}
		\centering
		\includegraphics[width=\linewidth]{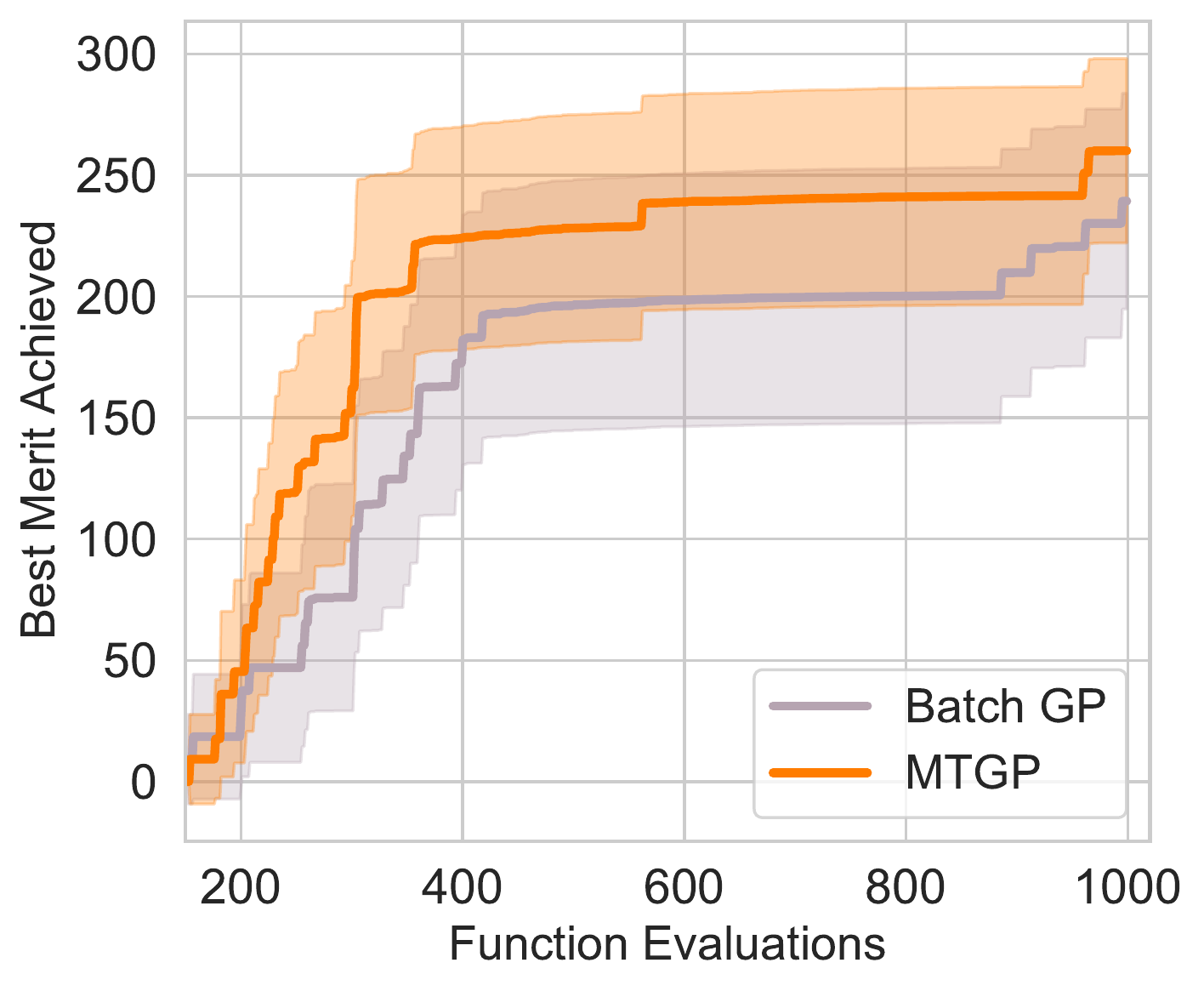}
		\caption{Merit, Lunar Lander, $m=50$}
		\label{fig:lunar_lander_feas}
	\end{subfigure}
	\begin{subfigure}{0.23\textwidth}
		\centering
		\includegraphics[width=\linewidth]{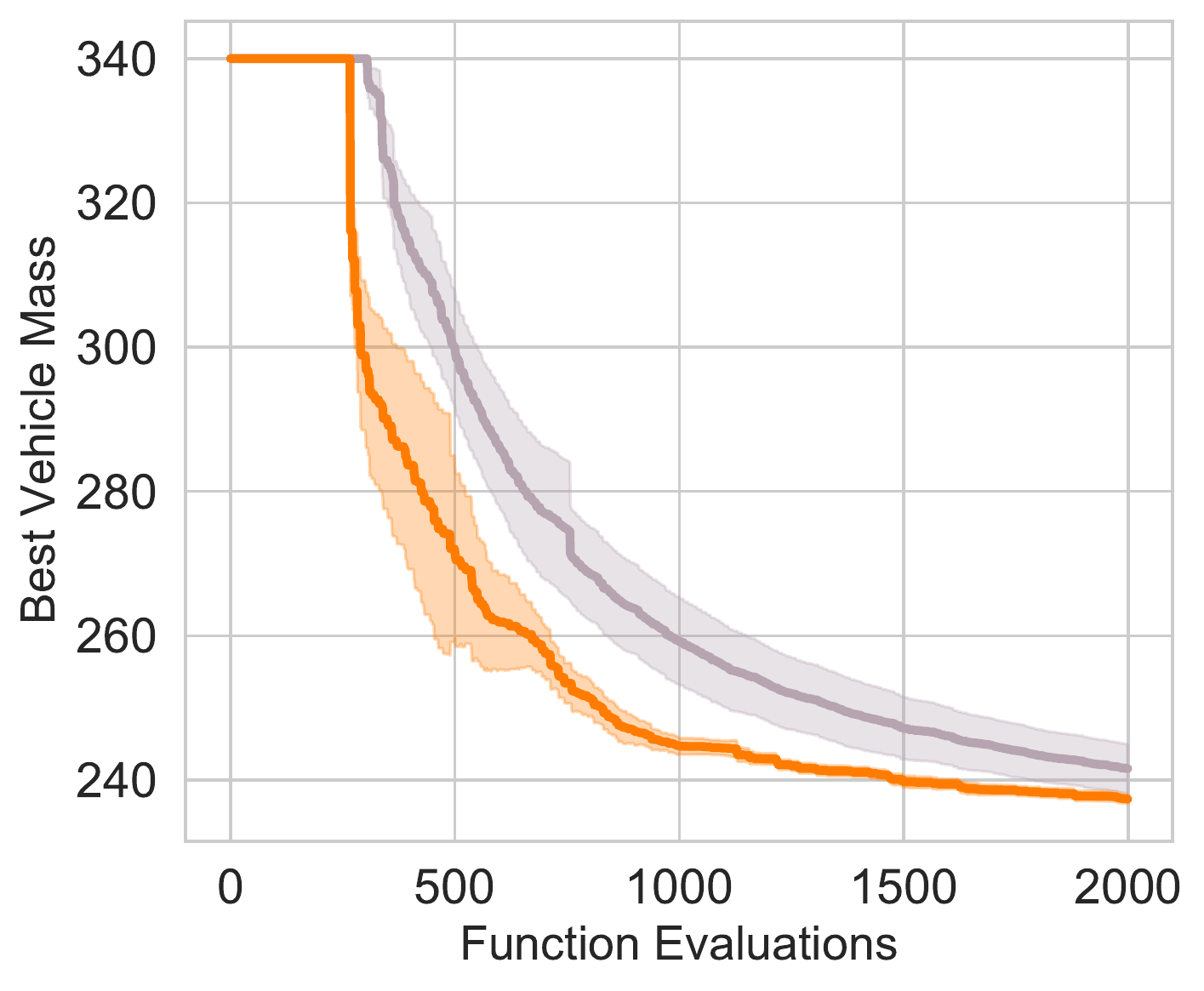}
		\caption{Vehicle mass, MOPTA08, $m=68$}
		\label{fig:mopta_feas}
	\end{subfigure}   
		\caption{Scalable constrained Bayesian Optimization on Lunar Lander $m=50,100$ \textbf{(a-c)} and on MOPTA08 \textbf{(d)}.
		On all three problems, a multi-task GP provides better solutions with better data efficiency.
		The batch GP reaches feasibility on lunar lander with $50$ constraints \textbf{(a)} and a competitive solution \textbf{(c)} but requires more trials, while on $100$ constraints, it simply runs out of memory while the MTGP succeeds. 
		On MOPTA08, the MTGP reaches a better solution faster \textbf{(d)}.}
	\label{fig:scbo}\label{fig:scbo_prob_feas}
\end{figure}

\textbf{Lunar Lander:}
We first consider the lunar lander problem with both $50$ and $100$ constraints from the OpenAI Gym \citep{openaigym}. 
Following \citet{eriksson2020scalable}, we initialize with 150 data points and use TuRBO with Thompson sampling with batches of $q = 20$ for a total of $1000$ function evaluations and repeat over $30$ trials.
Using a multi-task GP reduces the number of iterations to achieve at least a $50\%$ chance of feasibility by about $75$ steps for the $50$ constraint problem (Figure \ref{fig:scbo_prob_feas_50}).
On the $100$ constraint problem, the batch GP runs out of memory after $350$ steps on a single GPU and never achieves feasibility, as indicated in Figure \ref{fig:scbo_prob_feas_100}.
In Figure \ref{fig:lunar_lander_feas}, we show the best merit ($f(x)  \prod_{i=1}^m 1_{c_i(x) \leq 0}$) achieved 
where the MTGPs are able to achieve feasibility in fewer samples, but do not reach significantly higher reward. 
Wall clock times for the $m = 50$ constraint problem, a table of the steps to achieve feasibility, and a comparison to PCA-GPs \citep{higdon2008computer} are given in Appendix~\ref{app:further_exps}. 

\textbf{MOPTA08:}
We next compare to batch GPs on the MOPTA08 benchmark problem \citep{jones2008large} which has $68$ constraints that measure the feasibility of a vehicle's design, while each dimension involves gauges, materials, and shapes.
We use Thompson sampling to acquire points with a batch size of $q = 10,$ $130$ initial points, and optimize for $2000$ iterations repeating over $9$ trials.
The results are shown in Figure \ref{fig:mopta_feas}, where we again observe that SCBO with MTGPs significantly improves both the time to feasibility and the best overall objective found. Using MTGPs would have been computationally infeasible in this setting without our efficient posterior sampling approach.

\subsection{Composite Bayesian Optimization with the HOGP}

Finally, we push well beyond current scalability limits by extending BO to deal with \emph{many thousands} of tasks, enabling us to perform sample-efficient optimization in the space of images.
\emph{Composite BO} is a form of BO where the objective is of the form $\max_{x}g(h(x))$, where $h$ is a multi-output black-box function modelled by a surrogate and $g$ is cheap to evaluate and differentiable.
Decomposing a single objective into constituent objectives in this way can provide substantial improvements in sample complexity \citep{astudillo_bayesian_2019}.
\citet{balandat_botorch_2020} gave convergence guarantees for optimizing general sampled composite acquisition function objectives that extend to MTGP models under the same regularity conditions. 
However, both works experimentally evaluate only batch independent multi-output GPs.

We compare to three different baselines: random points (Random), expected improvement on the metric (EI), and batch GPs optimizing EI in the composite setting (EI-CF).
We consider the HOGP \citep{zhe_scalable_2019} with Matheron's rule sampling (HOGP-CF) as well as an extension of HOGPs with a prior over the latent parameters that encourages smoothly varying latent dimensions (HOGP-CF + GP); see Appendix~\ref{app:hogp_extension} for details.
More detailed descriptions of the problems are provided in Appendix~\ref{app:exp_details}.

\begin{figure*}[t!]
	\begin{subfigure}{\textwidth}
		\centering
		\includegraphics[width=0.8\linewidth]{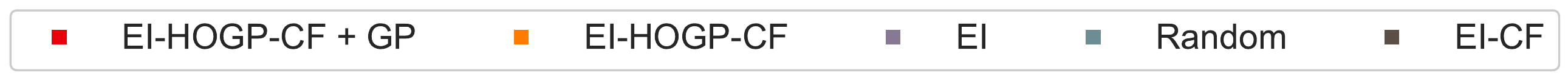}
	\end{subfigure}
	\begin{subfigure}{0.24\textwidth}
		\centering
		\includegraphics[width=\linewidth]{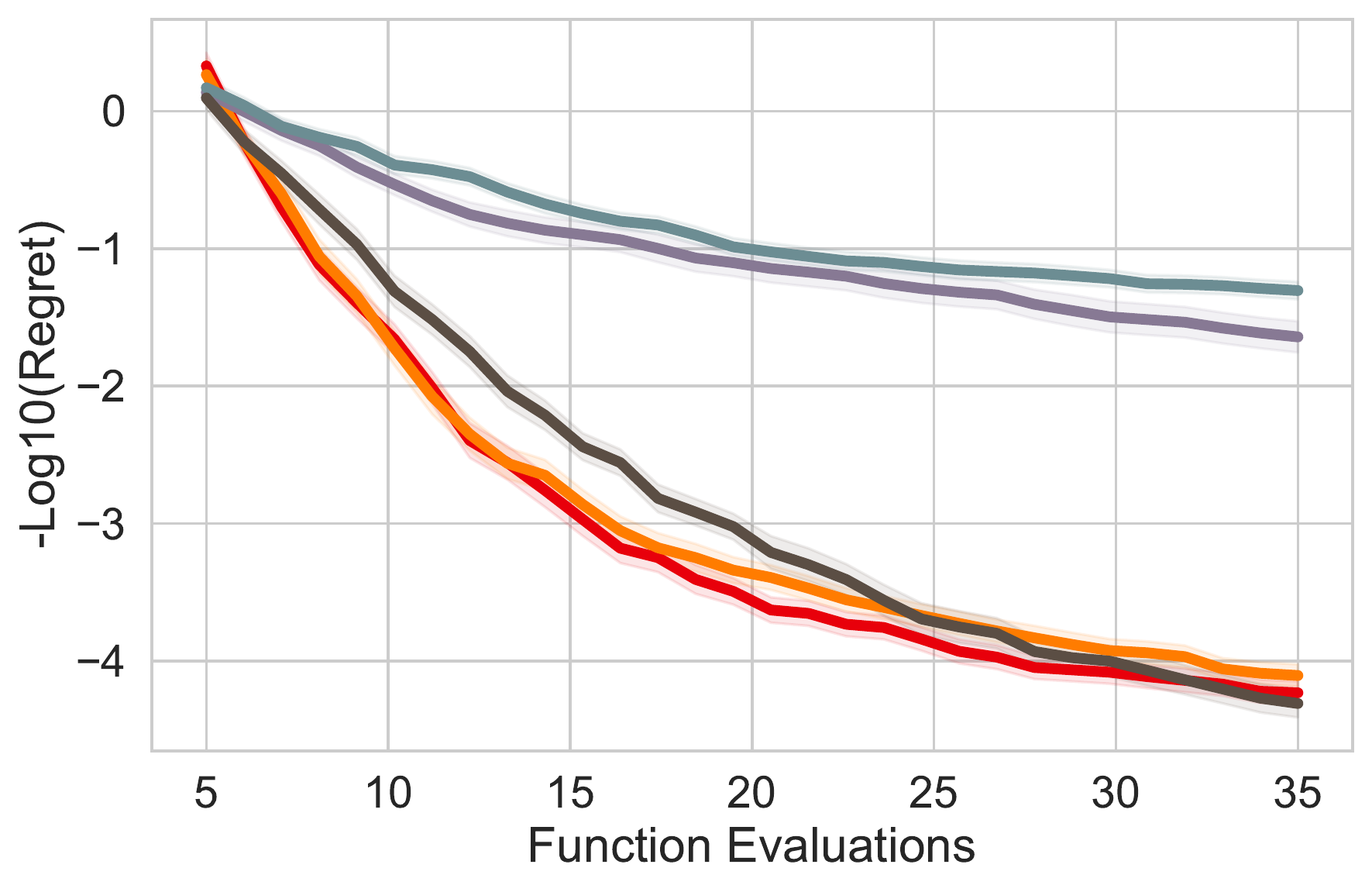}
		\caption{\centering Pollutant concentration\linebreak ($t=3 \times 4, d=4$)}
		\label{fig:environmental}
	\end{subfigure}
	\begin{subfigure}{0.24\textwidth}
		\centering
		\includegraphics[width=\linewidth]{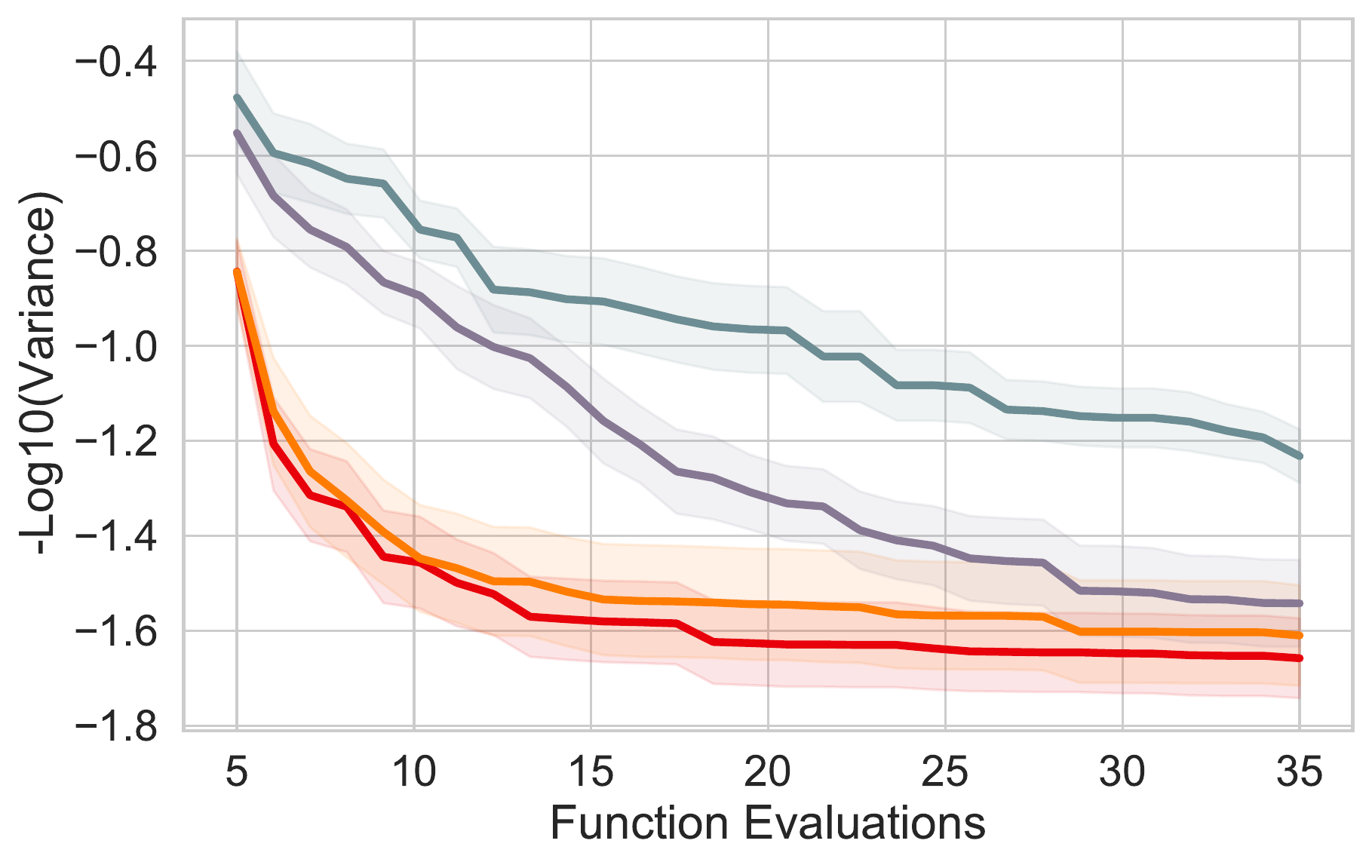}
		\caption{\centering PDE control\linebreak ($t=2 \times 64 \times 64, d=4$)}
		\label{fig:pde}
	\end{subfigure}
	\begin{subfigure}{0.24\textwidth}
		\centering
		\includegraphics[width=\linewidth]{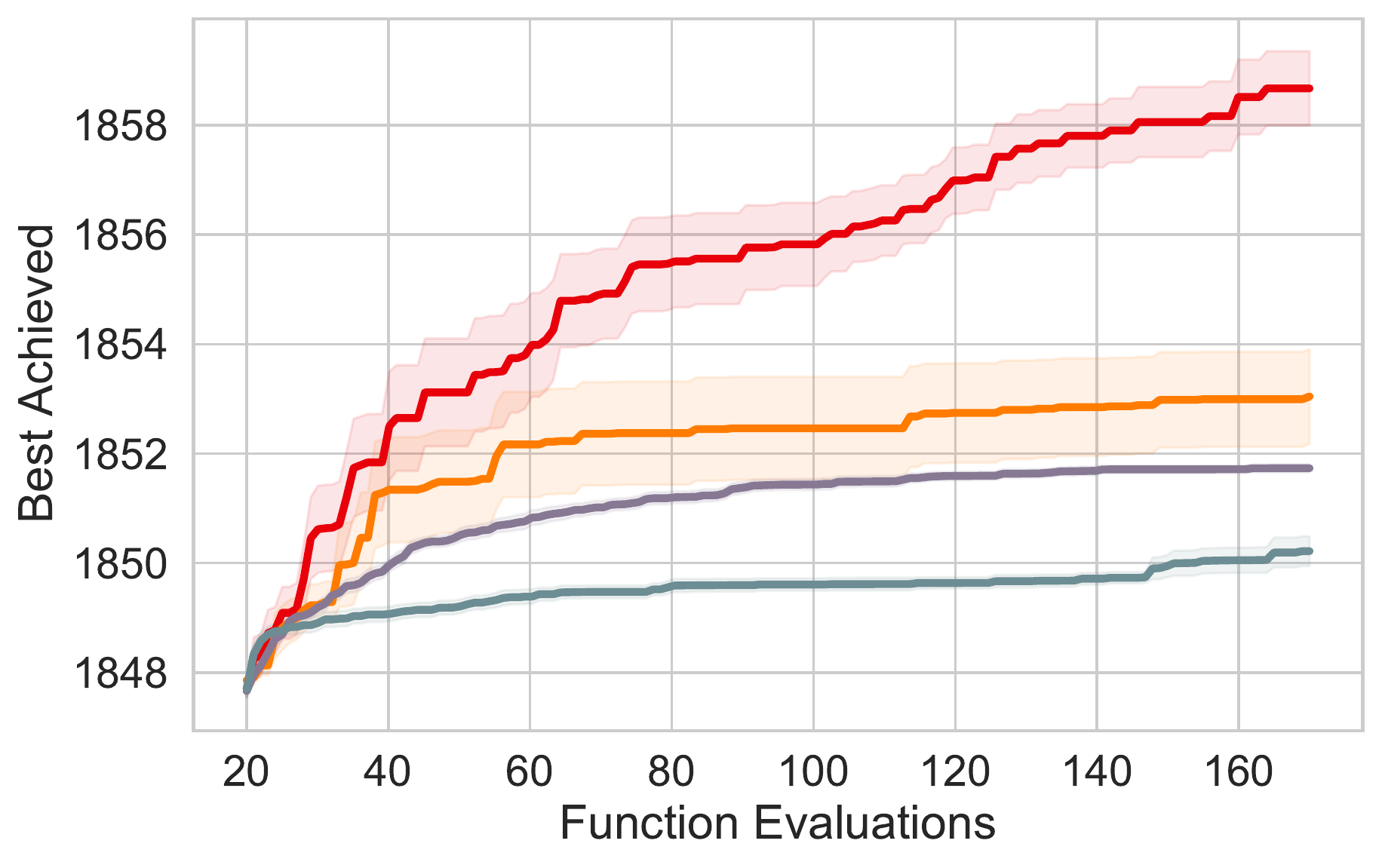}
		\caption{\centering Cell-tower coverage\linebreak ($t=2 \times 50 \times 50, d=30$)}
		\label{fig:celltower}
	\end{subfigure}
	\begin{subfigure}{0.24\textwidth}
		\centering
		\includegraphics[width=\linewidth]{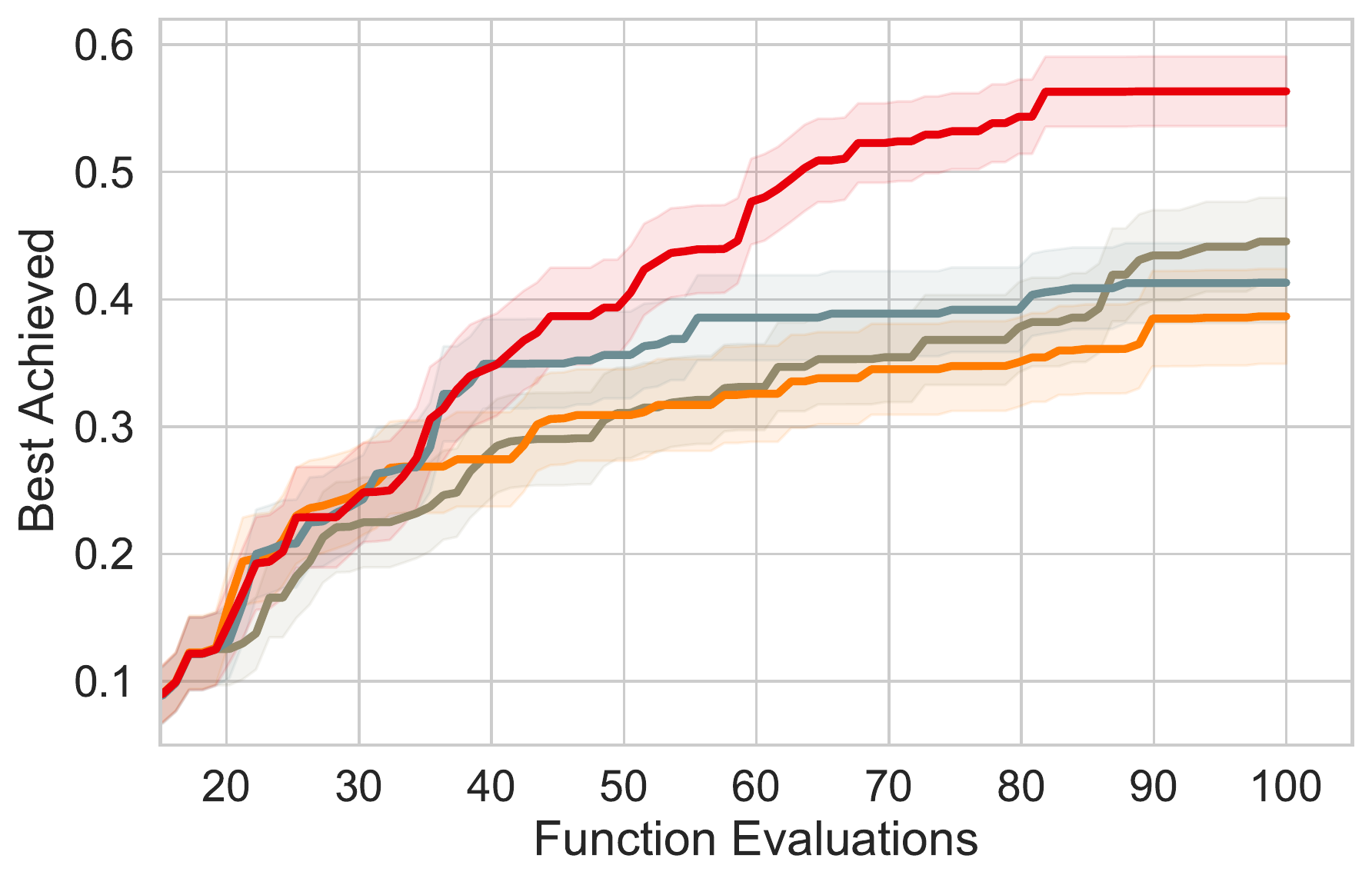}
		\caption{\centering Optics\linebreak ($t=16 \times 64 \times 64, d=4$)}
		\label{fig:optics}
	\end{subfigure}
	\caption{Performance of BO with and without HOGP-based composite objectives (EI-HOGP-CF) on four scientific problems. EI-HOGP-CF outperforms a standard BO model directly on the metric itself (EI) and a random baseline (Random). 
		Composite BO with the HOGP also outperforms composite BO with independent GPs (EI-CF). Our smooth latents for the HOGP (EI-HOGP + GP) typically outperform the HOGP itself. EI-CF is only feasible on the smallest problem.}
	\label{fig:hogp}
	\vspace{-0.5cm}
\end{figure*}

\textbf{Chemical Pollutants:}
Following \citet{astudillo_bayesian_2019}, we start with a simple  spatial problem in which environmental pollutant concentrations are observed on a $3 \times 4$ grid originally defined in \citet{bliznyuk2008bayesian}. 
The goal is to optimize a set of four parameters to achieve the true observed value by minimizing the mean squared error of the output grid to the output grid of the true parameters. The results, over $50$ trials, are shown in Figure \ref{fig:environmental}, where we find that the HOGP models with these few tasks outperform both independent batch GPs (but slightly) and BO on the metric itself.

\textbf{Optimizing PDEs:}
As a larger experimental problem, we  consider optimizing the two diffusivity and two rate parameters of a spatial Brusselator PDE (solved in \texttt{py-pde} \citep{zwicker2020py}) to minimize the weighted variance of the PDE output as an example of control of a dynamical system. Here, we solve the PDE on $64 \times 64$ grid, producing output solutions of size $2 \times 64 \times 64.$
Results over $20$ trials are shown in Figure \ref{fig:pde}, where the HOGP models outperform EI fit on the metric and the random baseline.

\textbf{Cell-Tower Coverage:} 
Following \citet{dreifuerst2020optimizing}, we optimize the simulated ``coverage map" resulting from the transmission power and down-tilt settings of $15$ cell towers (for a total of $30$ parameters) based on a scalarized quality metric combining signal power and inference at each location so as to maximize total coverage, while minimizing total interference. 
To reduce model complexity, we down-sample the simulator output to $50 \times 50$, initializing the optimization with $20$ points.
Figure~\ref{fig:celltower} presents the results over $20$ trials, where the HOGP models with composite objective EI outperform EI, indicating that modeling the full high-dimensional output is valuable.

\textbf{Optical Interferometer:}
Finally, we consider the tuning of an optical interferometer by the alignment of two mirrors as in \citet{sorokin2020interferobot}.
Here, the problem is to optimize the mirror coordinates to align the interferometer so that it reflects light without interference. There is a sequence of $16$ different interference patterns and the simulation outputs are $64 \times 64$ images (a tensor of shape $16 \times 64 \times 64$).
Thus, we jointly model {$65$,$536$} output dimensions. 
Scaling composite BO to a problem of this size would be impossible without the step change in scalability our method provides. Results are shown in Figure~\ref{fig:optics} over $20$ trials, where we find that the HOGP-CF + GP models outperform EI, with the HOGP + GP under-performing (perhaps due to high frequency variation in the images).

Across all of our experiments, we consistently find that composite BO is considerably more sample efficient than BO on the metric itself, with significant computational improvements gained from using the HOGP as compared to batch GPs, which are infeasible much beyond batch sizes of $100.$
Furthermore, our structured prior approach for the latent parameters of the HOGP tends to outperform the random initialization strategy in the original work of \citet{zhe_scalable_2019}.

\section{Conclusion}
We demonstrated the utility of Matheron's rule for sampling the posterior in multi-task Gaussian processes. Combining Matheron's rule with scalable Kronecker algebra enables posterior sampling in $\mathcal{O}(n^3 + t^3)$ rather than the previous $\mathcal{O}(n^3t^3)$ time. This renders posterior sampling from high-order Gaussian processes \citep{zhe_scalable_2019} practical,  for the first time unlocking Bayesian Optimization with composite objectives defined on high-dimensional outputs. 
This increase in computational efficiency dramatically reduces the time required to do multi-task Bayesian Optimization, and thus enables practitioners to achieve better automated Bayesian decision making. 
While we focus on the application to Bayesian Optimization in this work, our contribution is much broader and provides a step change in scalability to all methods that in involve sampling from MTGP posteriors. 
We hope in the future to explore stronger inter-task covariance priors to make MTGP model fits even more sample efficient.

\section*{Acknowledgements}
WJM, AGW are supported by an Amazon Research Award, NSF I-DISRE 193471, NIH R01 DA048764-01A1, NSF IIS-1910266, and NSF 1922658 NRT-HDR:FUTURE Foundations, Translation, and Responsibility for Data Science. WJM was additionally supported by an NSF Graduate Research Fellowship under Grant No. DGE-1839302, and performed part of this work during an internship at Facebook.
We would like to thank Paul Varkey for providing code and David Eriksson, Qing Feng, and Polina Kirichenko for helpful discussions.

\bibliographystyle{apalike}
\bibliography{references}

\clearpage

\appendix
\renewcommand\thefigure{A.\arabic{figure}}
\renewcommand\theequation{A.\arabic{equation}}   
\renewcommand{\thesection}{\Alph{section}}
\setcounter{figure}{0}
\setcounter{equation}{0}
  \vbox{
    \hsize\textwidth
    \linewidth\hsize
    \vskip 0.1in
      \hrule height 4pt
  \vskip 0.25in
  \vskip -\parskip
    \centering
    {\centering \LARGE\bf Supplementary Materials for Bayesian Optimization with High-Dimensional Outputs\par}
      \vskip 0.29in
  \vskip -\parskip
  \hrule height 1pt
  \vskip 0.1in
  }
\section*{Organization}

The Appendix is organized as follows:
\begin{itemize}
    \item Appendix \ref{app:limitations} describes limitations and negative societal impacts of our work.
    \item Appendix \ref{app:rel_work} describes further background and related work on Kronecker matrix vector products, Matheron's rule, multi-task Gaussian process models, and sampling multi-task posteriors using LOVE \citep{pleiss_constant-time_2018}.
    \item Appendix \ref{app:methods} gives a more detailed description of sampling multi-task Gaussian process posteriors using Matheron's rule, our set of priors for the HOGP \citep{zhe_scalable_2019}, and a proof of convergence using Matheron's rule sampling to optimize MC acquisition functions.
    \item Appendix \ref{app:further_exps} describes two more experiments on multi-task BO and contextual BO, before giving more detail and results on the experiments in the main paper.
\end{itemize}

\section{Limitations and Societal Impacts}\label{app:limitations}
From a practical perspective, we see several inter-related limitations:
\begin{itemize}
    \item If the underlying multi-output function we are trying to model has very different lengthscales for each output, then the shared data covariance matrix of the MTGP may not be able to model each output very well. In practice, this seems to be rather rare, but may prove to be more problematic for the HOGP model due to the number of outputs that we model.
    \item Numerical instability can be an issue when solving systems of equations using eigen-decompositions; however, we did not find it to be problematic as our implementation performs the eigen-decomposition in double precision despite all other computations being performed in single precision. 
    \item As currently described, we can only apply our method to block design, fully observed settings where all data points and tasks are observed at the same time. In future work, we hope to extend past this limitation, perhaps using the approaches of \citep{zhang2007maximum,wilson2014fast}.
    \item \emph{Autokrigeability} may play a larger role here than it does in standard MTGP scenarios, especially when the black box function is observed without observation noise. See \citet{bonilla_multi-task_2007} for a longer description of this problem. However, we still get substantial computational enhancements from using the HOGP and Kronecker structure compared to modeling each output with an independent GP model.
    \item Negative transfer can arise if the relationship between tasks is highly non-linear. Non-linear MTGP models are needed to remedy this issue, rather than the ICM model we consider \citep{boustati2019non}.
\end{itemize}

Looking farther out, we do not broadly anticipate direct negative societal impacts as a result of our work. 
However, Bayesian Optimization, which we focus on in this paper, is a very generic methodology for optimizing black box functions.
This technology can be used for good reasons such as public health surveillance and modelling \citep{2020arXiv200407641L,andrade2020finding} or technological design, such as the radio frequency tower location and optics problems discussed in this paper.
These types of applications should hopefully increase the likelihood of deployment of new advanced technologies such as $5G$ cell coverage globally and thus help to provide more people with stable jobs and employment.

\section{Further Background and Related Work}\label{app:rel_work}
In this Appendix, we note of several more references in the geostatistics community who use Matheron's rule in multi-task Gaussian processes (MTGPs), give more background on MTGPs, before moving into a more detailed description of LOVE predictive variances and fast sampling \citep{pleiss_constant-time_2018}, including for multi-task GPs.

\subsection{An Extended Reference on Kronecker Structure}\label{app:kron}
We may exploit Kronecker structure in matrices in order to perform more efficient matrix vector multiplies and solves.
For a more detailed introduction to Kronecker matrices and their properties, please see \citet[Chapter 5]{saatcci2012scalable} as well as Sections 1.3.7-8 and 12.3 of \citet{golub_matrix_2013}.
Matrix vector multiplies (MVMs) are efficient and can be computed from:
\begin{align*}
	z = (K_1 \otimes K_2)\text{vec}(A) = \text{vec}(K_2 A K_1^\top);
\end{align*}
if $K_1 \in \mathbb{R}^{n_1 \times n_1}$ and $K_2 \in \mathbb{R}^{n_2 \times n_2}.$ As a result, computing $z$ costs $\mathcal{O}(n_1^2 + n_2^2 + n_1 n_2 (n_1 + n_2))$ time \citep[12.3]{golub_matrix_2013}.

There are several other useful Kronecker product properties (again summarized from \citet{saatcci2012scalable,golub_matrix_2013} amongst other sources):
$
(A \otimes B)(C \otimes D) = AC \otimes BD,
$
if the shapes match, and 
$
(K_1 \otimes K_2)^{-1} = (K_1^{-1} \otimes K_2^{-1}),
$
and 
$
\log |K_1 \otimes K_2| = n_2 \log |K_1| + n_1 \log |K_2|.
$
Root (cholesky) decompositions also factorize across the Kronecker product as $A \otimes B = LL^\top \otimes RR^\top = (L \otimes R)(L^\top \otimes R^\top).$

Using these properties, we then are able to compute matrix inverses of Kronecker plus constant diagonal matrices as
$
(K_1 \otimes K_2 + \sigma^2 I)^{-1} = (Q_1 \otimes Q_2)(\Lambda_1 \otimes \Lambda_2 + \sigma^2 I)^{-1}(Q_1 \otimes Q_2)^\top,
$
where $K_i = Q_i \Lambda_i Q_i^\top$ (the eigen-decomposition of $K_i$).
As eigen-decompositions cost cubic time in the size of the matrix then the total cost for these matrix solves is $\mathcal{O}(n^3 + t^3 + nt(n+t))$ with the final term coming from matrix vector products.
In general, there is no efficient way to compute matrix inverses of the form: $(K_1 \otimes K_2 + T_l)^{-1}$ where $T_l$ is a diagonal matrix that is non-constant. One can use conjugate gradients to compute solves in that setting.

However, as \citet{rakitsch2013noise} demonstrate, if we assume a structured noise term, e.g. a likelihood that is $\vec{Y} \sim \mathcal{N}(f, \Sigma_X \otimes \Sigma_T),$ then there is an efficient method of computing matrix inverses and solves:
\begin{align*}
(K_1 \otimes K_2 + \Sigma_X \otimes \Sigma_T)^{-1} &= (Q_X \Lambda_X^{1/2} \otimes Q_T \Lambda_T^{1/2})(\tilde Q_1 \otimes \tilde Q_2)\\
&(\tilde \Lambda_1 \otimes \tilde \Lambda_2 + I)^{-1}(\tilde Q_1 \otimes \tilde Q_2)^\top (Q_X \Lambda_X^{1/2} \otimes Q_T \Lambda_T^{1/2})^\top ,
\end{align*}
where
$\tilde Q_1 \tilde \Lambda_1 \tilde Q_1^\top = \Lambda_X^{-1/2}Q_X^\top K_1 Q_X \Lambda_X^{-1/2}$
and $\tilde Q_2$ is defined in the same manner.
This costs two eigen-decompositions and several matrix vector multiplications for a total of $\mathcal{O}(n^3 + t^3 + nt(n+t))$ time.

\citet{rakitsch2013noise} and \citet{bonilla_multi-task_2007} brush the $nt$ terms in scaling under the rug as they are dominated by the cubic time complexity of the eigen-decompositions.
We follow this notation in our results.

\subsection{Matheron's Rule}
Matheron's rule is well known in the geostatistics literature where it is called ``prediction by conditional simulation" \citep{chiles2005prediction,chiles2009geostatistics}.
There, it is also known that it can be applied to multi-task Gaussian processes, as described in \citet{de1994reminders} and mentioned in \citet{emery2007conditioning}.
\citet{larocque2006conditional} use Matheron's rule to sample in the multi-task setting (termed co-kriging in that literature) and study the uncertainty of the ICM kernel on groundwater use cases.
However, they focus only on two to three tasks and do not exploit the Kronecker structure in the multi-task covariances.
\citet{doucet_note_2010} gave a didactic explanation of Matheron's rule with the goal of introducing it to the broader machine learning community, explaining its applications in sampling Kalman filters.

\subsection{Multi-task Gaussian Process models}\label{app:mtgps}
Both computationally efficient, e.g. \citet{bruinsma2020scalable}, and variational methods, e.g. \citet{nguyen2014collaborative,dai2017efficient}, can be made more efficient in our Matheron's rule implementation.
Furthermore, other kernels for MTGPs such as the linear model of coregionalization (LMC) and the semiparametric latent factor models \citep{rasmussen_gaussian_2008} can also be extended to use Matheron's rule in the way that we mention here.
See \citet{alvarez2012kernels,bruinsma2020scalable} for an extended description of the relationships between various models of multi-task GPs.
In most of the paper, we focus solely on the exact setting with ICM kernels for both didactic and implementation purposes. We detail the extension to the LMC case in Appendix~\ref{appdx:extension_LMC}.

\paragraph{Training Multi-task GPs: }
To train single task GPs, we optimize the marginal log likelihood with gradient based optimization; this approach extends to training multi-task GPs as well.
In the single task setting, the marginal log likelihood (MLL) is:
$\log p(y) = \frac{1}{2}\left(N \log 2\pi - \log |K+\sigma^2 I| - y^\top (K + \sigma^2 I)^{-1} y\right)$ (Eq. 5.8 in \citet{rasmussen_gaussian_2008}).
To extend the MLL into the multi-task setting, we only need to exploit Kronecker identities as described in Section \ref{app:kron} and focus solely on the constant diagonal case.\footnote{Please see \citet{rakitsch2013noise} for the structured noise case.}
The MLL becomes
\begin{align*}
\log p(y) = \frac{1}{2}\left(NT \log 2\pi - \log |K_X \otimes K_T +\sigma^2 I| - \text{vec}(y)^\top (K_X \otimes K_T + \sigma^2 I)^{-1} \text{vec}(y)\right)
\end{align*}
The log determinant term simplifies to
\begin{align*}
\log |K_X \otimes K_T +\sigma^2 I| = \log |\Lambda_X \otimes \Lambda_T + \sigma^2 I|
\end{align*}
which is just the determinant of a diagonal matrix.
The quadratic form similarly simplifies 
\begin{align*}
\text{vec}(y)^\top (K_X \otimes K_T + \sigma^2 I)^{-1} \text{vec}(y)= \text{vec}(y)^\top(Q_X \otimes Q_T)(\Lambda_X \otimes \Lambda_T + \sigma^2 I)^{-1}(Q_X \otimes Q_T)^\top \text{vec}(y)^\top.
\end{align*}

We then estimate the kernel hyper-parameters with gradient based optimization of the log marginal likelihood, following \citet{stegle2011efficient,rakitsch2013noise}.
The predictive means and variances of the MTGP are given by \citet{bonilla_multi-task_2007}.

\subsection{LOVE Variance Estimates and Sampling Multi-Task Posteriors}\label{app:love}

To compute $\textcolor{blue}{\mu^*}$ in \eqref{main:eq:exact_pred_mean}, we need to solve the linear system, $(K_{XX} + \sigma^2 I )^{-1} \mathbf{y}$; this solution costs $\mathcal{O}(n^3)$ when using the Cholesky decomposition \citep{rasmussen_gaussian_2008}.
Recently, \citet{gardner_gpytorch_2018} have proposed using preconditioned conjugate gradients (CG) to solve these linear systems in $\mathcal{O}(rn^2)$ time, where $r$ is the number of conjugate gradients steps.
Similarly, to compute the predictive variance in \eqref{main:eq:exact_pred_var}, we need to solve $n_{\text{test}}$ systems of equations of the form $(K_{XX} + \sigma^2 I )^{-1}K_{X \mathbf{x_{\text{test}}}},$ which would naively cost $\mathcal{O}(n^3 n_{\text{test}})$ time, reduced to $\mathcal{O}(n_{\text{test}}n + n^2 n_{\text{test}})$ time if we have precomputed the Cholesky decomposition of $(K_{XX} + \sigma^2 I )$.

\citet{pleiss_constant-time_2018} propose to additionally use a cached Lanczos decomposition (called LOVE) such that $RR^\top \approx (K_{XX} + \sigma^2 I )^{-1}$ and then simply perform matrix multiplications against $K_{X \mathbf{x_{\text{test}}}}$ to compute the predictive variances.
The time complexity of the Lanczos decomposition is also $\mathcal{O}(n^2 r)$ for computing a rank $r$ decomposition.
Sampling proceeds similarly by computing a rank $r$ decomposition to $\textcolor{blue}{\Sigma}$ in \eqref{eq:dist_sampling}.
The overall time complexity for computing $s$ samples at from the predictive distribution in the exact formulation is reduced to $\mathcal{O}(rn^2 + s r n_{\text{test}} + r n_{\text{test}}^2 + n^2 n_{\text{test}}).$
These advances in GP inference have enabled exact single-output GP regression on datasets of up to one million data points \citep{wang_exact_2019}.
Furthermore, \citet{gardner_gpytorch_2018,pleiss_constant-time_2018} demonstrate that one can choose $r < n$ while maintaining accuracy up to numerical precision in floating point. 

\paragraph{LOVE for multi-task predictions.} It is possible to  exploit the Kronecker structure in the posterior distribution to enable more efficient sampling than the naive $\mathcal{O}((n_{\text{test}}t)^3)$ approach \citep{gardner_gpytorch_2018}. 
$(K_{XX} \otimes K_T + \sigma^2 I_{nT})$ admits an efficient matrix vector multiply (MVM) due to its Kronecker structure --- this MVM takes $\mathcal{O}(nt(n+t))$ time, see Appendix \ref{app:kron}.
If we use LOVE to additionally decompose the matrix such that $LL^\top \approx (K_{XX} \otimes K_T + \sigma^2 I_{nT})^{-1},$ then computing $L$ that has rank $r$ takes $\mathcal{O}(r(nt(n+t)))$ time but has a storage cost of $nt \times r$, which is multiplicative in the combination of $n$ and $t$.
Then, computing $(K_{x_{\text{test}}, X} \otimes K_T)L$ takes $\mathcal{O}((n_\text{test}nt + nt^2)r)$ time.
This represents an improvement over the naive method, but still ends up requiring computing 
$\textcolor{blue}{\Sigma^*} = AA^\top = (K_{x_{\text{test}}, x_{\text{test}}} \otimes K_T) - (K_{x_{\text{test}}, X} \otimes K_T)LL^\top (K_{x_{\text{test}}, X}^\top \otimes K_T)$,
which is a $n_{\text{test}}t \times n_{\text{test}}t$ matrix, and therefore costs at least $(n_{\text{test}}t)^2$ time to decompose and thus sample.
The overall time complexity is then $\mathcal{O}(r(n_{\text{test}}t)^2 + rn(n_\text{test} + t^2)).$
Finally, the matrix $A$ must be re-computed from scratch for each new test point $x_{\text{test}},$ and only the matrix $L$ can be reused for different test points (as in a Bayesian Optimization loop).

\section{Further Methods}\label{app:methods}

In this Appendix, we start by giving a more detailed derivation of our efficient Matheron's rule implementation for sampling from multi-task Gaussian processes, then we describe a new set of priors for the HOGP model, before closing with a convergence proof of using Matheron's rule in optimizing Monte Carlo acquisition functions.

\subsection{Details of Using Matheron's Rule}
We derive only the zero mean case here for simplicity.

Succinctly, to generate $f(x_{\text{test}}) | Y = y$ under the ICM, we may draw a joint sample from the prior 
\begin{align}
	(f, Y) \sim \mathcal{N}\left(0, \left(\begin{array}{cc}
		K_{XX}  & K_{Xx_{\text{test}}}  \\
		K_{x_{\text{test}}X} & K_{x_{\text{test}}x_{\text{test}}}
	\end{array} \right) \otimes K_T\right)
	\label{eq:app_joint}
\end{align}
and then update the sample via an update from computing $(K_{\text{train}} + T_l)^{-1}(y - Y - \epsilon).$
Here, $T_l$ represents the noise likelihood used --- it could be either constant diagonal: $\sigma^2I,$ non-constant diagonal with a variance term for each task: $D,$ or Kronecker structured itself: $D_n \otimes T_t,$
where $T_t$ is a dense matrix.
The formula is given as 
\begin{align}
	\bar{f} = f + K_{x_{\text{test}} X}(K_{XX} + \sigma^2 I)^{-1}(y - Y - \epsilon).
	\label{eq:matherons_rule}
\end{align}

The joint covariance matrix, $\Ktt,$ in \eqref{eq:app_joint} is highly structured
\begin{align*}
	\Ktt = \left(\begin{array}{cc}
		K_{XX}  & K_{x_{\text{test}}X}  \\
		K_{Xx_{\text{test}}} & K_{x_{\text{test}}x_{\text{test}}}
	\end{array} \right) \otimes K_T = \tilde{R}\tilde{R}^\top \otimes LL^\top = (\tilde{R} \otimes L)(\tilde{R} \otimes L)^\top,
\end{align*}
where $\tilde R \tilde R^\top \approx K_{(X, x_{\text{test}}), (X, x_{\text{test}})}$ and $LL^\top = K_T$ and we exploit Kronecker structure.
To compute $\tilde R,$ we follow \citet{jiang_efficient_2020}'s method for fantasization (given in Proposition 2 therein):
\begin{align*}
	\left(\begin{array}{cc}
		K_{XX} & K_{x_{\text{test}}X}  \\
		K_{Xx_{\text{test}}} &  K_{x_{\text{test}}x_{\text{test}}}
	\end{array} \right) = \left(\begin{array}{cc}
		R & 0  \\
		L_{12} & L_{22}
	\end{array} \right)\left(\begin{array}{cc}
		R & 0  \\
		L_{12} & L_{22}
	\end{array} \right)^\top,
\end{align*}
where $R R^T = K_{XX}$ and $L_{12}^\top = R^{-1}K_{Xx_{\text{test}}}$.
To compute $L_{22},$ we have to compute
\begin{align*}
	L_{22} = (K_{x_{\text{test}}x_{\text{test}}} - L_{12}L_{12}^\top)^{1/2}.
\end{align*}
If we assume a rank $r$ decomposition of $K_{XX},$ computed in $\mathcal{O}(n^2r)$ time (e.g. a LOVE decomposition \citet{pleiss_constant-time_2018}), then computing $L_{12}$ costs $\mathcal{O}(n_{\text{test}}r n)$ if we have stored $R^{-1}$ (or $R^+$).
Similarly, computing $L_{22}$ costs $\mathcal{O}(n_{\text{test}}^2 r)$ time if we use a Lanczos decomposition (for large $n_{\text{test}}$). We could also use contour integral quadrature \citep{pleiss_fast_2020} to compute $L_{22}v$ at the expense of having to re-compute it every time we want to draw a new sample.
Sampling then proceeds by computing 
\begin{align}
	(f, Y) = \left(\left(\begin{array}{cc}
		R & 0  \\
		L_{12} & L_{22}
	\end{array} \right) \otimes L\right) z,
\end{align}
where $z \sim \mathcal{N}(0, I).$
We can then compute $\epsilon \sim \mathcal{N}(0, T_l),$ where $T_l$ is the noise distribution. 
Typically, $T_l$ will be diagonal so this sampling just requires taking the square root of $T_l;$
it could alternatively use a Kronecker structured root decomposition if not diagonal.

We then need to compute 
\begin{align*}
	w = (K_{XX} + T_l)^{-1}(y - Y - \epsilon),
\end{align*}
via efficient Kronecker solves as described in Section \ref{app:kron} --- for example, if $T_l$ is a constant diagonal, use the Kronecker eigen-decomposition and add the constant to the eigenvalues.
The diagonal solves generally cost $\mathcal{O}(n^3 + t^3 + nt(n+t))$ time, while even $T_l = \Sigma_{TT} \otimes \Sigma_{NN},$ full rank task and data noises, still costs $\mathcal{O}((n^3 + t^3) + nt(n+t))$ as we only need to perform extra matrix multiplications \citep{rakitsch2013noise}.
Finally, we only need to compute a Kronecker matrix vector multiplication, computing $z = K_{x_{\text{test}} X} w$ and $\bar{f} = f + z.$
This exploits Kronecker identities and costs $\mathcal{O}(nt(n+t)).$

We choose to mention the $nt$ terms for precision, despite it typically being dropped in the literature due to the solve costs being dominated by the eigen-decompositions of training and task covariance matrices \citep{rakitsch2013noise,higdon2008computer,bonilla_multi-task_2007}.
In all cases, the $nt$ terms come solely from the Kronecker matrix vector multiplications.
The overall time complexity of the operations is then $\mathcal{O}(n^3 + t^3 + nt(n+t))$ time, which is $\mathcal{O}(n^3 + t^3)$ time.

The non-zero mean case can be implemented by adding in the mean function into the joint sample at \eqref{eq:app_joint} and again at the end of \eqref{eq:matherons_rule}.

The extension to the HOGP model proceeds like the general ICM case if we replace $L$ by the root decomposition of the kernels across all tensors, $L = \otimes_{i=2}^d L_i$ such that $\otimes_{i=2}^d K_i = \otimes_{i=2}^d L_i L_i^\top.$
Again, we only need to update the root decomposition on the data covariance and can re-use the root decomposition on the latent covariances.

Overall memory complexities are shown in Table \ref{tab:memory}; we ignore the fixed constant train decomposition costs of $K_{XX}$ and/or $K_{XX} + \sigma^2 I.$
For single output GPs, this is a constant $\mathcal{O}(n^2)$ ($\mathcal{O}(nr)$ if LOVE is used).
For multi-task GPs, it becomes $\mathcal{O}(n^2 + t^2 + nt)$ (or $\mathcal{O}(ntr)$).
For the HOGP, it is $\mathcal{O}(\sum_{i=1}^k d_i^2 + \prod_{i=1}^k d_i)$ (or $\mathcal{O}(r(\sum_{i=1}^k d_i)).$
The multiplicative scaling in memory is the cost of a single vector (the eigenvalues of $K_{XX}$) for the HOGP and the MTGPs. Unfortunately, combining Lanczos partial eigen-decompositions does not help reduce the memory by as much in the MTGP or HOGP setting due to the necessity of some zero-padding.

\begin{table*}[h!]
	\centering
	\caption{Memory complexities after pre-computation for posterior sampling in single-output, multi-task, and high-order (HOGP) Gaussian Process models. Matheron's rule allows decomposition across the Kronecker product of the train and task covariances, enabling significant improvements in memory scaling. We ignore pre-computation costs, while the multiplicative terms are single vectors.
	}
	\begin{tabular}{c|c|c}
		\toprule Model & \multicolumn{1}{c}{Distributional (Standard) \eqref{eq:dist_sampling}} & \multicolumn{1}{|c}{\textbf{With Matheron's rule  \eqref{eq:matheron_gp}}} \\\toprule
		Single-Output & $n_{\text{test}}^2$  & $n_{\text{test}}^2 + nn_\text{test}$ \\ \midrule
		Multi-Task & $(n_{\text{test}}t)^2$ & \textcolor{blue}{$n_{\text{test}}^2 + t^2 + nt$} \\ \midrule
		HOGP &$(n_{\text{test}}^2)\prod_{i=2}^d d_i^2$  &
		$\textcolor{blue}{n_{\text{test}}^2 + \sum_{i=2}^d d_i^2 + \prod_{i=2}^d d_i}$ \\\bottomrule
	\end{tabular}
	\label{tab:memory}
\end{table*}

Finally, time complexities when using Lanczos decompositions throughout are shown in Table \ref{tab:compute_love}, with the corresponding memory requirements after pre-computation shown in Table \ref{tab:memory_love}.
These present further improvements to the Cholesky based approaches described throughout and enable Matheron's rule sampling with MTGPs to scale to larger $n$ and larger $t$ than even the exact settings.

\begin{table*}[h!]
	\centering
	\caption{Time complexities for posterior sampling in single-output, multi-task, and high-order (HOGP) Gaussian Process models with LOVE fast predictive variances and Lanczos decompositions of rank $r$. \textcolor{blue}{Time complexities shown in blue} are our contributions that have not yet been considered by the literature. Standard Sampling multi-task Gaussian processes scales multiplicatively in the combination of the number of tasks, $t,$ and the number of data points, $n,$ while using Matheron's rule allows for structure exploitation that reduces the combination to become additive in these components. 
	}
	\begin{small}
		\begin{tabular}{c|c|c}
			\toprule Model & \multicolumn{1}{c}{Distributional (Standard) \eqref{eq:dist_sampling}} & \multicolumn{1}{|c}{\textbf{With Matheron's rule  \eqref{eq:matheron_gp}}} \\\toprule
			Single-Output  &$\mathcal{O}(r(n^2 + n_{\text{test}}^2))$  &$\mathcal{O}(r(n^2 + n_{\text{test}}^2))$ \\ \midrule
			Multi-Task 
			& $\mathcal{O}(rt^2(n^2 + n_{\text{test}}^2))$  &$\textcolor{blue}{\mathcal{O}(r((n^2 + n_{\text{test}}^2) + t^2)}$ \\\midrule
			HOGP & --- &
			$\textcolor{blue}{\mathcal{O}(r((n^2 + n_{\text{test}}^2) + \sum_{i=2}^d d_i^2))}$\\\bottomrule
		\end{tabular}
	\end{small}
	\label{tab:compute_love}
\end{table*}

\begin{table*}[h!]
	\centering
	\caption{Memory complexities after pre-computation for posterior sampling in single-output, multi-task, and high-order (HOGP) Gaussian Process models when using LOVE posterior covariances and Lanczos decompositions of rank $r$. Matheron's rule allows decomposition across the Kronecker product of the train and task covariances, enabling significant improvements in memory scaling.
	}
	\begin{tabular}{c|c|c}
		\toprule Model & \multicolumn{1}{c}{Distributional (Standard) \eqref{eq:dist_sampling}} & \multicolumn{1}{|c}{\textbf{With Matheron's rule  \eqref{eq:matheron_gp}}} \\\toprule
		Single-Output  &  $n_{\text{test}}r$  &  $n_{\text{test}}r + + nn_\text{test}$ \\ \midrule
		Multi-Task  & $(n_{\text{test}}t)r$ &  \textcolor{blue}{$r(n_{\text{test}} + t)+ nt$}\\\midrule
		HOGP  & --- &
		$\textcolor{blue}{r( n_{\text{test}} + \sum_{i=2}^d d_i)+ \prod_{i=2}^d d_i}$\\\bottomrule
	\end{tabular}
	\label{tab:memory_love}
\end{table*}

\subsubsection{Extension to Linear Model of Coregionalization}
\label{appdx:extension_LMC}
We close this section by noting that the linear model of coregionalization (LMC) as described in \citet{alvarez2012kernels} can be written as a sum of Kronecker products: $K_{\text{train}} = \sum_{q=1}^Q B_q \otimes K_q(X, X').$
We do not know of an efficient solve of sums of more than two Kronecker products, and we do not have a strong implementation of approximation methods or specialized preconditioners for solves of the form $(K_{\text{train}} + T_l)^{-1} z.$
However, exploiting Kronecker strucutre, matrix vector products $K_{\text{train}} v$ cost $\mathcal{O}(Q nt (n +t))$ so that we can use conjugate gradients to compute solves and Lanczos to compute root decompositions in $\mathcal{O}(r Q nt (n + t))$ time and $\mathcal{O}(r nt)$ memory.
We can similarly compute a dense root decomposition update to form $MM^\top \approx \Ktt$ by following the same strategy as before (e.g. root updates \citep{jiang_efficient_2020}) but on matrices of size $nt$ and with an update of size $n_\text{test} t.$
The structured MVMs make the updates more efficient, as computing $L_{22}$ costs only $rQn_{\text{test}}t(n_\text{test} + t)$ and $L_{12}$ costs $rQnt(n_\text{test} + t)$ to form the dense $r \times (n+n_\text{test})t$ updated root decomposition $M.$
Thus, we can achieve efficient sampling using Matheron's rule and Lanczos variance estimates in effectively $\mathcal{O}(rQt(n^2+n_\text{test}^2))$ time.

By comparison, sampling using the distributional approach would require dense factorizations of non-structured matrices, e.g. $\textcolor{blue}{\Sigma^*},$ that do not have fast MVMs, thereby proving to be more expensive both computationally and memory wise.
Indeed, the advantages will be magnified for large $n_{\text{test}}$ as then forming and decomposition $\textcolor{blue}{\Sigma^*}$ may quickly become too expensive memory wise.
We can exploit structured MVMs for sampling using Matheron's rule.
Implementation wise, this provides $3x$ speedups on a single GPU, while significantly improving the memory overhead; however, we leave detailed exploration for future work.

\subsection{Autokrigeability and the HOGP Model}\label{app:hogp_extension}

The High-Order Gaussian Process model has latent parameters, $x_l,$ for each latent dimension, so that $K_i = k(x_i, x_i).$
\citet{zhe_scalable_2019} initialize $x_l \sim \mathcal{N}(0, I)$ and optimize them as nuisance hyper-parameters, possibly regularizing them with weight decay.
For completeness, they consider multi-dimensional latent dimensions, e.g. $x_l$ is a matrix, while we consider here $x_l$ as only one dimensional. Our analysis holds for multi-dimensional latents.

In the noiseless limit, the HOGP model falls prey to autokrigeability as described by \citet{bonilla_multi-task_2007}.
If we were predicting a vector, this would not be an issue; however, we are predicting a tensor.
In general, we expect the tensor's dimensions to be smoothly varying --- that is, as we move down a row, we expect the covariance to be smoothly varying (e.g. it has smooth spatial structure).
This prior assumption can be demonstrated on the variance as shown in Figure \ref{fig:tensor_variance} for simulated data: as we move down rows and columns, the sample variance stays at least somewhat consistently high (the first column) or low (the tenth column).

Considering $n = 1$ test point, only one latent dimension, and then taking the limit as $\sigma^2 \rightarrow 0,$ the posterior variance becomes
\begin{align*}
	\Sigma &:= K_{x_{\text{test}} x_{\text{test}}} \otimes K_L - (K_{x_{\text{test}} x} \otimes K_L)(K_{XX}^{-1} \otimes K_L^{-1})(K_{X x_{\text{test}}} \otimes K_L) \\
	&= (K_{x_{\text{test}} x_{\text{test}}} - K_{x_{\text{test}} X}K_{XX}^{-1} K_{X x_{\text{test}}}) \otimes K_L \\
	&= a K_L,
\end{align*}
with $a = (K_{x_{\text{test}} x_{\text{test}}} - K_{x_{\text{test}} X}K_{XX}^{-1} K_{X x_{\text{test}}})$ (a scalar). 
The posterior variances for each output are given by the diagonals of $K_L$ or of the diagonal of $\otimes_{i=2}^d K_{i}$ for a $d$-tensor.
The covariances between outputs are given by $K_L.$

The trouble arises from the fact that if the latent parameters, $v_l,$ are not smoothly varying across $l$ then $K_L$ will not be smoothly varying either.
A priori, we might expect that for a given set of outputs in the tensor, say indices $0, 1, 2$, that their posterior variances would also be smoothly varying (as the underlying process across the tensor is ``smooth" in some sense), shown in Figure \ref{fig:tensor_variance}.
Referring back to Figure \ref{fig:hogp_coverage_maps} for intuition, by smooth, we mean that each pixel in the outcome maps is close in some sense to its nearest eight neighbors.
We should then expect the model's predictive posterior variances to vary in a similar fashion to the model's predictive posterior, which is data dependent.

For the HOGP tensor, the predictive posterior variance over an entire tensor (e.g. the coverage maps with $3$ dimensions) is given by the product of the diagonal of each posterior.
Thus, the inter-latent relationships can very quickly produce a ``jagged" posterior variance as shown in Figure \ref{fig:random_variance}, with the posterior covariance becoming even more highly patterned (Figure \ref{fig:random_covariances}).

For small $\sigma^2,$ we can approximate $\Sigma$ as $a K_L$ with $a = (K_{x_{\text{test}} x_{\text{test}}} - K_{x_{\text{test}} X}(K_{XX}^{-1} + \sigma^2 K_{XX}^{-2}) K_{X x_{\text{test}}})$ and consider the properties of $K_L.$ For smoothly varying covariances between indices in the posterior covariance matrix, we want smoothly varying $K_L$ and thus smoothly varying $x_l.$

\paragraph{Smoothly Varying Latent Dimensions: the HOGP + GP}
To produce smoothly varying latent dimensions, we initialize $x_l$ as a random draw from a multivariate normal distribution (or a Gaussian process) with zero mean and with a Matern $2.5$ kernel and lengthscale $1$ in all of our experiments.
The kernel is evaluated on $(0, 1./d_i, 2./d_i, \cdots, (d_i - 1)/d_i, 1.)$ for inputs.
We then use this distribution as a prior on the $x_l$ latents as well to help produce smoothly varying latents.

Example prior draws are shown in Figure \ref{fig:hogp_latents} in orange for two of its latent dimensions. 
The induced posterior variances are shown in Figure \ref{fig:hogp_variances}, which are considerably more smoothly varying than the random initializations.
It is somewhat closer to the true variances of the function (these are un-trained models).
Similarly, the (squeezed) posterior covariance matrix as shown in Figure \ref{fig:hogp_covariances} shows much less covariance patterning than the random covariances in Figure \ref{fig:random_covariances}.
Importantly, after training the model, the largest impact is not on the predictive mean, but rather the posterior covariances and thus the posterior samples.
We refer to HOGP models with this type of latent dimension prior and initialization as the HOGP + GP in the main text.
We leave a theoretical analysis of these priors to future work.
In Figure \ref{fig:hogp_prior_choice}, the true latent function is $f(x, y) = \sin(2 x * i) * \cos(0.4 y j) + \epsilon,$
where $i,j$ are the tensor indices (here, $(0, 31)^2$) and $\epsilon \sim \mathcal{N}(0, 0.01^2).$

\begin{figure}[t!]
	\centering
	\begin{subfigure}{0.32\textwidth}
		\includegraphics[width=\linewidth]{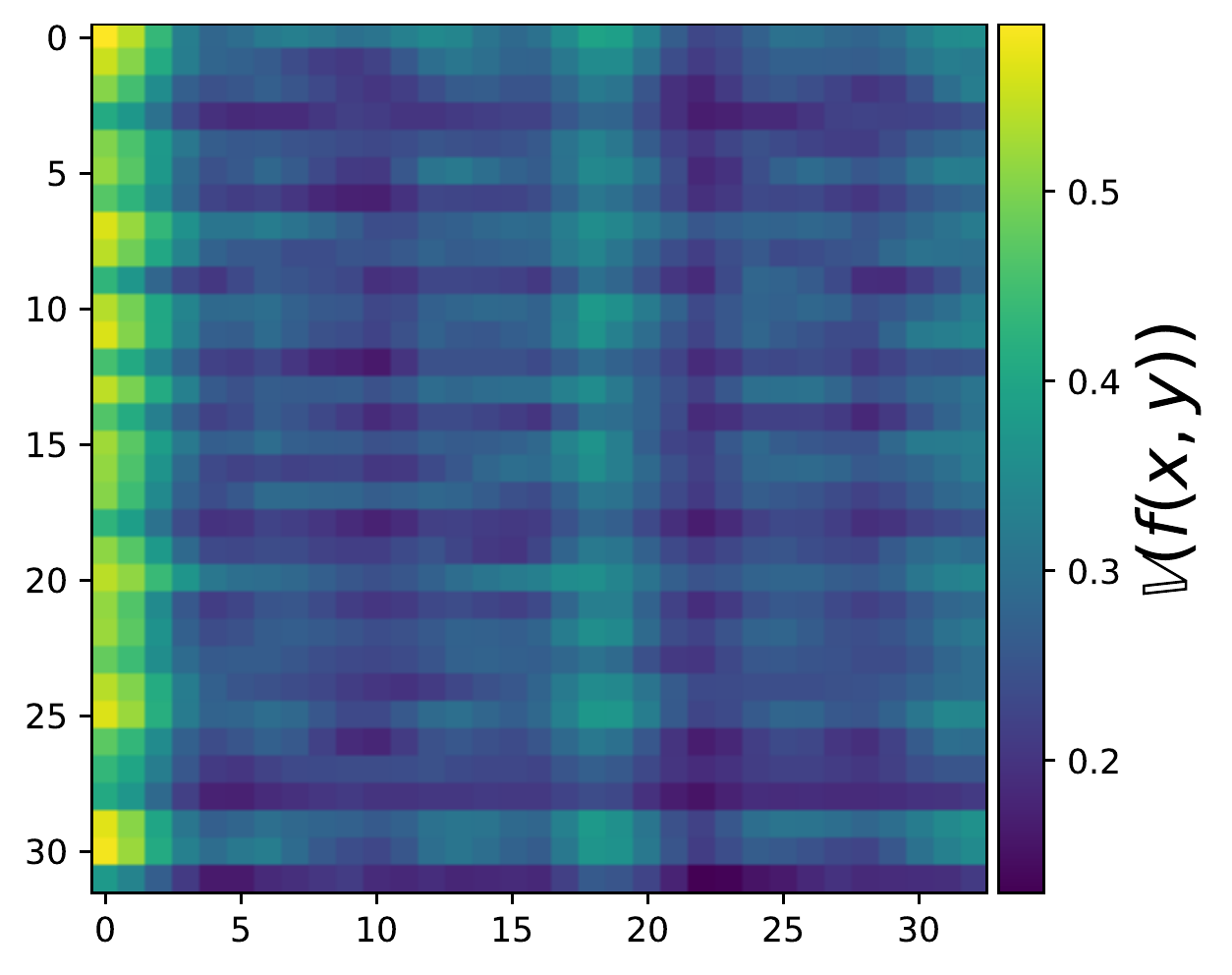}
		\caption{Variance of function}
		\label{fig:tensor_variance}
	\end{subfigure}
	\hfill
	\begin{subfigure}{0.32\textwidth}
		\includegraphics[width=\linewidth]{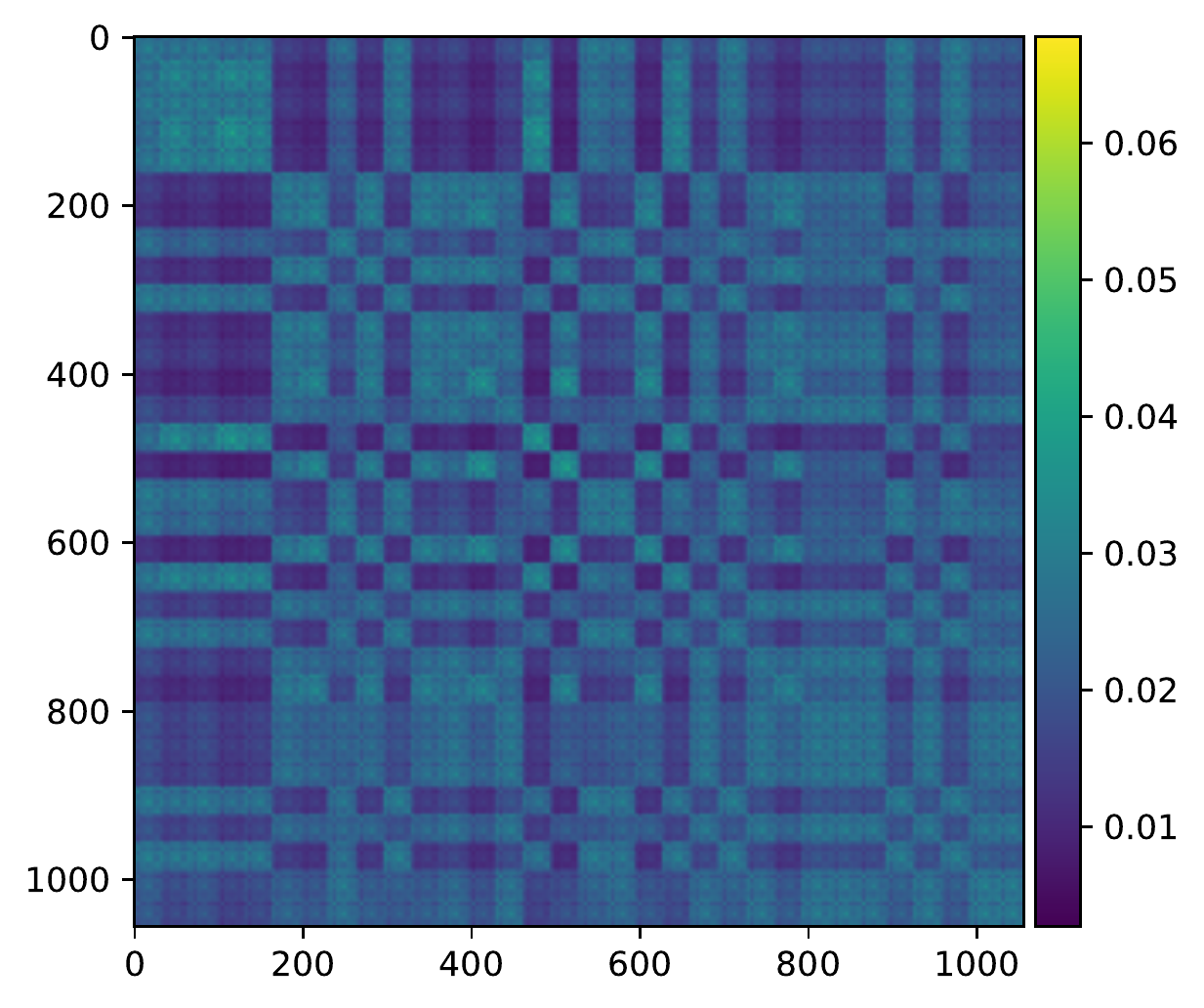}
		\caption{Random covariance}
		\label{fig:random_covariances}
	\end{subfigure}
	\hfill
	\begin{subfigure}{0.32\textwidth}
		\includegraphics[width=\linewidth]{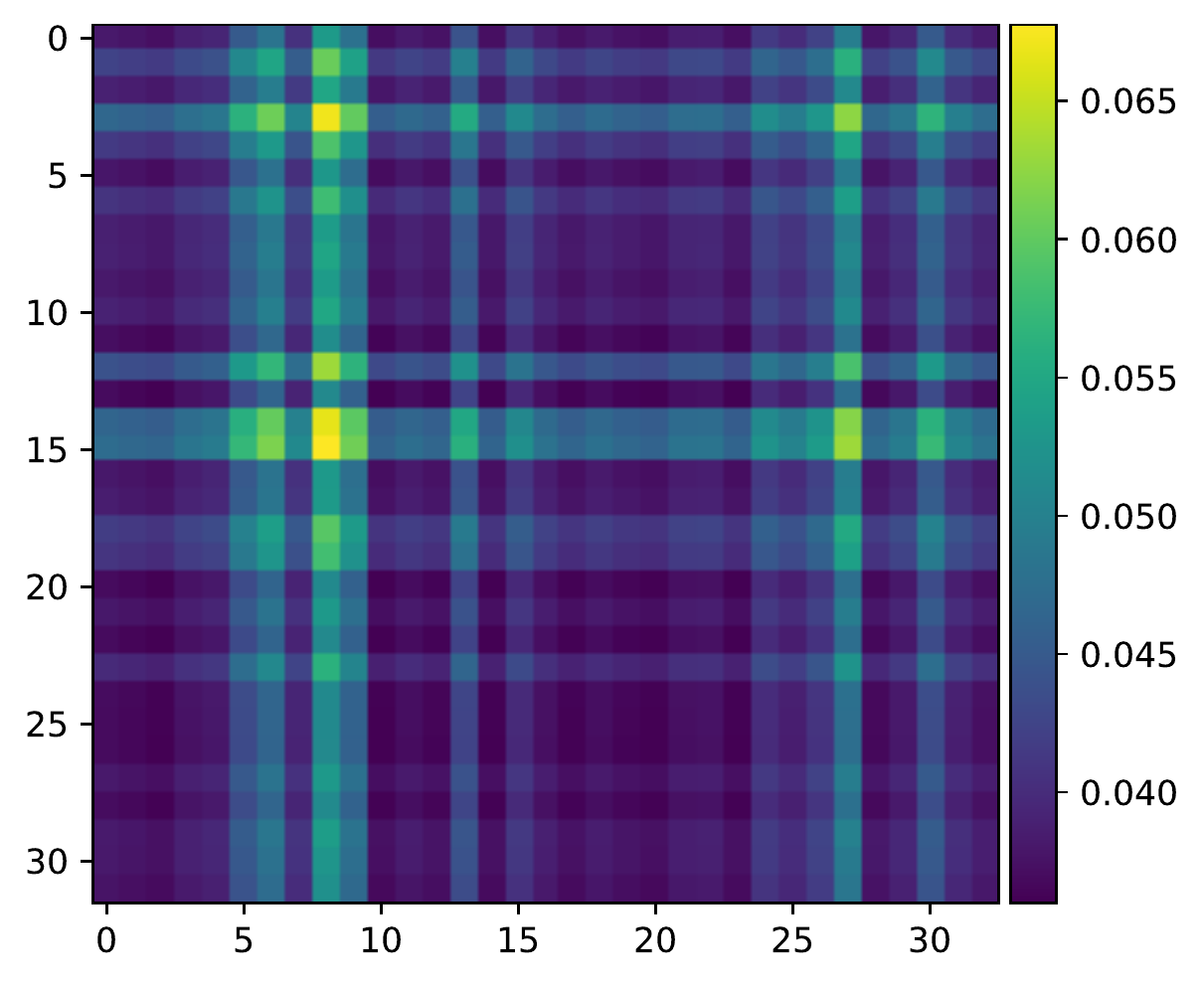}
		\caption{Random variance}
		\label{fig:random_variance}
	\end{subfigure}
	\hfill
	\begin{subfigure}{0.32\textwidth}
		\includegraphics[width=\linewidth]{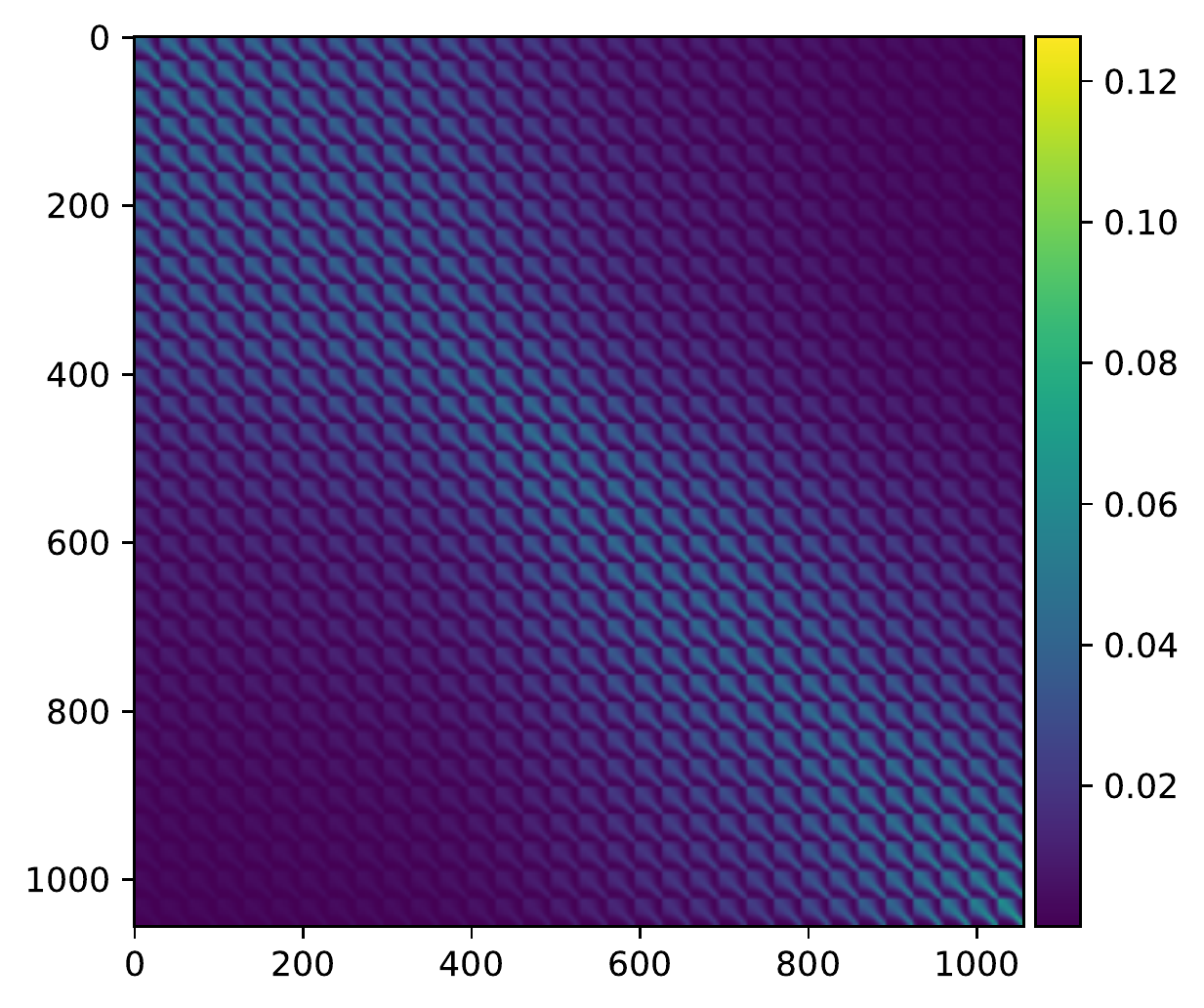}
		\caption{GP covariance}
		\label{fig:hogp_covariances}
	\end{subfigure}
	\hfill
	\begin{subfigure}{0.32\textwidth}
		\includegraphics[height=4cm]{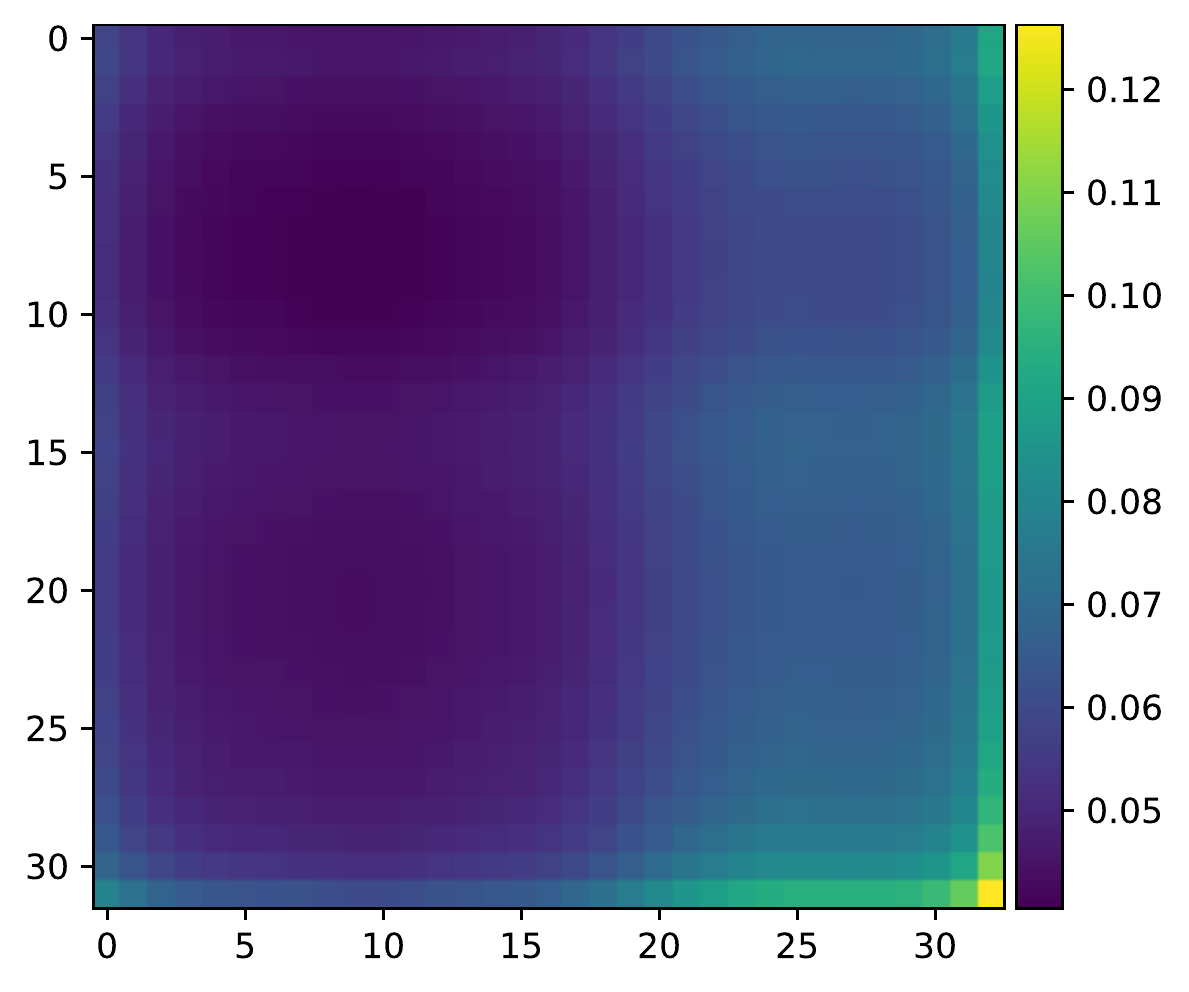}
		\caption{GP variance}
		\label{fig:hogp_variances}
	\end{subfigure}
	\hfill
	\begin{subfigure}{0.32\textwidth}
		\includegraphics[height=4cm]{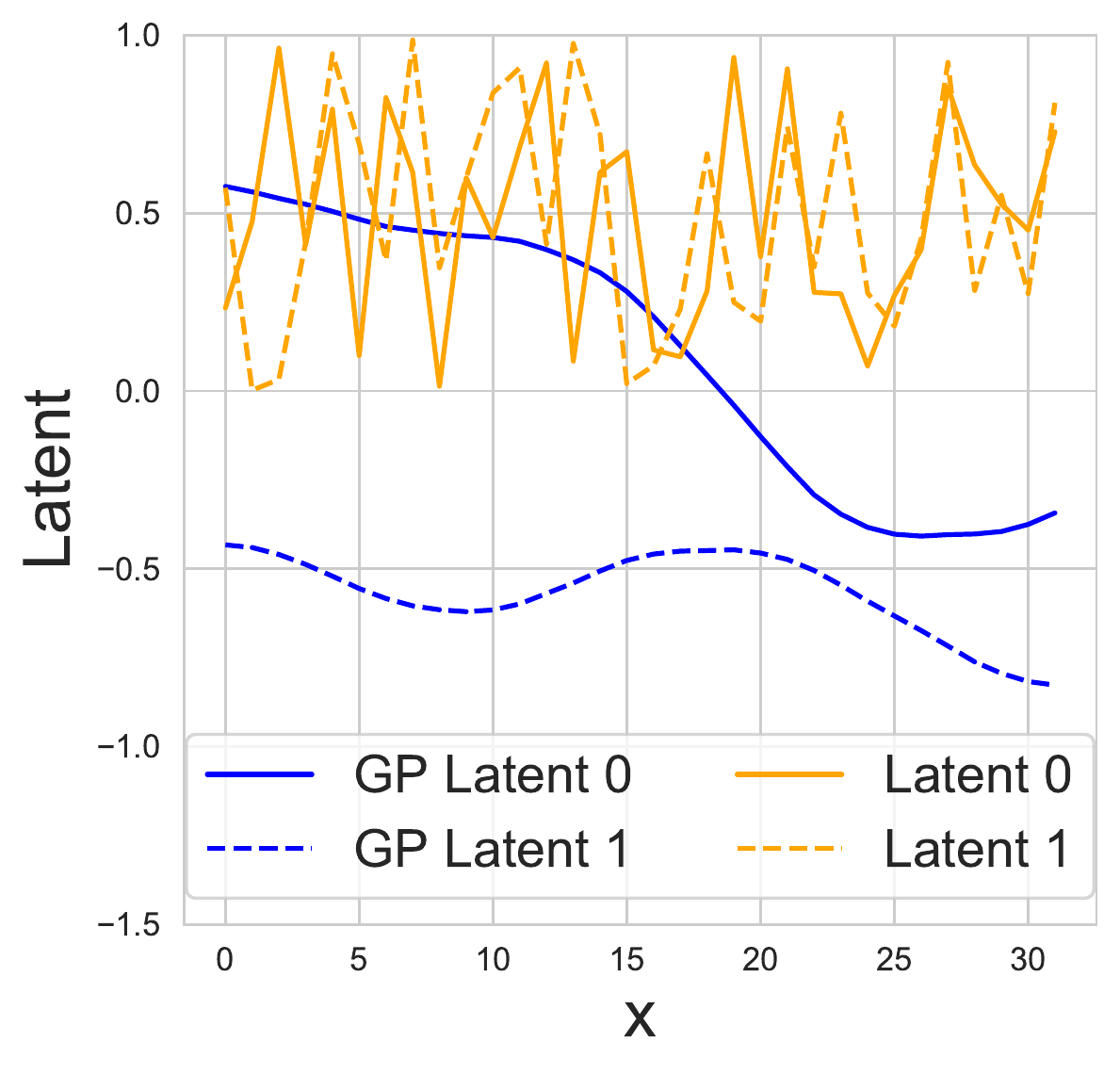}
		\caption{Initialized latents}
		\label{fig:hogp_latents}
	\end{subfigure}
	\hfill
	\caption{\textbf{(a)} Sample variance of true function draws over the indices. \textbf{(b)} Posterior covariance from a HOGP model initialized with random latent dimensions shown as orange lines in \textbf{(f)}. \textbf{(c)} Posterior variance, shown on the output dimensions for random latents, the variance varies jaggedly. \textbf{(d)} Posterior covariance from a HOGP model initialized with latent dimensions drawn from a GP, shown as blue lines in \textbf{(f)}. \textbf{(e)} Posterior variance, shown on the output dimensions for the GP latents; the variance now varies smoothly. \textbf{(c,e)} Flattened (each output is now a single pixel) posterior covariance for random and GP-drawn latents. The GP-drawn latent covariance is much more smoothly varying.}
	\label{fig:hogp_prior_choice}
\end{figure}

\subsection{Convergence of Sample Average Approximation of Monte-Carlo Acquisition Functions}
\label{app:convergence}

Following \citet{balandat_botorch_2020}, we consider the following class of acquisition functions: 
\begin{align}
	\label{eq:app:convergence:acq}
	\alpha(x; \Phi, y) = \mathbb{E} \bigl[ a(g(f(x)), \Phi) \,|\, Y=y\bigr], 
\end{align}
Here $x\in \mathbb{R}^{q\times d}$ is a set of $q$ candidate points, $g: \mathbb{R}^{n_{\text{test}} \times t} \rightarrow \mathbb{R}^{n_{\text{test}}}$ is an objective function, $\Phi \!\in\! \mathbf{\Phi}$ are parameters independent of~$x_{\text{test}}$ in some set $\mathbf{\Phi}$, and $a: \mathbb{R}^{n_{\text{test}}} \times \mathbf{\Phi} \!\rightarrow\! \mathbb{R}$ is a utility function that defines the acquisition function. 

Letting $\xi^i(x)$ denote a sample from $f(x)|(Y=y)$, we have the following Monte Carlo approximation of~\eqref{eq:app:convergence:acq}:
\begin{align}
	\label{eq:app:convergence:acq_mc}
	\hat\alpha_N(x; \Phi, y) := \frac{1}{N} \sum_{i=1}^N a(g(\xi^i(x)), \Phi) 
\end{align}
Suppose $\mathbb{X} \subset \mathbb{R}^d$ is a feasible set (the ``search space''). Let $x^{\text{opt}} := \argmax_{x \in \mathbb{X}^{n_{\text{test}}}} \alpha(x; \Phi, y)$ and denote by $\mathcal{X}^{\text{opt}}$ the associated set of maximizers. Similarly, let $\hat{x}^{\text{opt}}_N := \argmax_{x \in \mathbb{X}^{n_{\text{test}}}} \hat\alpha_N(x; \Phi, y)$. Then we have the following:

\begin{proposition}
	\label{prop:app:convergence}
	Suppose that $\mathbb{X}$ is compact, $f$ has a GP prior with continuously differentiable mean and covariance functions, and $g(\cdot)$ and $a(\cdot\, \Phi)$ are Lipschitz continuous. 
	If the base samples $\{\nu^i\}_{i=1}^{n + n_{\text{test}}}$ and $\{\epsilon^i\}_{i=1}^n$ are i.i.d. with $\nu^i \sim \mathcal{N}(0, 1)$ and  $\epsilon^i \sim \mathcal{N}(0, \sigma^2)$, respectively, then 
	\begin{enumerate}
		\item $\hat\alpha_N^{\text{opt}} \rightarrow \alpha^{\text{opt}}$ a.s.
		\item $\text{dist}(\hat{x}_N^{\text{opt}}, \mathcal{X}^{\text{opt}}) \rightarrow 0$ a.s.
	\end{enumerate}
\end{proposition}

To prove Proposition~\ref{prop:app:convergence}, we need the following intermediate result:

\begin{lemma}
	\label{lemma:app:HDMLipschitz}
	Under Matheron sampling, $a(g(\xi(x)), \Phi) = a(g(h(x, \tilde{\epsilon})))$ with $h: \mathbb{R}^{n_\text{test}\times d} \times \mathbb{R}^{2n+n_\text{test}}$ and $\tilde{\epsilon} \in \mathbb{R}^{2n+n_\text{test}}$ a random variable. Moreover, there exists an integrable function $\ell: \mathbb{R}^{2n+n_\text{test}} \rightarrow \mathbb{R}$ such that for almost every $\tilde\epsilon$ and all $x, x' \in \mathbb{X}$, 
	\begin{align}
		\label{eq:lemma:app:HDMLipschitz}
		\bigl| a(g(h(x, \tilde{\epsilon}))) - a(g(h(x', \tilde{\epsilon}))) \bigr| \leq \ell(\tilde\epsilon)\|x - x'\|.
	\end{align}
\end{lemma}

\begin{proof}[Proof of Lemma~\ref{lemma:app:HDMLipschitz}]
	From \eqref{eq:matheron_gp} we have that under Matheron sampling a posterior sample is parameterized as 
	\begin{align}
		\label{eq:app:convergence:parameterization}
		\xi(x) = f  + K_{xX}(K_{XX}+\sigma^2I)^{-1}y -   (K_{XX}+\sigma^2I)^{-1}(Y + \epsilon)
	\end{align}
	where $(f, Y)$ are joint samples from the GP prior.  
	We parameterize $(f, Y)$ as $(f, Y) = R(x) \nu$ where $R(x)$ is a root\footnote{For simplicity we assume that $R(x) \in \mathbb{R}^{n+n_\text{test}}$ for all $x$, but the results also apply to lower-rank roots (the argument follows from simple zero-padding.} of the covariance from \eqref{eq:joint_prior}, and $\nu \sim \mathcal{N}(0, I)$. We can thus write~\eqref{eq:app:convergence:parameterization} as 
	\begin{align}
		\xi(x) &= \left[\begin{array}{@{}cc|c@{}}
			I & - M(x) & 0 \\\hline
			0 & 0 & M(x)
		\end{array}\right]
		\left[\begin{array}{@{}c@{}}
			R(x) \nu \\\hline
			y- \epsilon^i
		\end{array}\right] 
		= M(x)y + A(x)     \tilde\epsilon
	\end{align}
	where $M(x) := K_{xX} (K_{XX} + \sigma^2 I)^{-1}$, 
	\begin{align*}
		A(x) := \left[\begin{array}{@{}cc|c@{}}
			I & - M(x) & 0 \\\hline
			0 & 0 & -\sigma M(x)
		\end{array}\right]
		\left[\begin{array}{@{}c|c@{}}
			R(x) & 0 \\\hline
			0 & I
		\end{array}\right]
	\end{align*}
	and $\tilde\epsilon_j \sim \mathcal{N}(0, 1)$ for $j=1, \dotsc, 2n + n_\text{test}$. 
	Therefore, 
	\begin{align*}
		\| \xi(x) - \xi(x')\| &\leq \|M(x) - M(x')\| y + \left\| A(x) - A(x')
		\right\| \|\tilde\epsilon \|.
	\end{align*}
	From the arguments of \citet{balandat_botorch_2020}, the assumption of continuously differentiable mean and covariance functions and the compactness of $\mathbb{X}$ imply that there exist $C_M, C_A < \infty$ s.t. $\|M(x) - M(x')\| \leq C_M$ and $\|A(x) - A(x')\| \leq C_A$ for all $x, x' \in \mathbb{X}$. Moreover, since both $g(\cdot)$ and $a(\cdot\, \Phi)$ are Lipschitz, there exists $L < \infty$ such that $\bigl| a(g(h(x, \tilde{\epsilon}))) - a(g(h(x', \tilde{\epsilon}))) \bigr| \leq L \| \xi(x) - \xi(x')\|$. 
	Consequently, \eqref{eq:lemma:app:HDMLipschitz} holds with $\ell(\tilde\epsilon) = LC_My + LC_A \tilde\epsilon$, which is integrable since (i) $y$ is almost surely finite and $\tilde{\epsilon}_j\sim \mathcal{N}(0, 1)$.
\end{proof}

\begin{proof}[Proof of Proposition~\ref{prop:app:convergence}]
	Lemma~\ref{lemma:app:HDMLipschitz} mirrors Lemma~1 from the supplementary material of \citet{balandat_botorch_2020}, and shows the corresponding result for the posterior samples parameterized under sampling using Matheron's rule. Proposition~\ref{prop:app:convergence}  then follows from the same arguments as in the proof of Theorem~1 in \citet{balandat_botorch_2020}.
\end{proof}

Similar to \citet{balandat_botorch_2020}, it is also possible to show the following:

\begin{proposition}
	\label{prop:app:convergence_rate}
	If, in addition to the assumptions of Proposition~\ref{prop:app:convergence}, (i) for all $x \in \mathbb{X}^{n_\text{test}}$ the moment generating function $t \mapsto \mathbb{E}[ e^{t a(g(h(x, \tilde{\epsilon})))}]$ is finite in an open neighborhood of $t=0$, and (ii) the moment generating function $t \mapsto \mathbb{E}[ e^{t\ell(\epsilon)}]$ is finite in an open neighborhood of $t=0$, then
	$\forall\, \delta>0$, $\exists\, K <\infty$, $\beta > 0$ s.t.  $\mathbb{P}\bigl(\textnormal{dist}(\hat{x}_{\!N}^\text{opt}, \mathcal{X}^\text{opt}) > \delta \bigr) \le K e^{-\beta N}$ for all $N \geq 1$.
\end{proposition}

This follows from the proof of Proposition~\ref{prop:app:convergence}; we can use exactly the same argument as in the proof of Theorem 1 in \citet{balandat_botorch_2020}. 

\section{Further Experiments and Experimental Details}\label{app:further_exps}\label{app:exp_details}

In this Appendix, we give further experimental details as well as some more experiments for the applications of Matheron's rule to various Bayesian Optimization tasks.
Experimental code is available at \url{https://github.com/wjmaddox/mtgp_sampler}.
Unless otherwise specified, all data is simulated.
The code primarily relies on PyTorch \citep{paszke2019pytorch} (MIT License), BoTorch \citep{balandat_botorch_2020} (MIT License), GPyTorch \citep{gardner_gpytorch_2018} (MIT License).

\subsection{Drawing Posterior Samples}

\begin{figure}[h!]
	\centering
	\includegraphics[width=\linewidth]{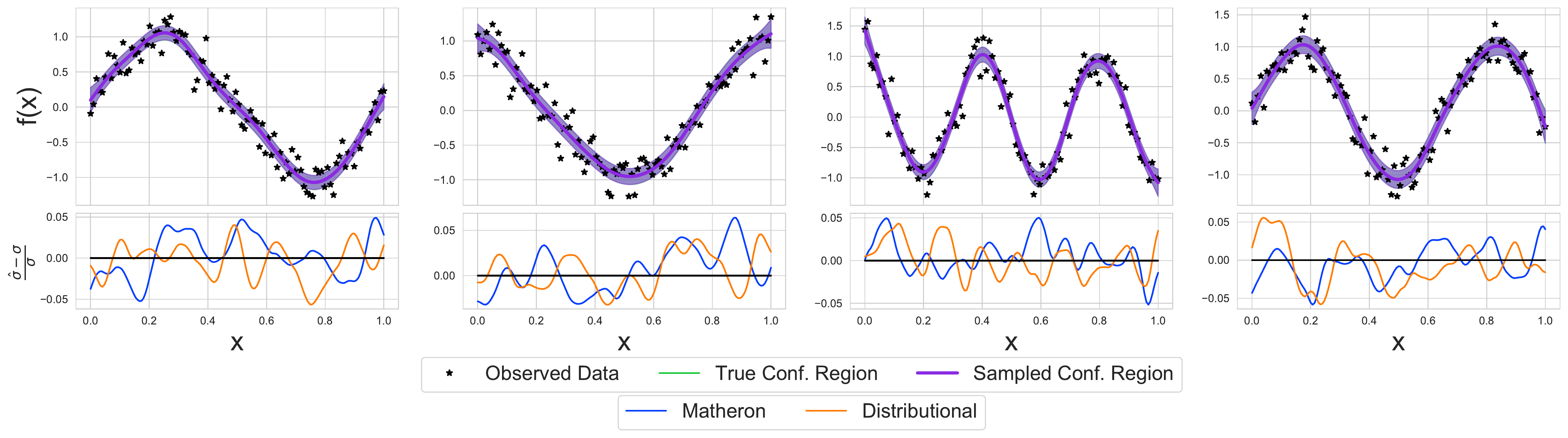}
	\caption{\textbf{Top:} Posterior mean and confidence regions for a two-task GP. The Matheron's rule sampled confidence region is overlaid on top of the true confidence region; these are visually indistinguishable. \textbf{Bottom:} Relative error of the estimated standard deviation from $1024$ samples drawn from the predictive posterior using either Matheron's rule or distributional sampling; plotted as a function of $\hat x.$ Again, the samples are effectively indistinguishable.}
	\label{fig:gp_sampling accuracy}
\end{figure}
In Figure \ref{fig:gp_sampling accuracy}, we show the accuracy of Matheron's rule sampling, where it is indistinguishable from conventional sampling from a MTGP \eqref{eq:dist_sampling} in terms of estimated standard deviations as well as the confidence regions. 
Here, we drew $1024$ samples from both sampling mechanisms and used the true mean and variance of the GP predictive posterior to shade the confidence regions.
This result is to be expected as Matheron's rule draws from exactly the same distribution as the predictive posterior.

\paragraph{Experimental Details}
For the plots in Figure \ref{fig:mt_speedup}, we fit a Kronecker multi-task GP with data from the Hartmann-$5D$ function as in \citet{feng_high-dimensional_2020} and Appendix \ref{app:cbo} for $25$ steps with Adam, before loading the state dict into the various implementations. 
We followed the same procedure for the LOVE experiments in Figure \ref{fig:mt_speedup_love}.

The GPU experiments were run sequentially over $10$ trials on a single $V100$ GPU with $16$GB of memory. We show the mean over the ten trails and two standard errors of the mean (on a \texttt{p3.2xlarge} AWS instance).
The CPU experiments were run sequentially over $6$ trails on a single CPU and given $128$GB of memory.
Specificially, this corresponds to a \texttt{c5d.18xlarge} AWS instance.

\begin{figure*}[h!]
	\centering
	\begin{subfigure}{0.24\textwidth}
		\centering
		\includegraphics[width=\linewidth]{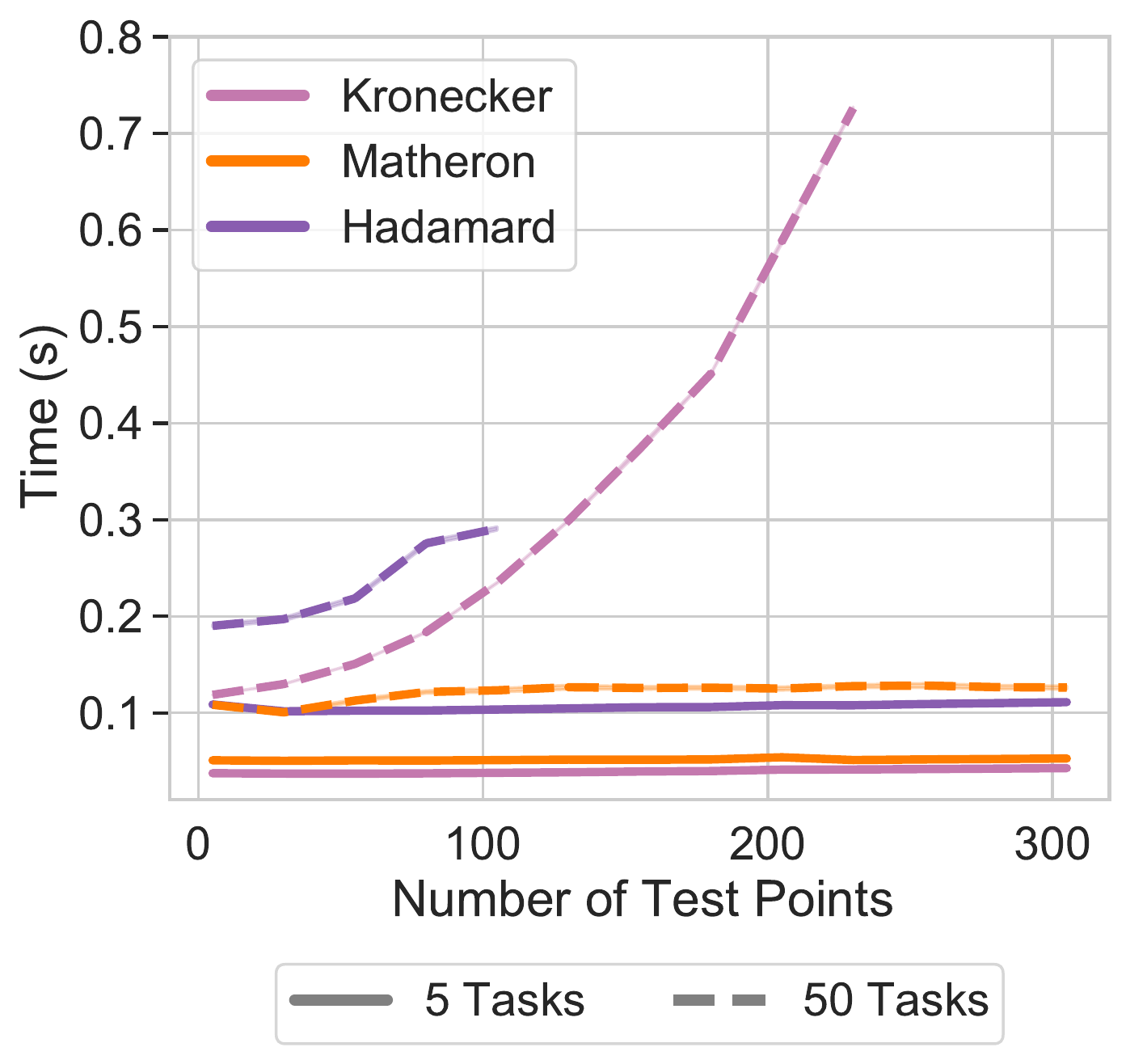}
		\caption{Test points, GPU}
		\label{fig:mt_speedup_test_love}
	\end{subfigure}
	\begin{subfigure}{0.24\textwidth}
		\centering
		\includegraphics[width=\linewidth]{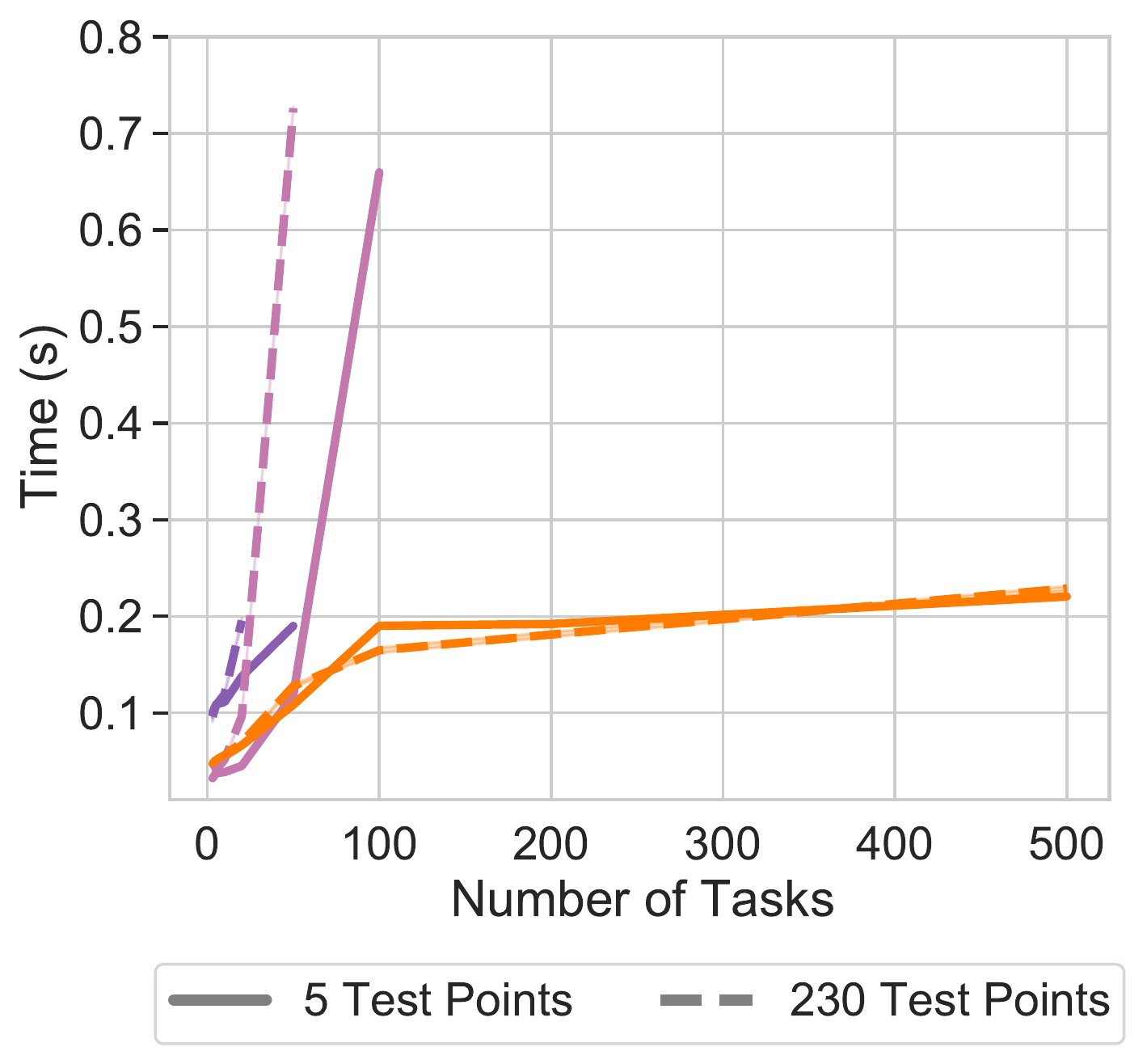}
		\caption{Tasks, GPU}
		\label{fig:mt_speedup_tasks_love}
	\end{subfigure}
	\begin{subfigure}{0.24\textwidth}
		\centering
		\includegraphics[width=\linewidth]{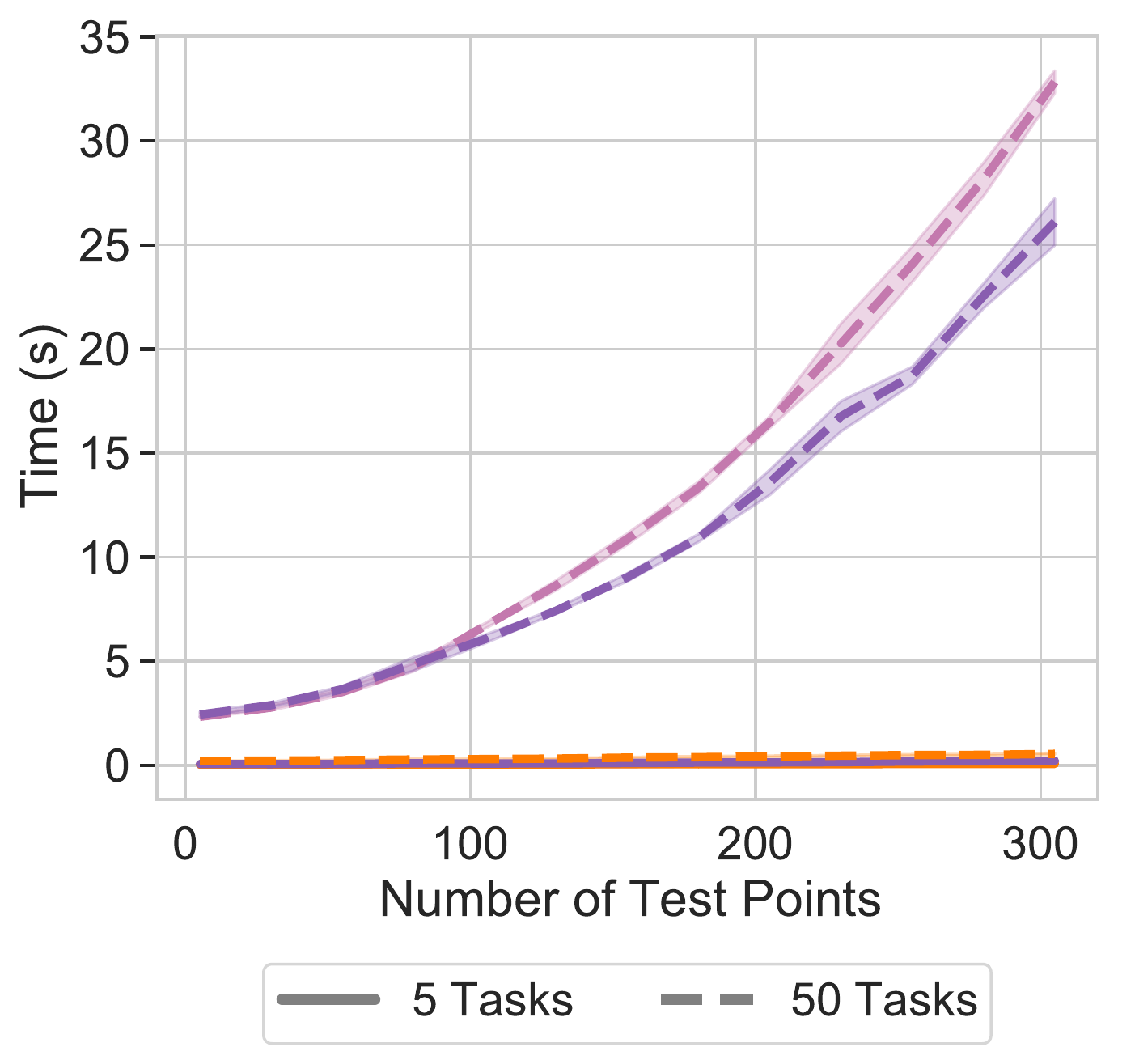}
		\caption{Test points, CPU}
		\label{fig:mt_speedup_test_cpu_love}
	\end{subfigure}
	\begin{subfigure}{0.24\textwidth}
		\centering
		\includegraphics[width=\linewidth]{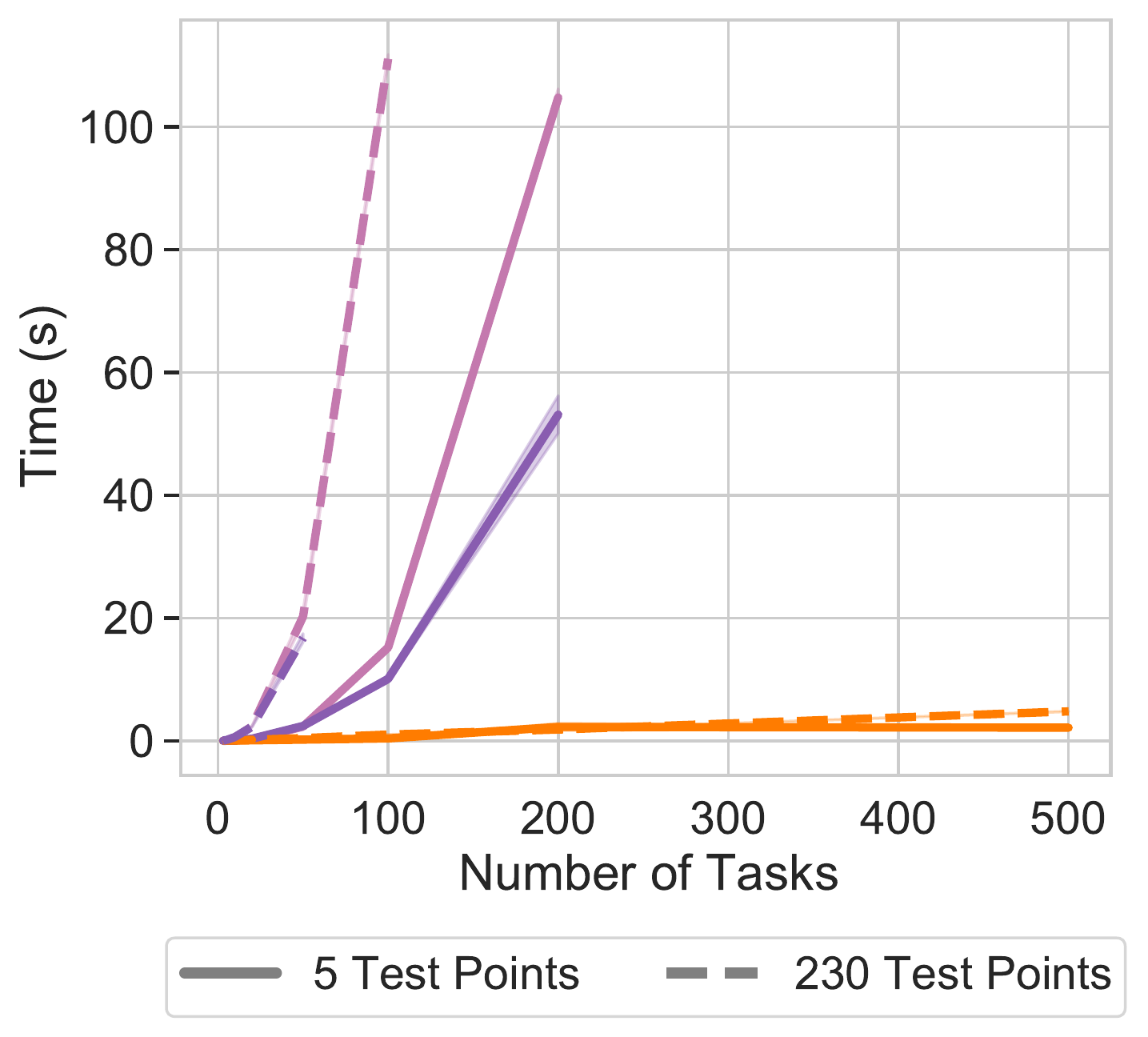}
		\caption{Tasks, CPU}
		\label{fig:mt_speedup_tasks_cpu_love}
	\end{subfigure}
	
	\caption{
		Timings for distributional sampling using LOVE cached predictive covariances with Hadamard (LCEM) and Kronecker (LCEM + Kronecker) variants of MTGPs as well as our sampling based on Matheron's rule (LCEM + Matheron) as the number of test points vary for fixed tasks \textbf{(a,c)} and as the number of tasks vary for fixed test points \textbf{(b,d)} on a single Tesla V100 GPU \textbf{(a,b)} and on a single CPU \textbf{(c,d)}.
		The multiplicative scaling of the number of tasks and data points creates significant timing and memory overhead, causing the Kronecker and Hadamard implementations to run out of memory for all but the smallest numbers of tasks and data points.
	}
	\label{fig:mt_speedup_love}  
\end{figure*}

\subsection{Multi-Task BO and Contextual BO with LCE-M Contextual Kernel}\label{app:cbo}
We now consider contextual BO (CBO), an extension of multi-task BO, following \citet{feng_high-dimensional_2020} and using their embedding-based kernels, the LCE-M model.
In CBO, the objective is to maximize some function over all observed contexts; given experimental conditions and a query, we want to choose the best action across all of the possible settings (the average in our experiments).
\citet{feng_high-dimensional_2020} considered contextual BO (CBO), and proposed the LCE-M model, an MTGP using an embedding-based kernel. LCE-M models each context as a task; all contexts are observed for any given input, and the multi-task kernel is then $k_n(x,x')k_t(i, j) = k_n(x,x')k_t(E(c), E(c')),$ where $E$ is a nonlinear embedding and $c$ are the contexts. 

We consider both multi-task and contextual versions of this problem. 
For the multi-task setting, we do not actually perform contextual optimization (selecting a different candidate for each context), but find a single action that maximizes the average outcome across all observed contexts. 
For both settings, we would intuitively expect that  using context-level observations can improve modelling and thus optimization performance. 
We use \texttt{qEI} (batch expected improvement with MC acquisition, $q=2$, \citet{balandat_botorch_2020}) on the objective and compare the (Hadamard) LCE-M model, a Kronecker variant, and one based on Matheron MTGP sampling.

\begin{figure}[t!]
	\centering
	\begin{subfigure}{0.32\linewidth}
		\includegraphics[width=\textwidth]{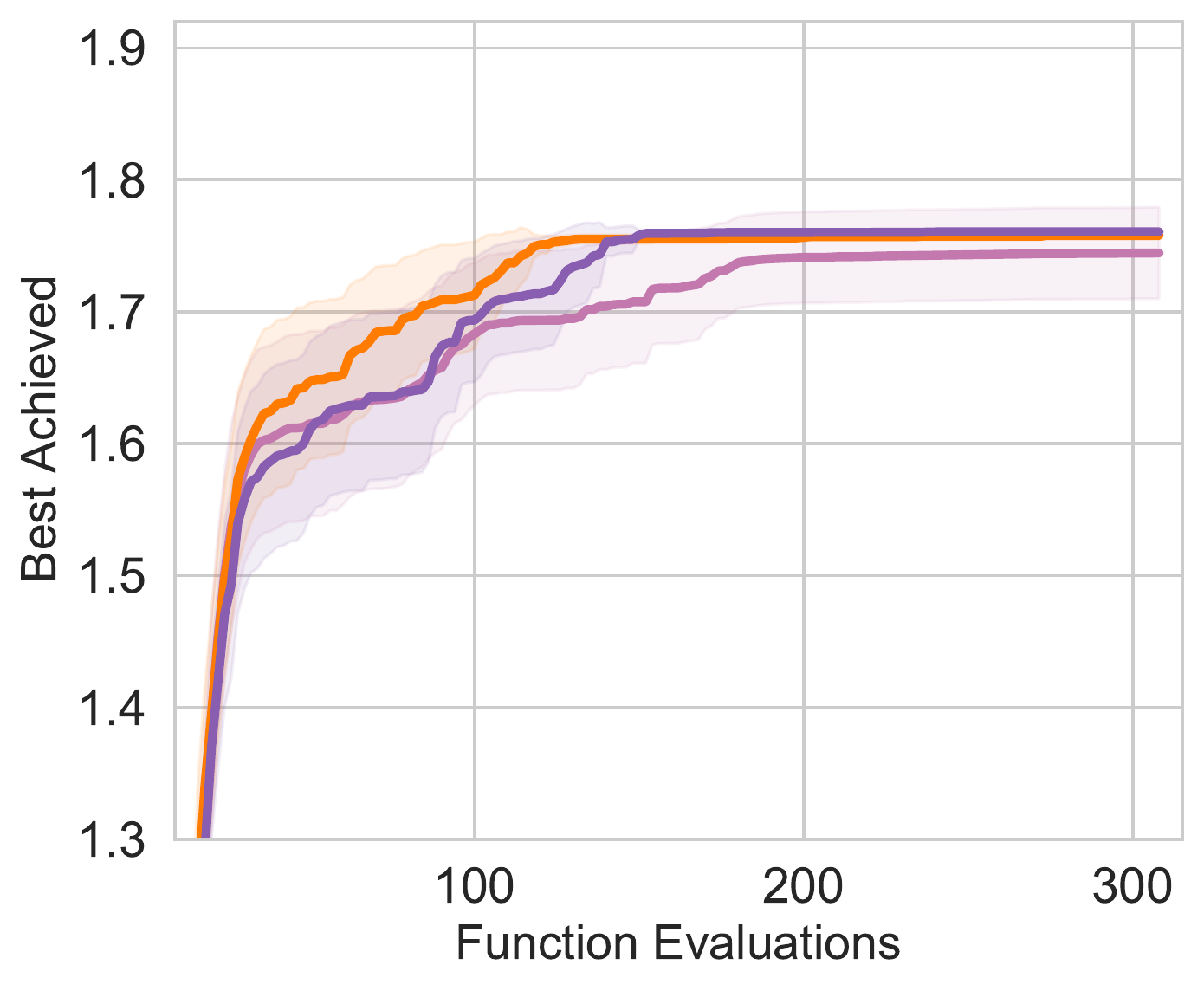}
		\caption{$5$ contexts}
	\end{subfigure}
		\begin{subfigure}{0.32\textwidth}
		\centering
		\includegraphics[width=\linewidth]{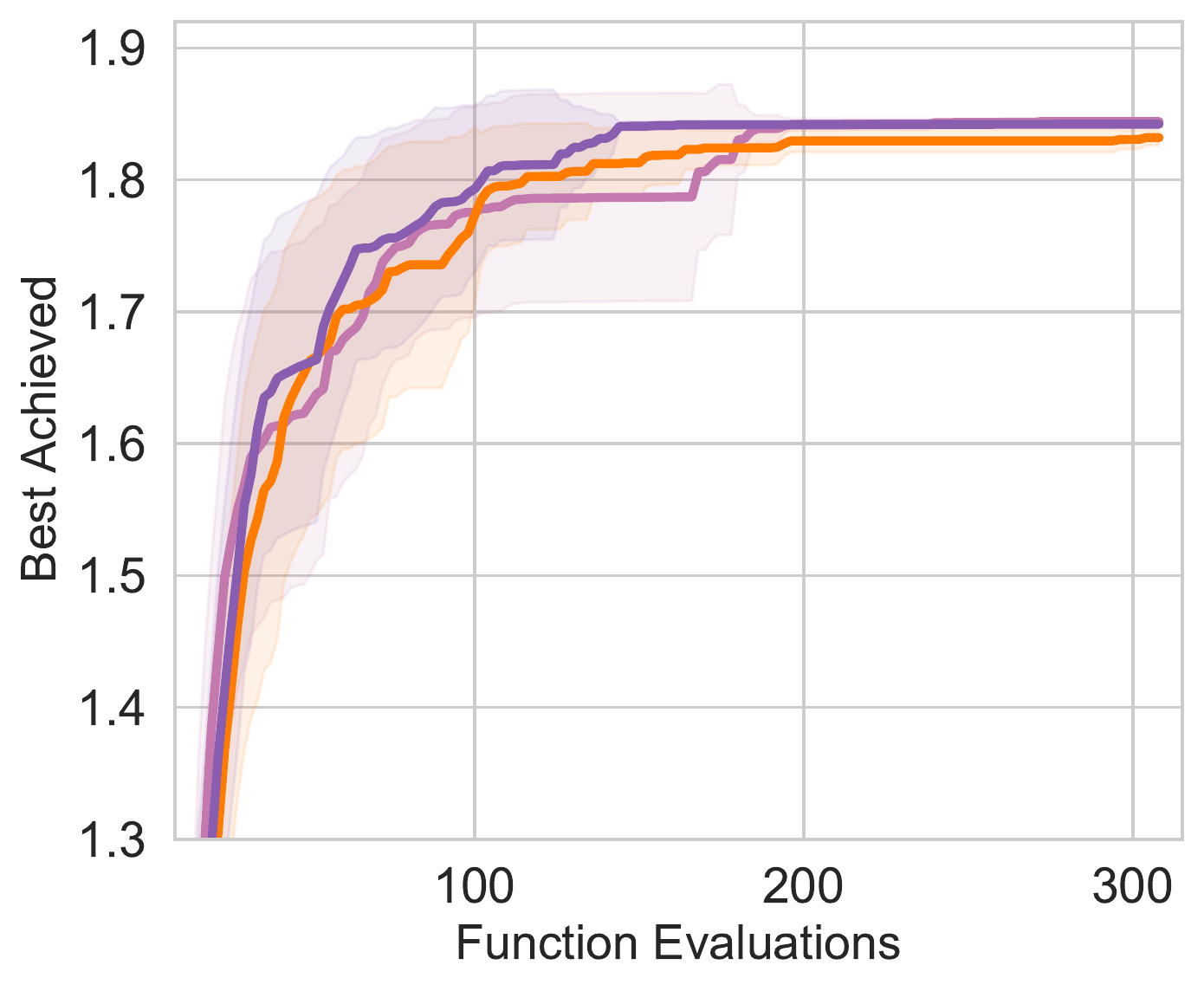}
		\caption{$10$ contexts}
		\label{fig:cbo_contexts_10}
	\end{subfigure}
	\begin{subfigure}{0.32\textwidth}
		\centering
		\includegraphics[width=\linewidth]{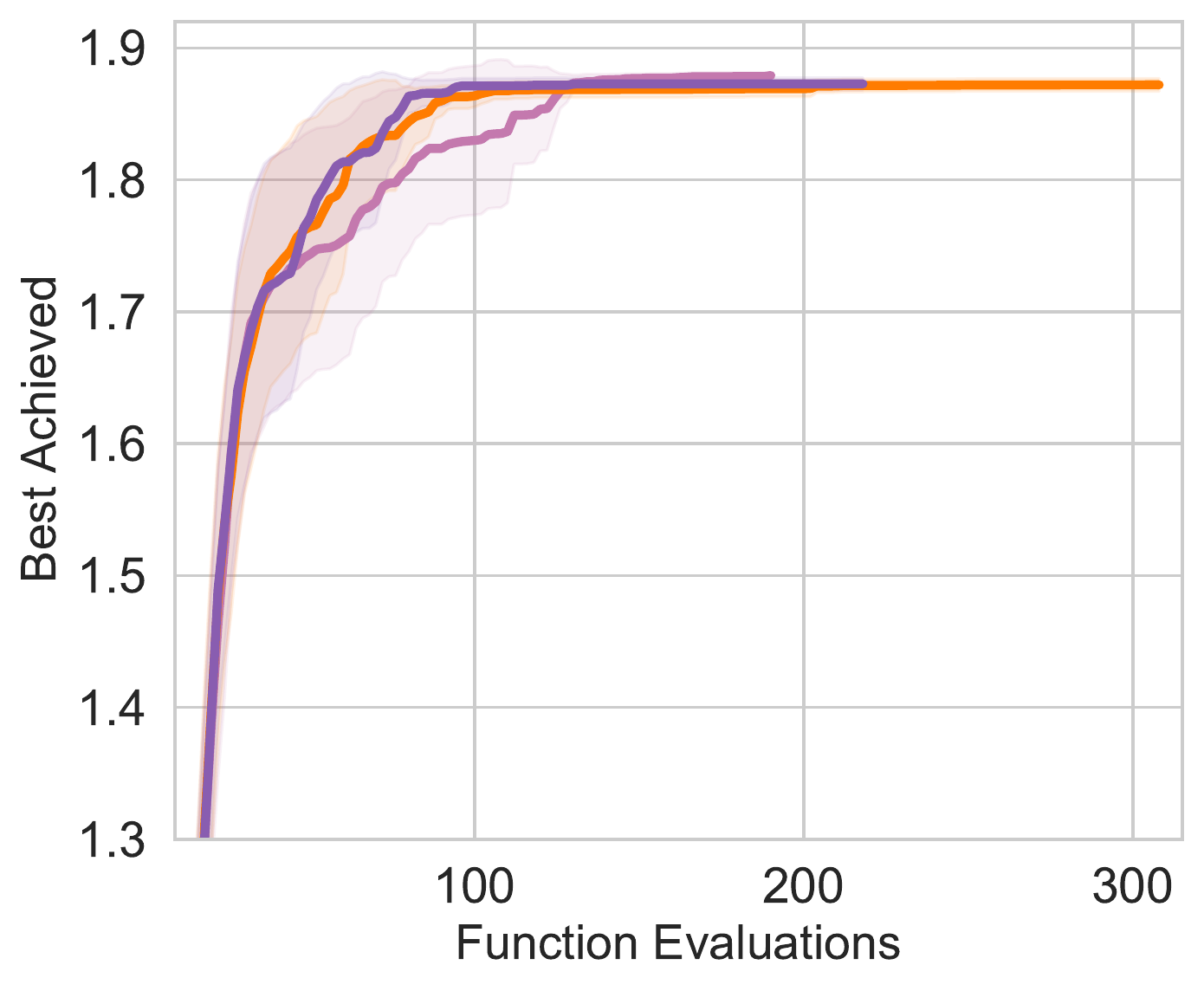}
		\caption{$20$ contexts}
		\label{fig:cbo_contexts_20}
	\end{subfigure}
	\begin{subfigure}{0.32\textwidth}
		\centering
		\includegraphics[width=\linewidth]{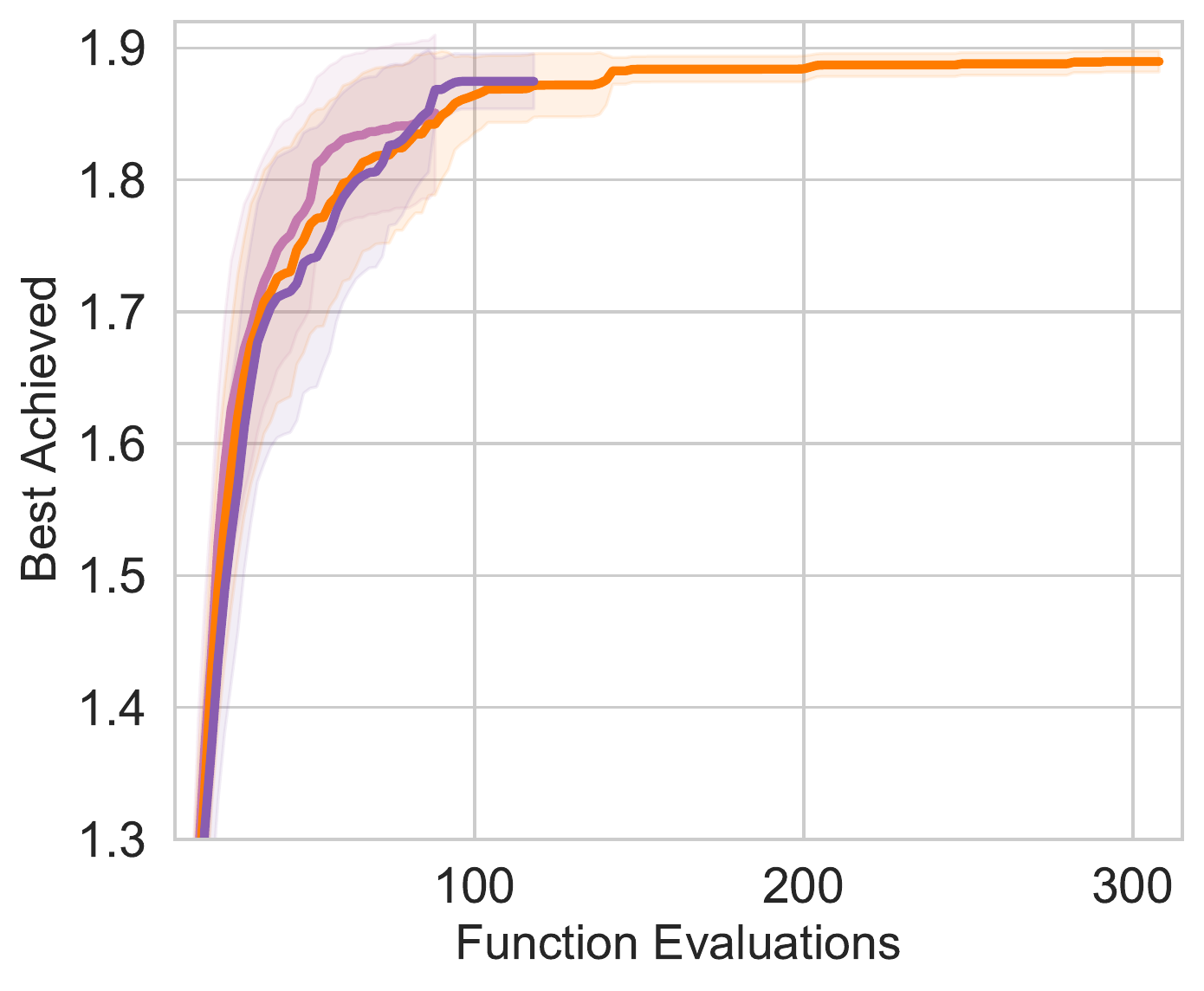}
		\caption{$50$ contexts}
		\label{fig:cbo_contexts_50}
	\end{subfigure}
	\begin{subfigure}{0.32\linewidth}
		\includegraphics[width=\textwidth]{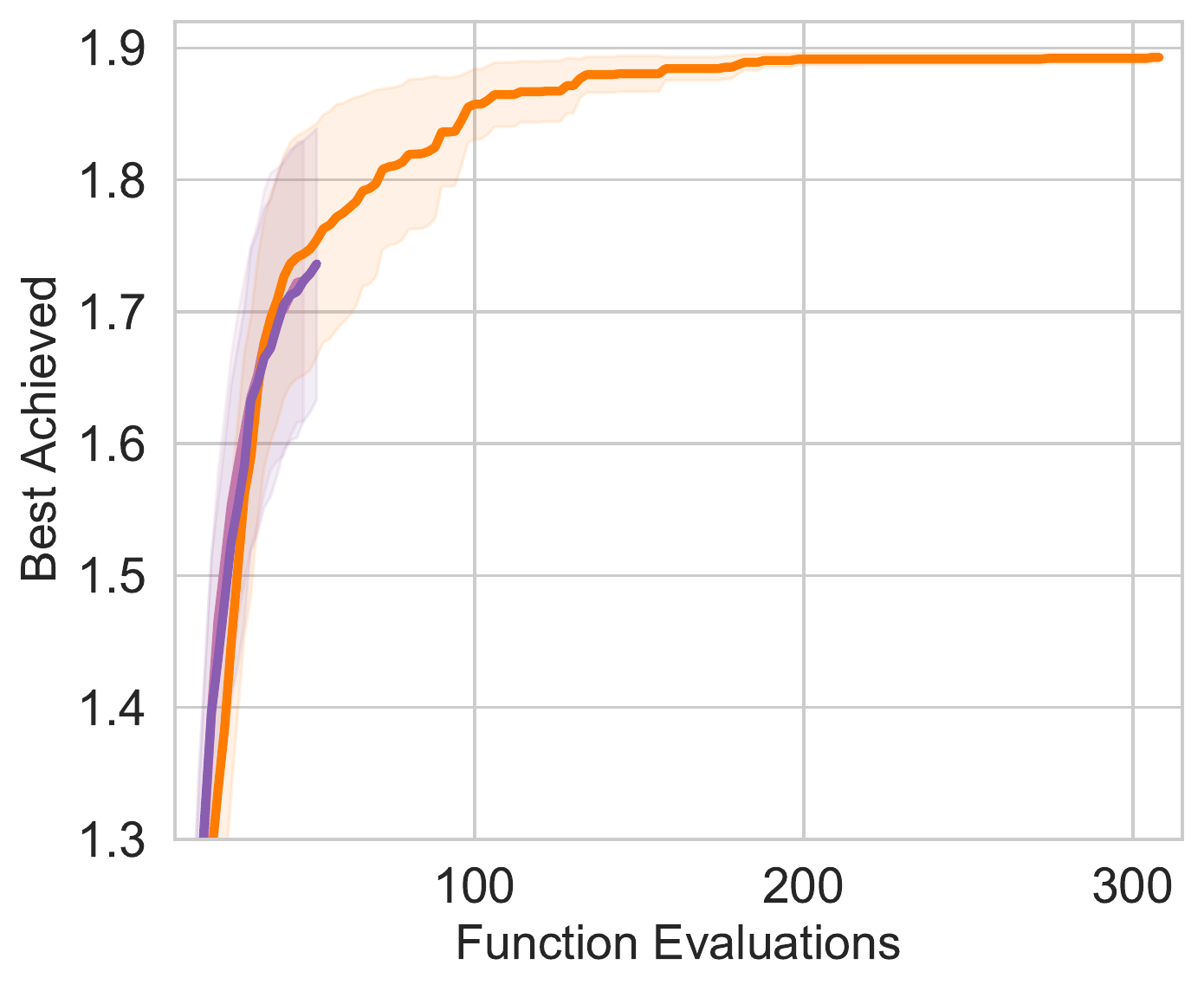}
		\caption{$100$ contexts}
		\label{fig:cbo_contexts_100}
	\end{subfigure}
	\begin{subfigure}{0.32\textwidth}
		\centering
		\includegraphics[width=\linewidth]{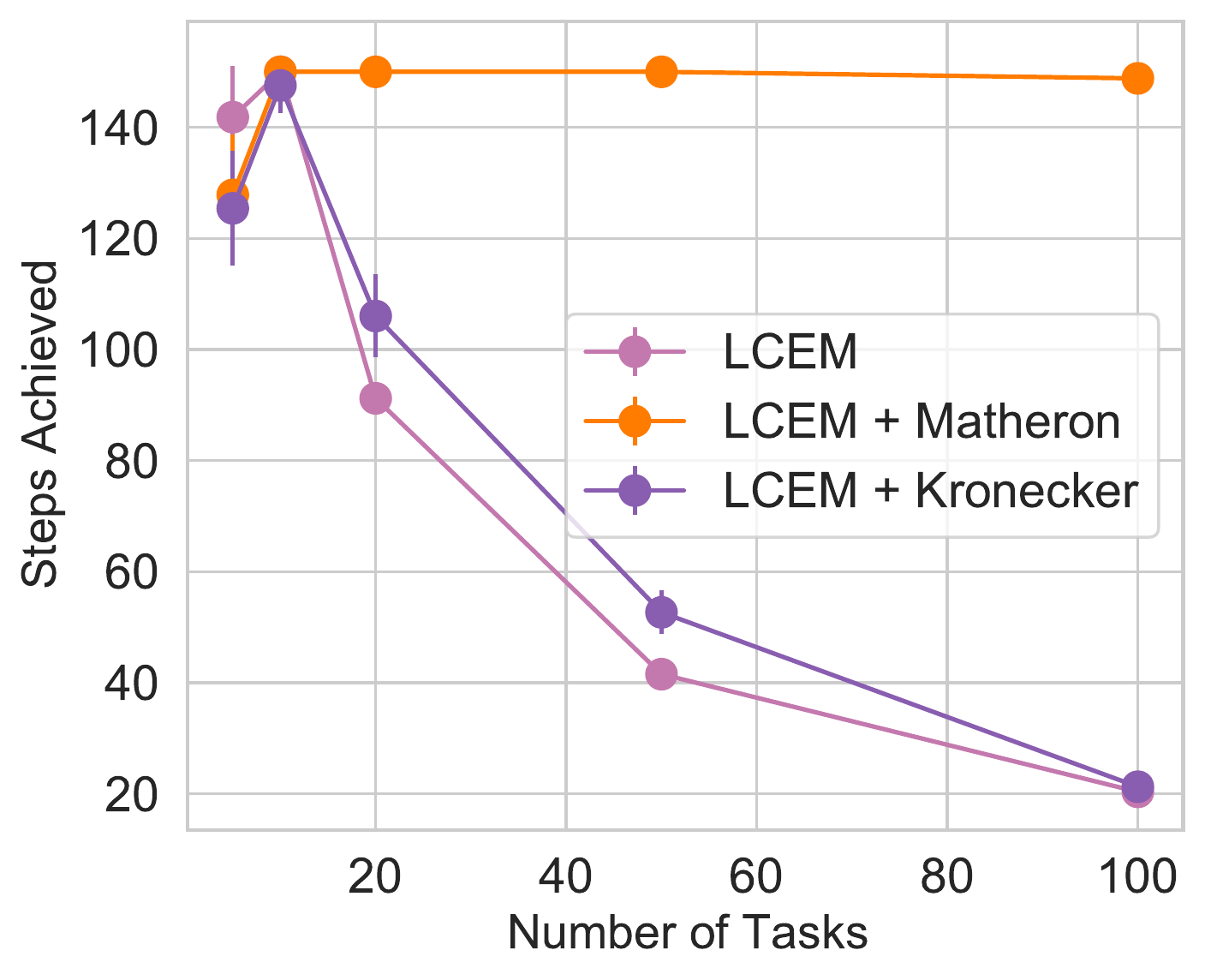}
		\caption{Steps achieved}
		\label{fig:mtbo_contexts_steps}
	\end{subfigure}

			\caption{Multi-task Bayesian Optimization with on Hartmann5D with LCE-M kernel using distributional sampling with Hadamard (LCEM) and  Kronecker (LCEM + Kronecker) models, and our Matheron-based posterior sampling (LCE-M + Matheron). For five contexts \textbf{(a)}, optimization performance is essentially identical as expected; for $100$ contexts \textbf{(b)}, only the LCEM + Matheron's rule model reaches a maximum as the others run out of memory.
			Results over $40$ trials for $10, 20,$ and $50$ contexts on Hartmann-$6$ translated into a contextual problem. We also show the number of average steps achieved in \textbf{(d)} where only LCEM + Matheron's rule is able to complete all $150$ steps. 
			\label{fig:cbo_contexts}
	}

	\label{fig:mtbo}
\end{figure}

\paragraph{Multi-task Bayesian Optimization: }Figure~\ref{fig:mtbo} shows results on the multi-task version of this problem for $5, 10, 20, 50, 100$ contexts. We can see that, as expected, the optimization performance is identical for the three sampling methods on five contexts; the $50$ and $100$ context case shows that our Matheron-based sampling achieves better performance overall as the other methods run out of memory.
In Figure \ref{fig:mtbo_contexts_steps}, we also show the average and steps achieved (confidence bars are two standard errors of the mean) as the number of tasks (contexts) increases.
This clearly shows the impact of the increased memory usage for both the LCEM and LCEM + Kronecker implementations, as they run out of memory after only a few steps, rather than being able to reach all $150$ optimization steps.

\paragraph{Contextual Bayesian Optimization:}

\begin{figure}[h!]
	\centering
	\begin{subfigure}{0.32\textwidth}
		\centering
		\includegraphics[width=\linewidth]{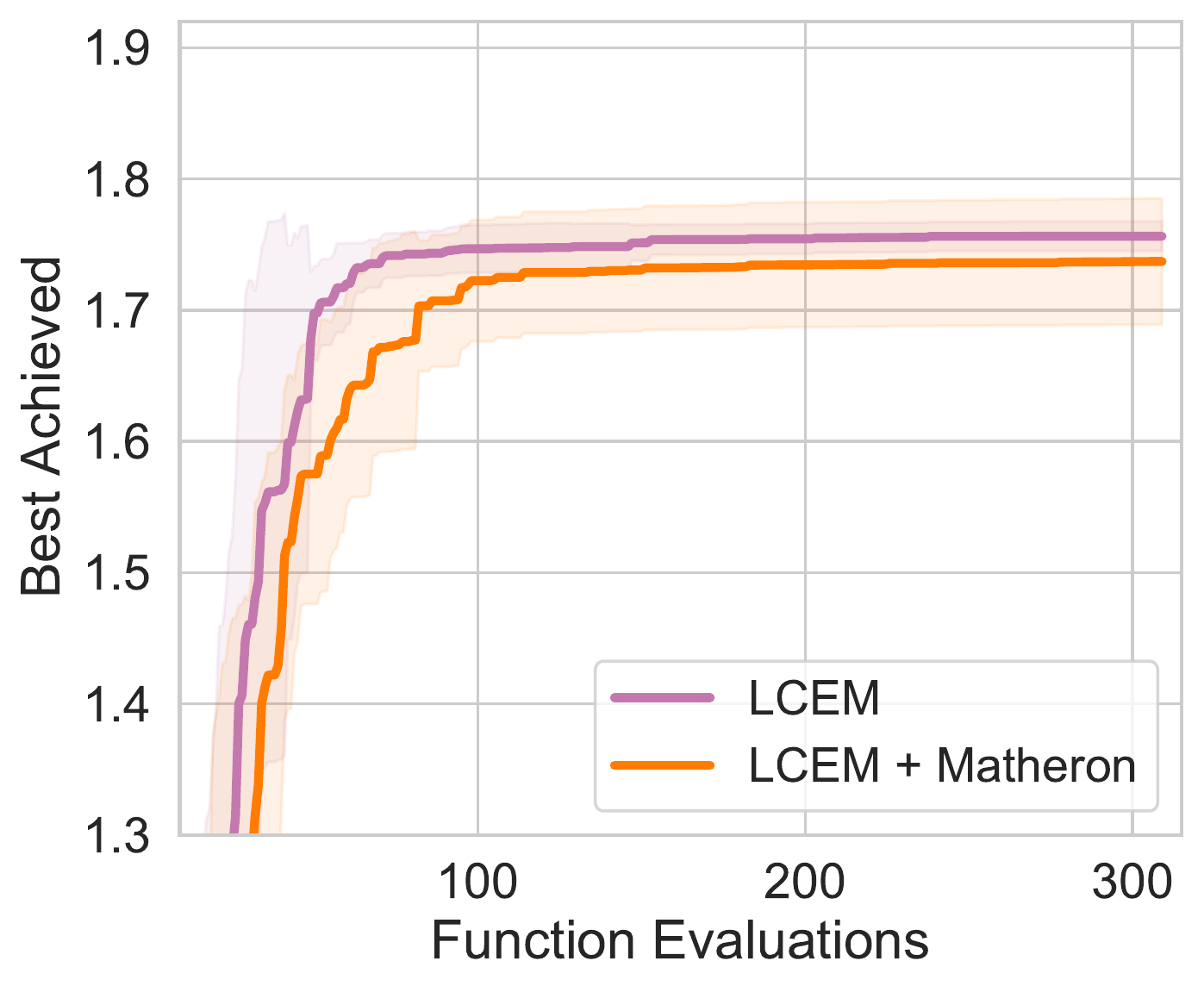}
		\caption{$50$ contexts}
		\label{fig:cbo_cbo_contexts_5}
	\end{subfigure}
	\begin{subfigure}{0.32\textwidth}
		\centering
		\includegraphics[width=\linewidth]{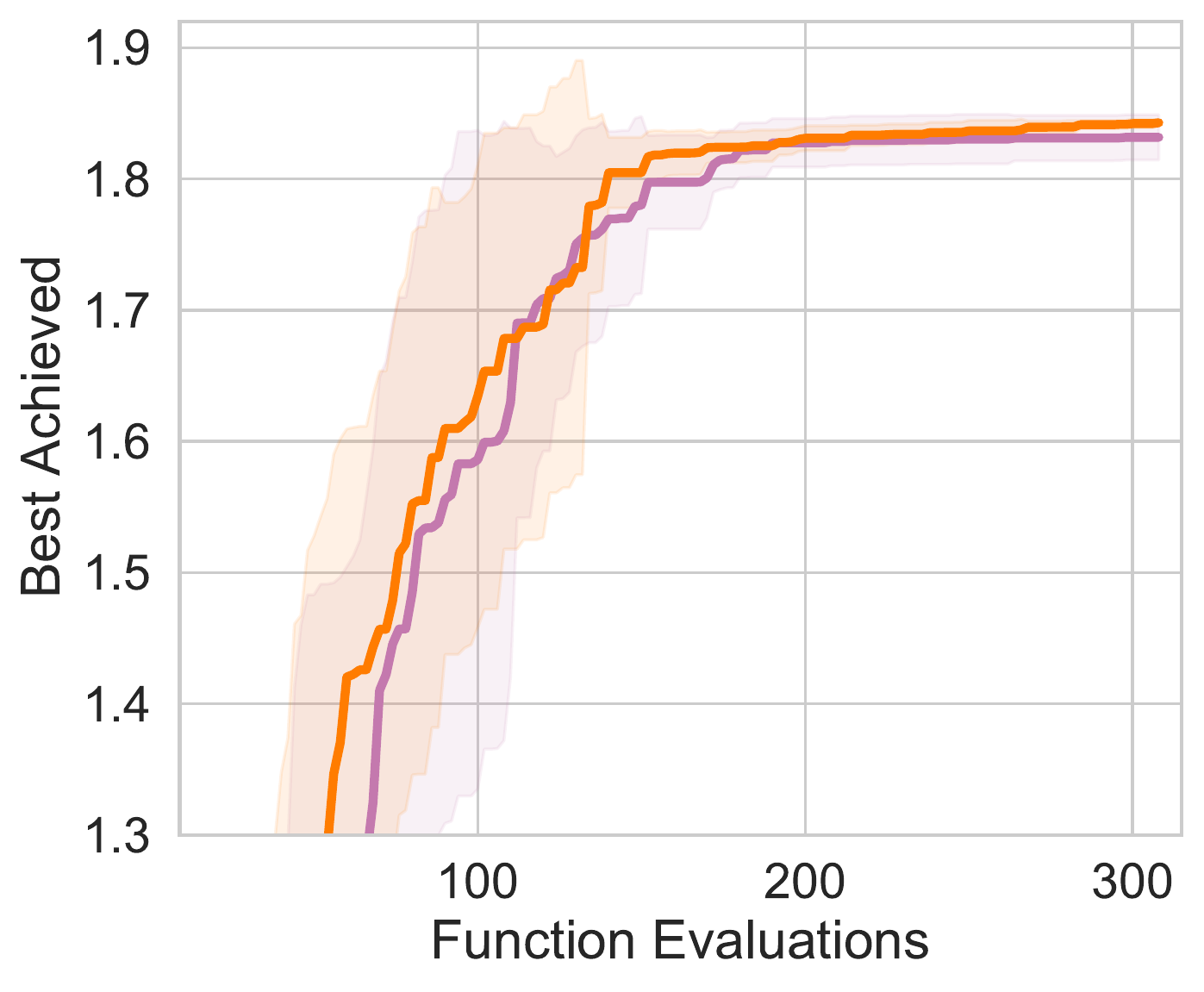}
		\caption{$10$ contexts}
		\label{fig:cbo_cbo_contexts_10}
	\end{subfigure}
	\begin{subfigure}{0.32\textwidth}
		\centering
		\includegraphics[width=\linewidth]{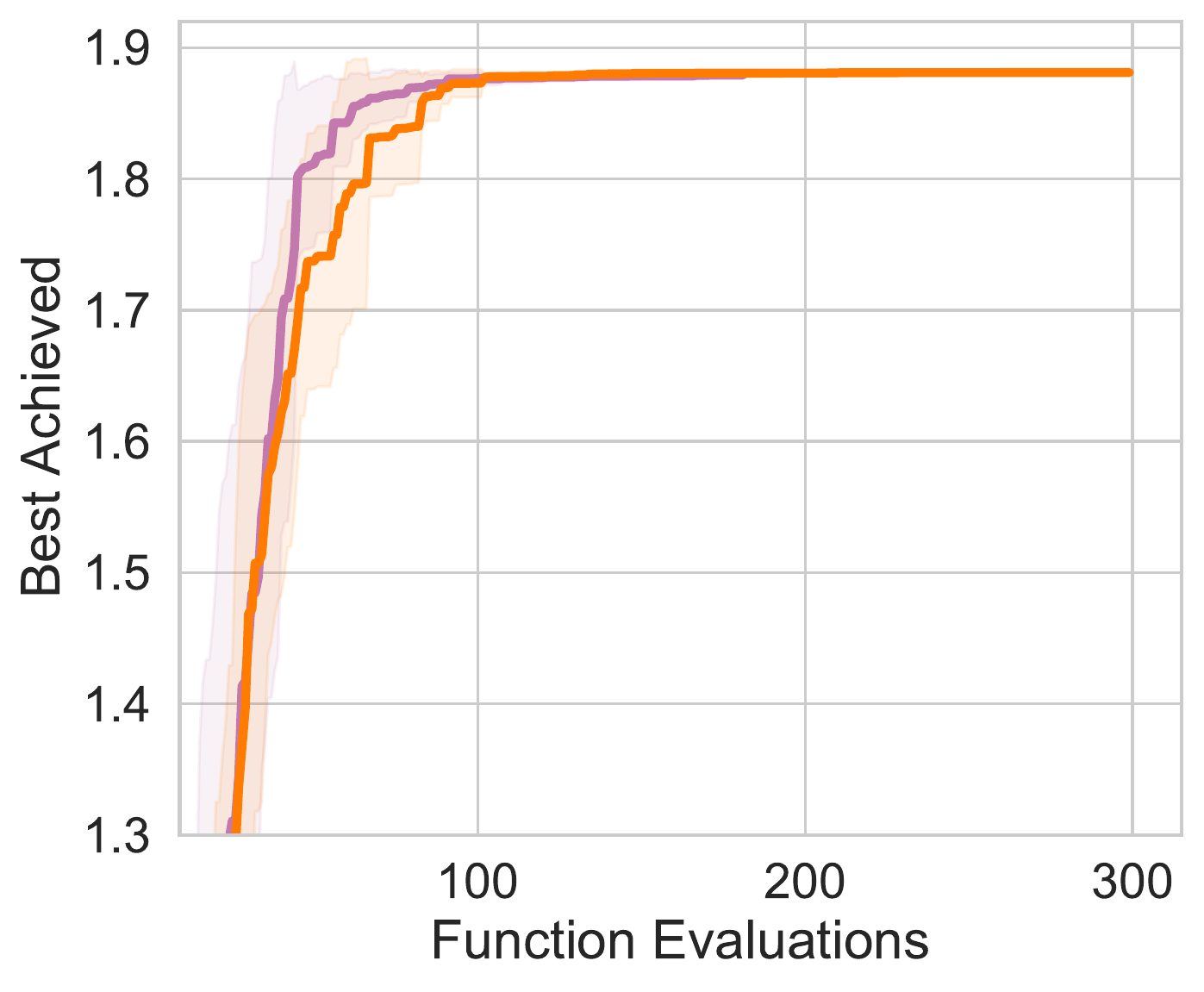}
		\caption{$20$ contexts}
		\label{fig:cbo_cbo_contexts_20}
	\end{subfigure}
	\begin{subfigure}{0.32\textwidth}
		\centering
		\includegraphics[width=\linewidth]{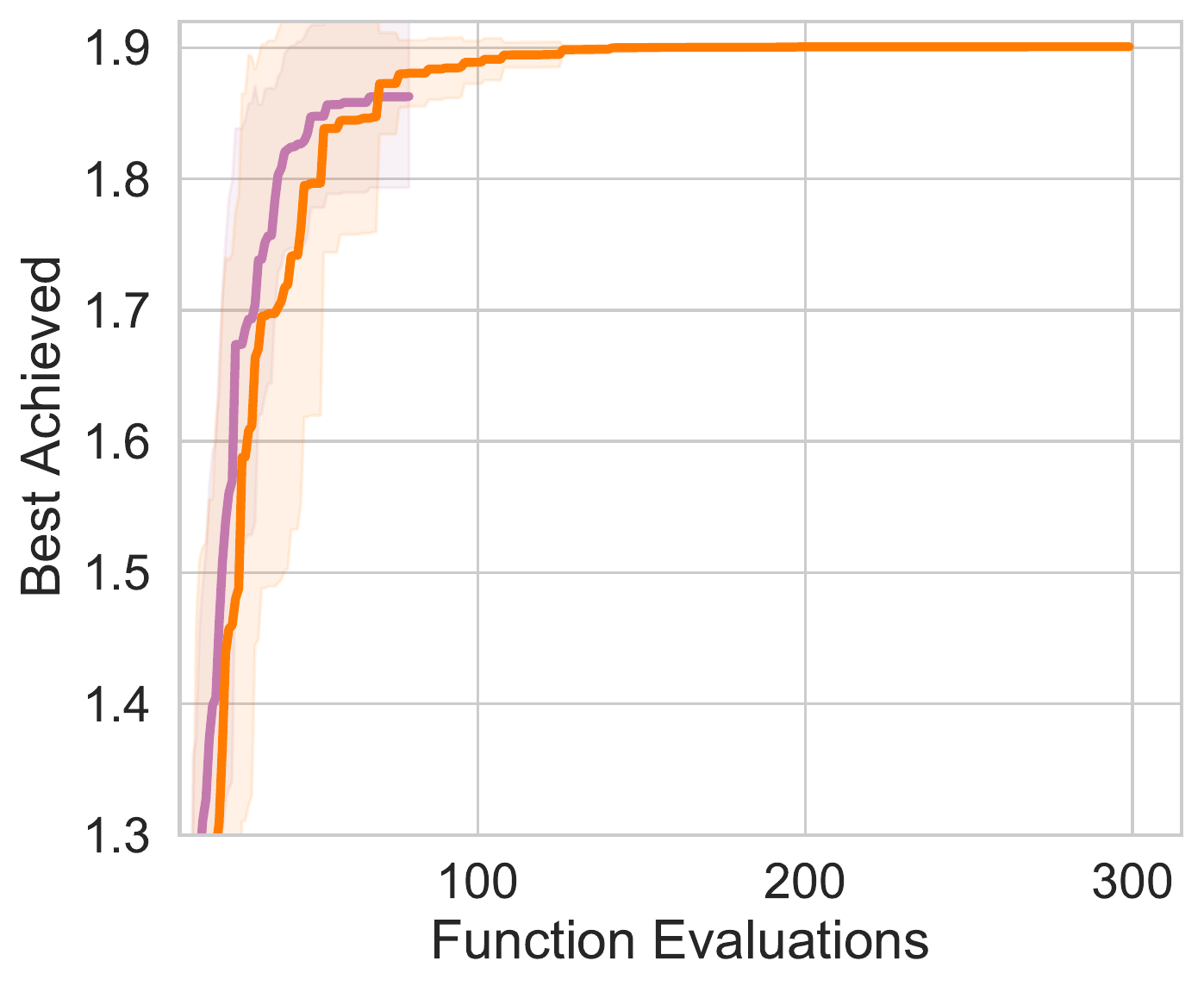}
		\caption{$50$ contexts}
		\label{fig:cbo_cbo_contexts_50}
	\end{subfigure}
	\begin{subfigure}{0.32\textwidth}
		\centering
		\includegraphics[width=\linewidth]{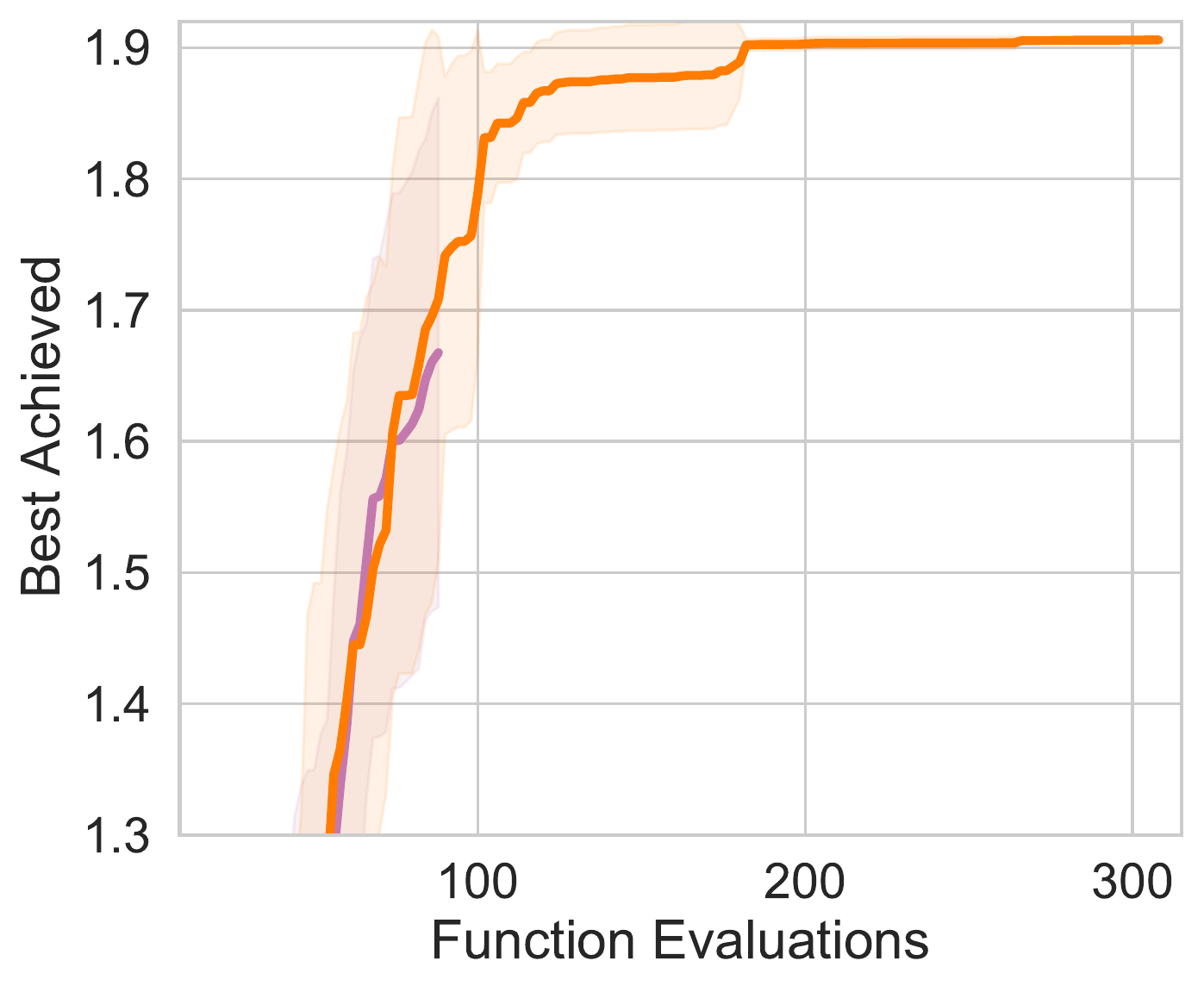}
		\caption{$100$ contexts}
		\label{fig:cbo_cbo_contexts_100}
	\end{subfigure}
	\begin{subfigure}{0.32\textwidth}
		\centering
		\includegraphics[width=\linewidth]{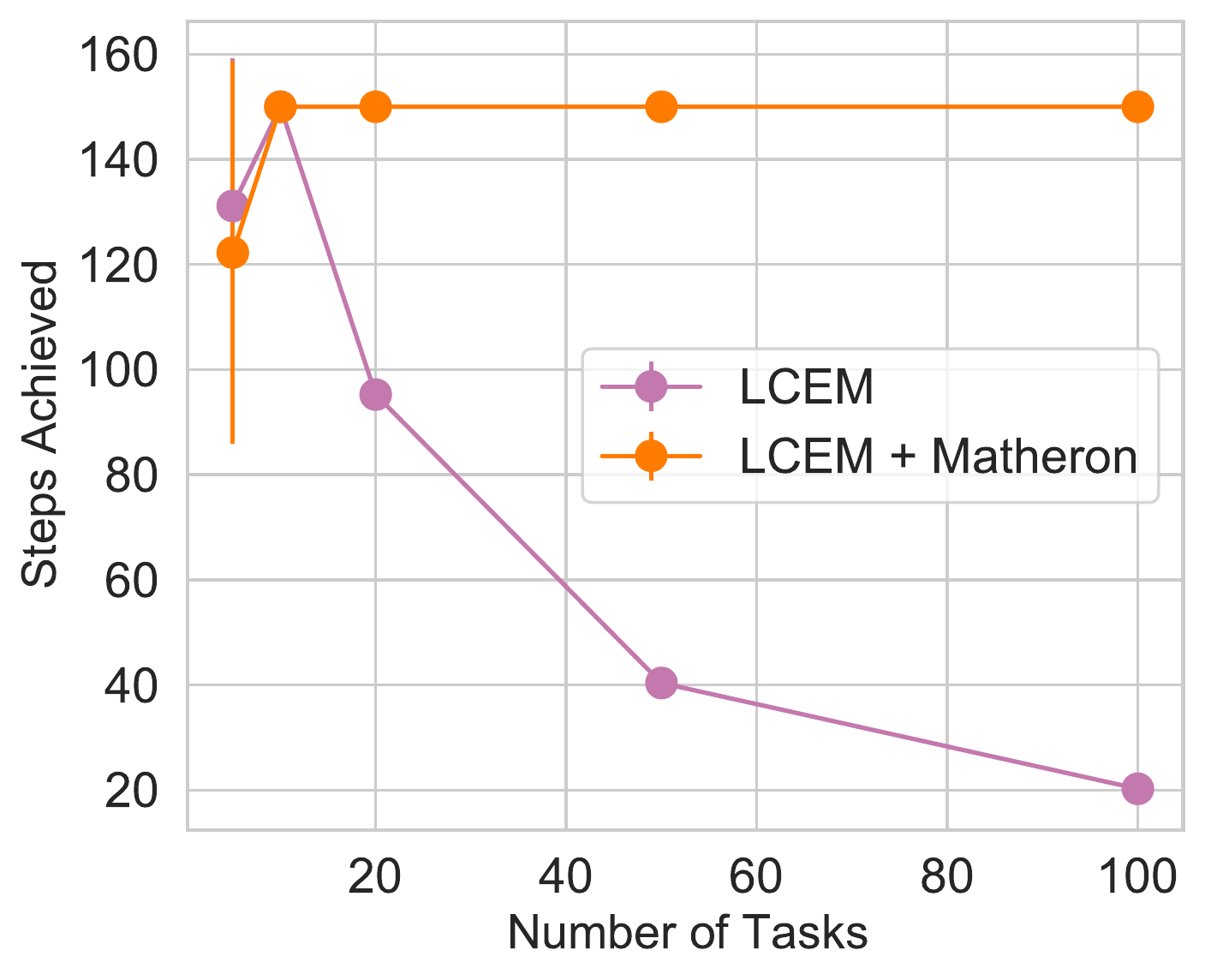}
		\caption{Steps achieved}
		\label{fig:cbo_cbo_contexts_steps}
	\end{subfigure}
	\caption{Results over $8$ trials for $10, 20,$ and $50$ contexts on Hartmann-$6$ translated into a contextual problem. We also show the number of average steps achieved in \textbf{(d)} where only LCEM + Matheron's rule is able to average $150$ steps (the maximum that we ran each trial for) of Bayesian Optimization.}
	\label{fig:cbo_cbo_contexts}
\end{figure}

In Figure \ref{fig:cbo_cbo_contexts}, we display the results of setting up the Hartmann-$5D$ problem as a contextual Bayesian Optimization. 
Here, we observe each context at every iteration, but need to choose a policy to for a randomly chosen context.
We then observe all contexts' observations for that observed policy, and plot the best average reward achieved across all contexts.
Again, the findings from the MTGP experiment carry over --- which is that optimization performance is nearly identical, but that the Matheron's rule implementation can achieve many more steps at a higher number of contexts, as shown in Figure \ref{fig:cbo_cbo_contexts_steps}.

\paragraph{Experimental Details:}
We fit the MTGPs with a constant diagonal likelihood with a $\text{Gamma}(1.1, 0.05)$ prior on the noise, ARD Matern kernels on the data with $\nu = 2.5,$ lengthscale priors of $\text{Gamma}(3.0, 6.0),$ and a $\text{Gamma}(2.0, 0.15)$ prior on the output dimension using Adam for $250$ steps.
The LCEM kernel uses similar priors and a one dimensional embedding, following \citet{feng_high-dimensional_2020}.
The LCEM implementation was Hadamard based and exactly from \url{https://github.com/pytorch/botorch/blob/master/botorch/models/contextual_multioutput.py}.
For all, we used $10$ initial points, ran $150$ steps of BO, normalized the inputs to $[0,1]^d$ and standardized the responses.
The simulated Hartmann-$5D$ function comes from \url{https://github.com/facebookresearch/ContextualBO/} (MIT License) \citep{feng_high-dimensional_2020}.
Here, we use a single $16$GB V100 GPU (\texttt{p3.8xlarge} AWS instance) and repeat the experiments $43$ times for $5$ tasks and $40$ times for $100$ tasks, showing the mean and two standard errors of the mean.
For BO loops, we used $128$ MC samples, $q=2$ batch size, $256$ initialization samples with a batch limit of $5$, $10$ restarts for the optimization loop and $200$ iterations of L-BFGS-B.
All other options were botorch defaults including the LKJ prior with $\eta = 2.0$ over the inter-task covariance matrix \citep{lewandowski2009generating}.

\subsection{Constrained Multi-Objective Bayesian Optimization}

For C2DTLZ2, we initialized with $25$ random points, while with OSY we initialized with $14$ random points, chosen via rejection sampling to find at least $7$ feasible points.
For both, we then optimized with $128$ MC samples, $10$ random restarts, $512$ base samples, a batch limit of $5$, an initialization batch limit of $20$ and for up to $200$ iterations of L-BFGS-B.
We used a batch size of $q=2$ for C2DTLZ2 optimizing for $200$ steps and a batch size of $q=10$ for OSY, optimizing for $30$ steps.
We used $16$ random seeds for OSY and $24$ for C2DTLZ2 and plot the mean and two standard errors of the mean.
For both, we used the default reference points as $(1.1, 1.1)$ and $(-75, 75)$ for the EHVI approaches.

The C2DTLZ2 function comes from BoTorch, while the OSY function is a reimplementation of \url{https://github.com/msu-coinlab/pymoo/blob/master/pymoo/problems/multi/osy.py} (Apache 2.0 License).
For these, we used $24$GB Nvidia RTX GPUs on an internal server.

\subsection{Scalable Constrained Bayesian Optimization}
In Figure \ref{fig:lunar_lander_time}, we show the wall clock time for fitting the batched GPs and the MTGPs on the lunar lander problem with $m = 50$ constraints. 
Here, the MTGPs are faster because they are somewhat more memory efficient; note that several runs for both reached convergence during optimization very quickly after reaching BoTorch's default Lanczos threshold ($n=800$) thus decaying the model fitting times.
We also show the time required to sample all tasks, again finding that the Thompson sampling time is much slower for the batched GPs.
Shown are means and log normal confidence intervals around the mean (hence the asymmetry).
In Table \ref{tab:lunar_feasibility}, we show the number of steps and the proportion of succeeded trials required to reach feasibility, finding that the MTGPs also reduced the number of steps required to achieve feasibility and thus improved the number of feasible runs.
For steps to feasibility, we again show means and log normal CIs around the mean.

\begin{figure}[h!]
	\begin{subfigure}{0.49\textwidth}
		\centering
		\includegraphics[width=\linewidth]{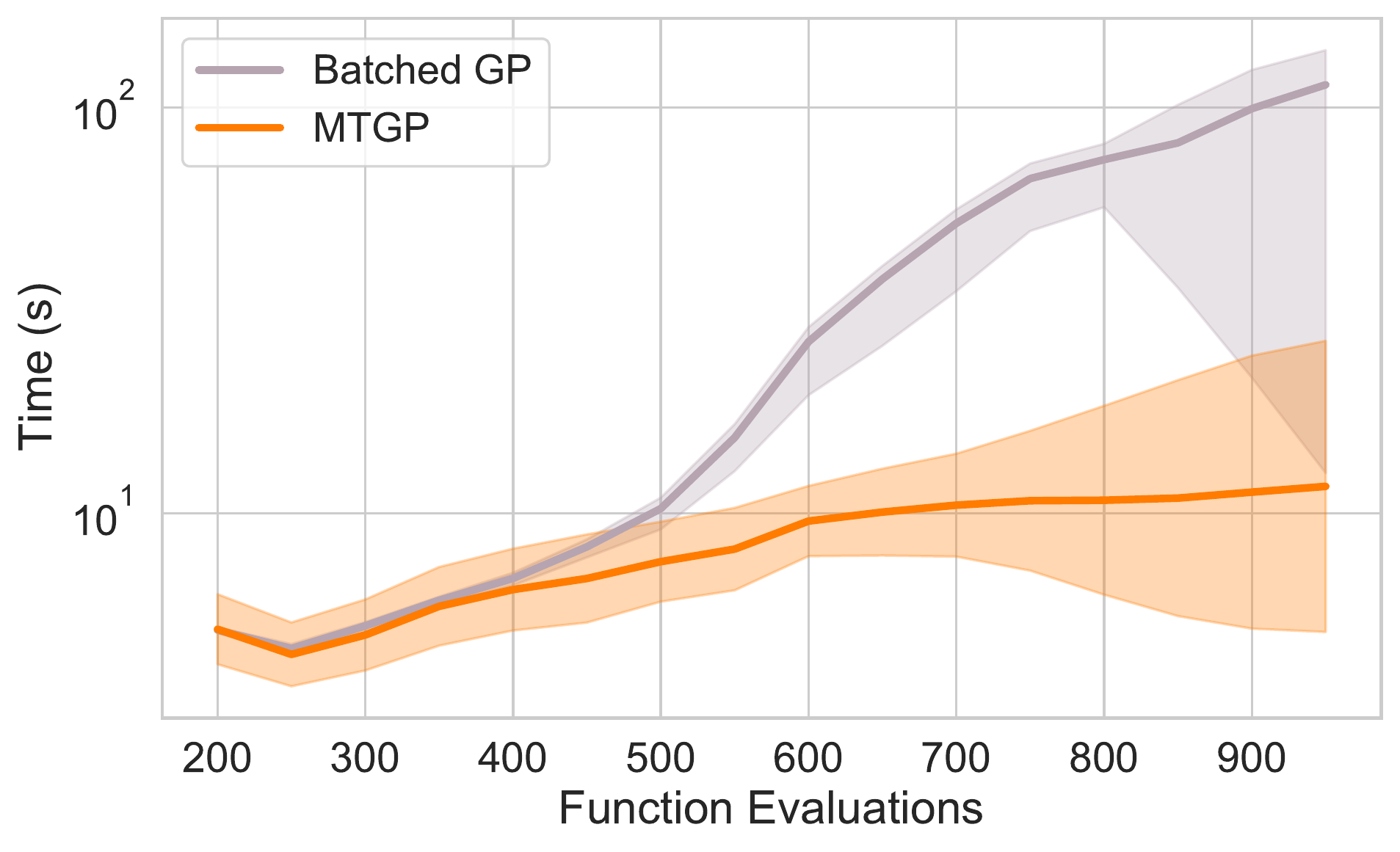}
		\caption{Fitting time}
		\label{fig:lunar_lander_fitting}
	\end{subfigure}
	\begin{subfigure}{0.49\textwidth}
		\centering
		\includegraphics[width=\linewidth]{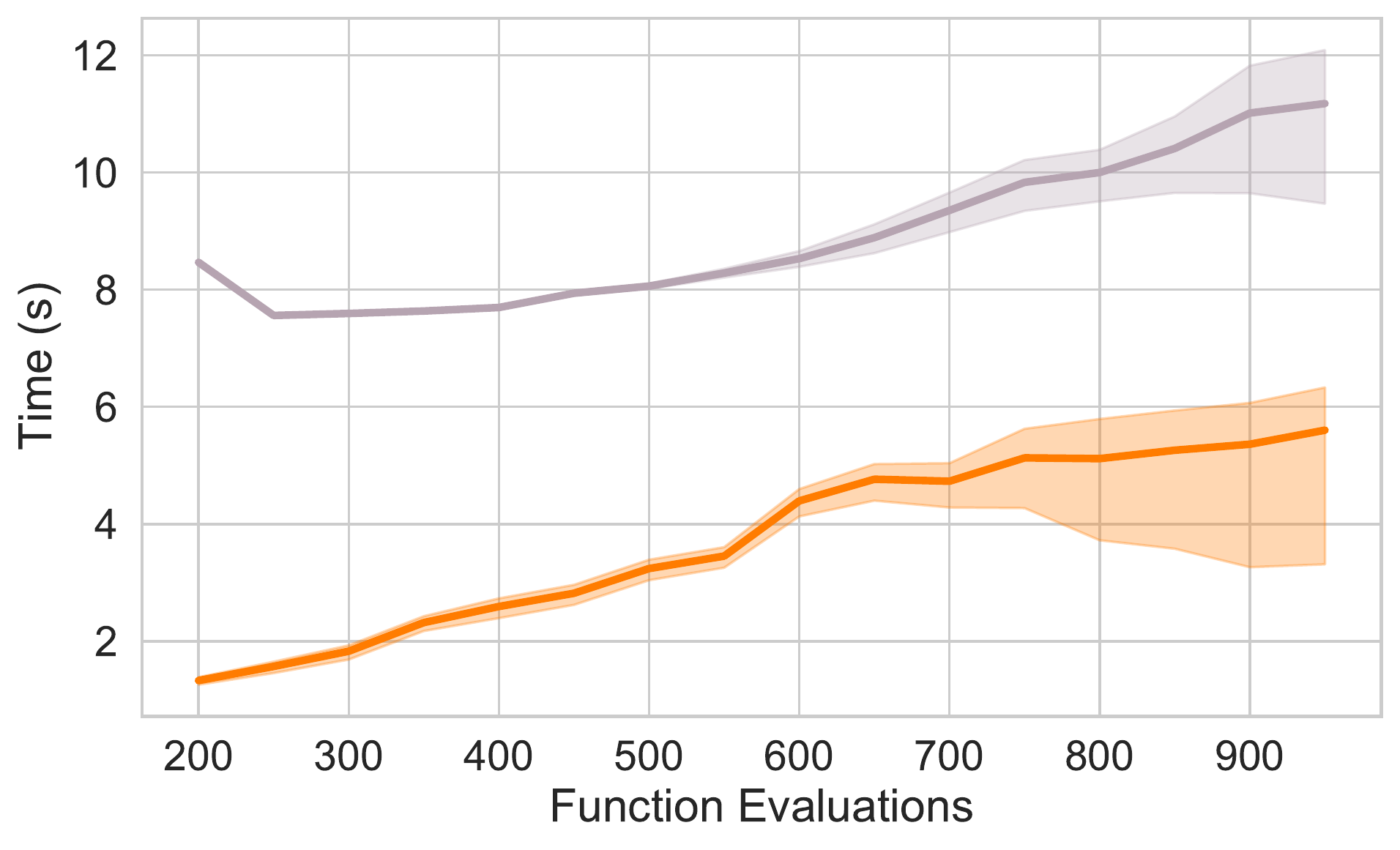}
		\caption{Thompson sampling time}
		\label{fig:lunar_lander_sampling}
	\end{subfigure}   
	\caption{\textbf{(a)} Wall clock time for model fitting with multi-task and batched Gaussian processes as a function of the number of function queries. \textbf{(b)} Time for Thompson sampling as a function of the number of function queries. In both cases, the multi-task Gaussian process is faster; training due to using conjugate gradients and Kronecker MVMs while Thompson sampling is faster due to the Matheron's rule approach that we use.}
	\label{fig:lunar_lander_time}
\end{figure}

\begin{table}[h!]
	\centering
	\caption{Number of function evaluations to achieve feasibility on the lunar lander (LL) and the MOPTA08 optimization problems, given that feasibility was reached, as well as the proportion of runs that achieved feasibility. Using a multi-task GP in SCBO achieves feasible outcomes with fewer function evaluations; on the $m=100$ constraint lunar lander, the batch GPs ran out of memory before reaching feasibility.}
	\begin{tabular}{c|c |c}
		\toprule
		Method & Time to Feasibility & Proportion \\\bottomrule
		& LL, $m=50$ &  \\\midrule
		MTGP SCBO & $314, (229, 333)$ & $26/30, (0.73, 1.0)$ \\
		Batch SCBO & $401, (273, 432)$ & $24/30, (0.64, 0.96)$\\
		PCA GP SCBO & $324, (239, 348)$ & $25/30, (0.68, 0.97)$\\\toprule
		& LL, $m=100$ & \\\midrule
		MTGP SCBO & $349, (260, 390)$ & $17/30 ,(0.33, 0.8)$\\
		Batch SCBO & --- & $0 / 30$ \\\toprule        
		& MOPTA08 &  \\\midrule
		MTGP SCBO & $292, (250, 327)$ & $9/9$ \\
		Batch SCBO & $415, (318, 493)$ & $8/8$ \\\bottomrule
	\end{tabular}
	\label{tab:lunar_feasibility}
\end{table}

Here, we followed the parameterizations and other implementation details from \citet{eriksson2020scalable} and used $30$ random seeds for the lunar lander problems and $9$ on the MOPTA08 problem ($8$ for batch GPs due to memory issues).
Here, we used a single $24$GB Titan RTX GPU for all experiments (part of an internal server), and used KeOPS \citep{feydy2020fast} for the batched GPs.
We used a full rank ICM kernel with a LKJ prior $\eta = 2.0$ \citep{lewandowski2009generating} and a smoothed box prior on the standard deviation of $(e^{-6}, e^{1.25}).$ for the multi-task GPs and diagonal Gaussian noise with a $\text{Horseshoe}(0.1)$ prior, constraining the diagonal noise to $[10^{-6}, 4.0].$

The executable for MOPTA08 is available at \url{https://www.miguelanjos.com/jones-benchmark} (no license provided).
The lunar lander problem uses \url{https://github.com/openai/gym/blob/master/gym/envs/box2d/lunar_lander.py} \citep{openaigym} (MIT License). 
The SCBO code from \citet{eriksson2020scalable} is currently unreleased; we implemented our own version.

\subsection{Composite Bayesian Optimization Experiments}

In all experiments with the HOGP, we used diagonal Gaussian likelihoods with $\text{Gamma}(1.1, 0.05)$ priors on the standard deviation, and Matern 2.5 kernels with a lengthscale prior of $\text{Gamma}(3., 6.).$
For the standard version, we randomly initialized latent parameters to be standard normal, while for the HOGP + GP models, we randomly sampled the latents from a Matern 2.5 kernel with lengthscale $1$ and input values as the indices, using the kernel as the covariance for a zero mean multivariate normal prior.
For all experiments we used \texttt{qEI} with a batch size of $1.$

For the environmental problem, we followed the implementations of \citet{balandat_botorch_2020,astudillo_bayesian_2019}, and used $8$ random restarts, $256$ MC samples, and $512$ base samples, a batch limit of $4,$ and an initialization batch limit of $8.$
These experiments were performed $50$ times on $16$GB V100 GPUs (part of an internal cluster).
The bounds are $(7, 0.02, 0.01, 30.01)$ and $(13.0, 0.12, 3., 30.295).$

For the PDE problem, we followed the example implementation given at \url{https://py-pde.readthedocs.io/en/latest/examples_gallery/pde_brusselator_expression.html#sphx-glr-examples-gallery-pde-brusselator-expression-py} (MIT License). For the metrics, we computed the weighted variance and minimized that function.
Non-finite outputs were set to $1e5.$
We up-weighted the weights on the first two rows and columns for each output to have weights $10x$ that of the rest of the inputs.
These were run $20$ times with $5$ initial points, $50$ optimization steps, using $64$ MC samples and $128$ raw samples with $4$ optimization restarts.
These were run on CPUs with $64$GB of memory (part of an internal cluster).
An example output, as well as solutions found by EI (objective value $0.1088$) and  composite EI using the HOGP model (objective value $0.0087$) are shown in Figure \ref{fig:pde_outputs}. 

\begin{figure}[t!]
	\centering
	\begin{subfigure}{0.32\textwidth}
		\includegraphics[width=\linewidth]{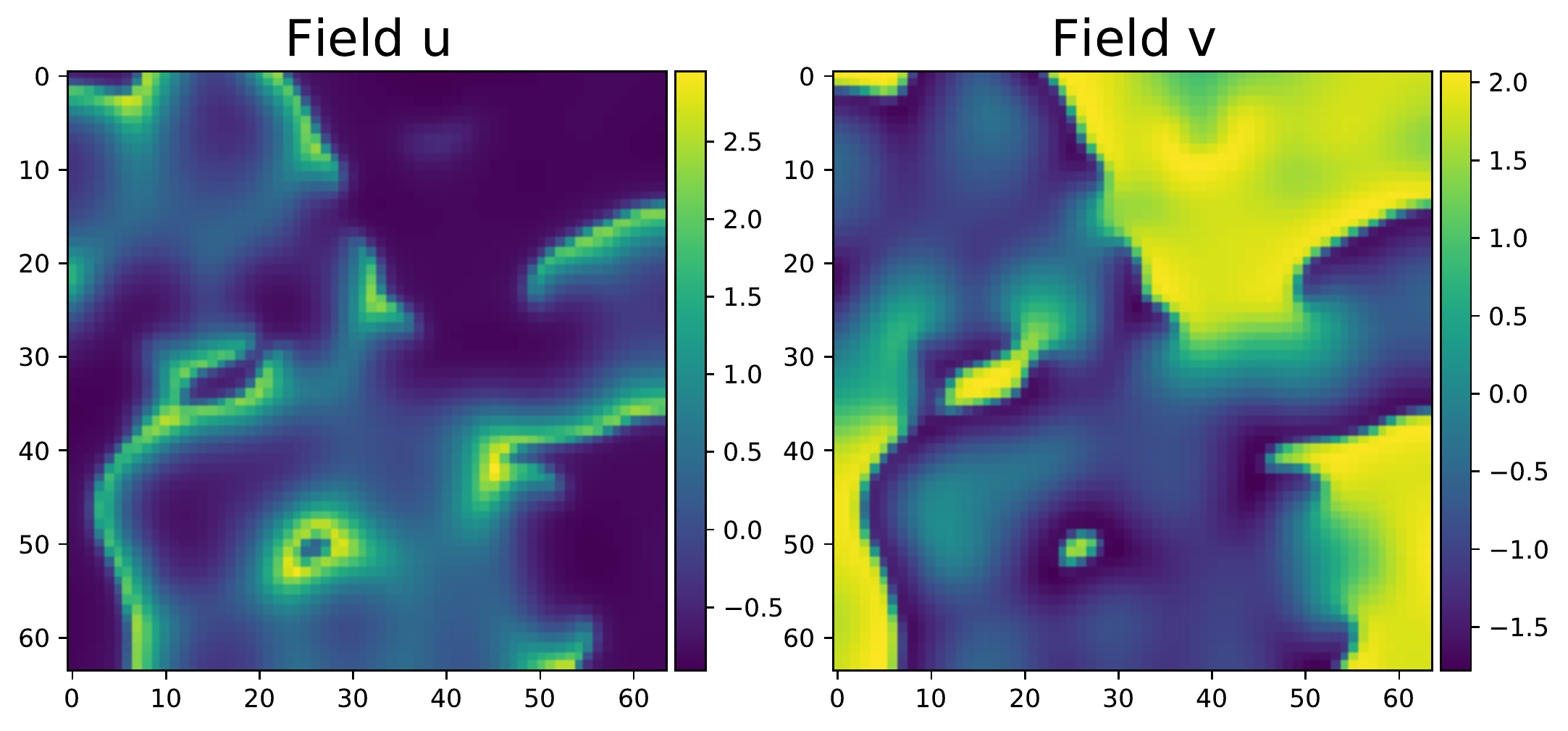}
		\caption{PDE output, random}
	\end{subfigure}
	\begin{subfigure}{0.32\textwidth}
		\includegraphics[width=\linewidth]{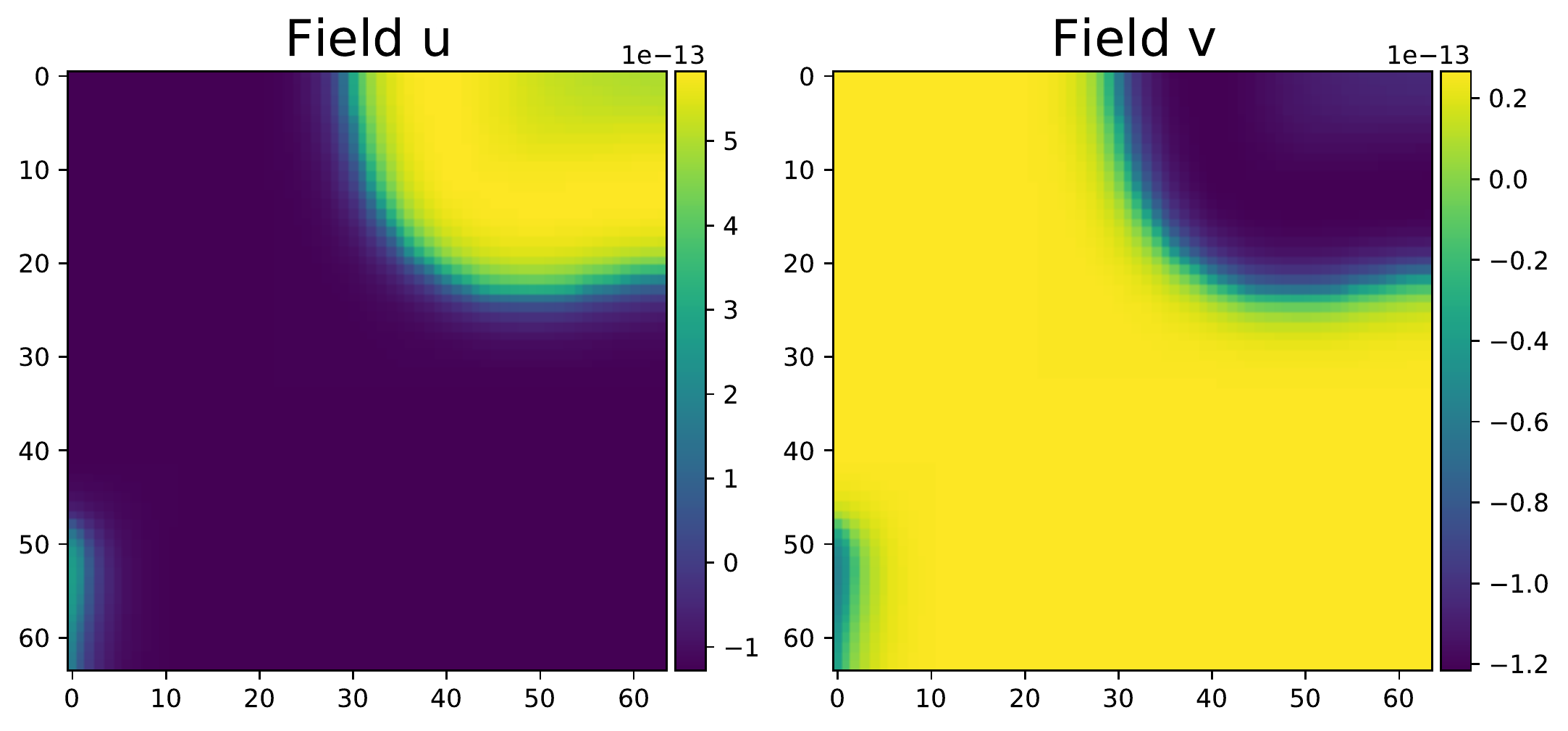}
		\caption{PDE output, EI}
	\end{subfigure}
	\begin{subfigure}{0.32\textwidth}
		\includegraphics[width=\linewidth]{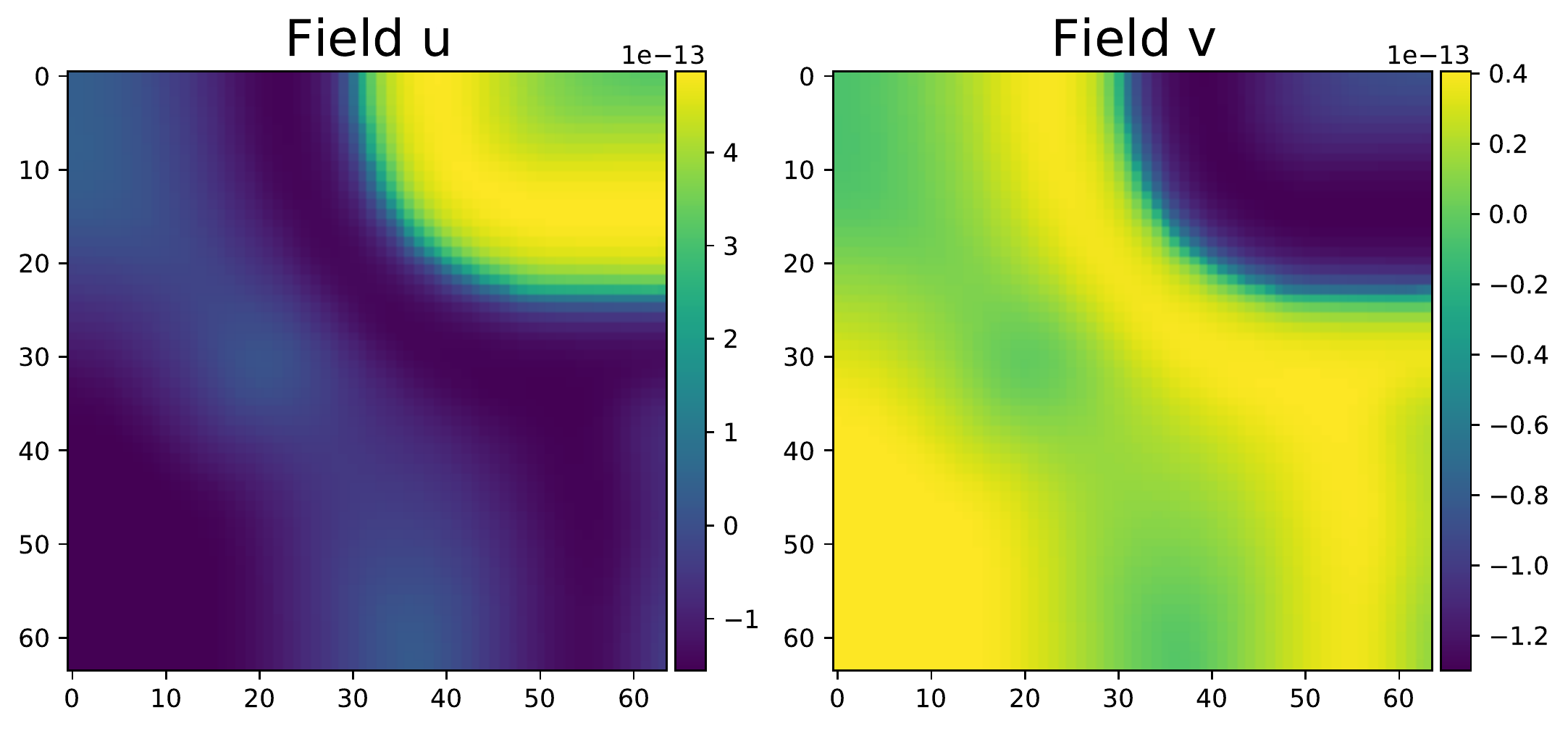}
		\caption{PDE output, EI + HOGP + GP}
	\end{subfigure}
	\caption{\textbf{(a)} Example solution from the Brusselator problem. \textbf{(b)} PDE solution as found via EI on the metric itself. \textbf{(c)} PDE solution as found via composite BO using EI with the HOGP model, which displays less variance than the EI solution. Running BO to minimize the variance is able to easily find settings of parameters within the bounds that push the reactivity to zero; the variance is reduced overall by using the HOGP + GP model. All solutions are de-meaned.}
	\label{fig:pde_outputs}
\end{figure}

On the radio frequency coverage problem, we initialized with $20$ points, downsampled the two $241 \times 241$ outputs to $50 \times 50$ for simplicity, ran the experiments over $20$ random seeds and for $150$ steps. We used $32$ MC samples, $64$ raw samples with a batch limit of $4$ and an initialization batch limit of $16.$
These were run on either $V100s$ with $32$GB of memory or $RTX8000s$ with $48$GB of memory on a shared computing cluster so we cannot tell which one was used.
The problem is $30$ dimensional and the first $15$ dimensions are $(0.0, 10)^{15}$ with the second $15$ coming as $(30.0, 50.0)^{15}.$
The first dimensions correspond to the downtilt angles of the transmitters (in $3D$ coordinates), while the second set corresponds to the power levels of the transmitters.

For metric specfication, we used the following equations, following \citet{dreifuerst2020optimizing}, where $R$ is the first output and $I$ is the second output:
\begin{align*}
	Cov_{\text{f, strong}} &= \sum_{i,j}^{50} \text{sigmoid}(-80 - R) \\
	Cov_{\text{g, weak, area}} &= \text{sigmoid}(R + 80) * \text{sigmoid}(I + 6 - R) \\
	Cov_{\text{g, weak}} &= \sum_{i,j}^{50} \text{sigmoid}(I * Cov_{\text{g, weak, area}} + 6 - R * Cov_{\text{g, weak, area}}) \\
	Obj &= 0.25 * Cov_{\text{f, strong}} + (1 - 0.75) * Cov_{\text{g, weak}}
\end{align*}
using the final line as the objective to maximize. $-80$ is the weak coverage threshold, while $6$ is the strong coverage threshold.
Representative coverage maps are shown in Figure \ref{fig:hogp_coverage_maps}, along side the maps of weak coverage and strong coverage, analogous to that of a random set of parameters in Figure \ref{fig:celltower_schematic}.

\begin{figure}[h!]
	\centering
	\begin{subfigure}{0.49\textwidth}
		\includegraphics[width=\linewidth]{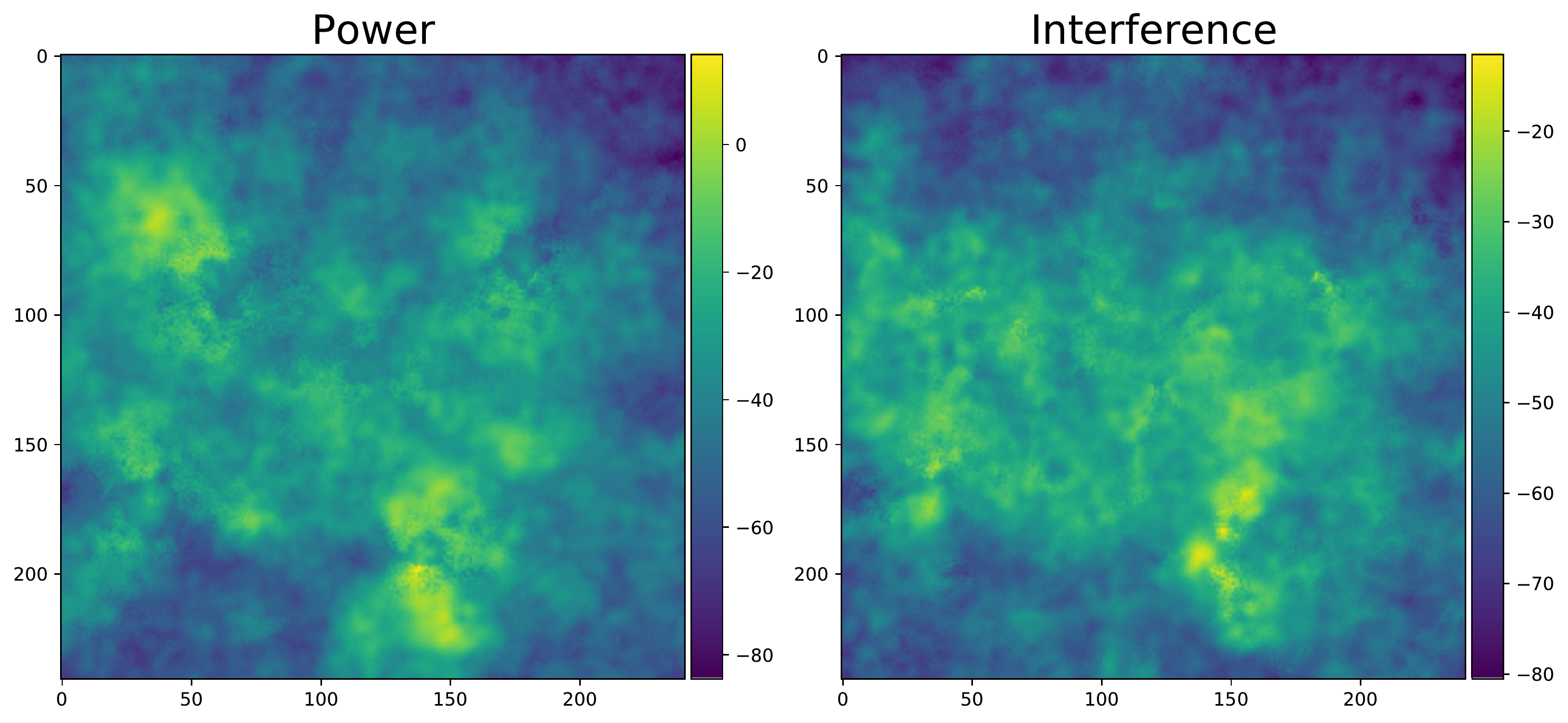}
		\caption{Coverage simulator, EI}
	\end{subfigure}
	\begin{subfigure}{0.49\textwidth}
		\includegraphics[width=\linewidth]{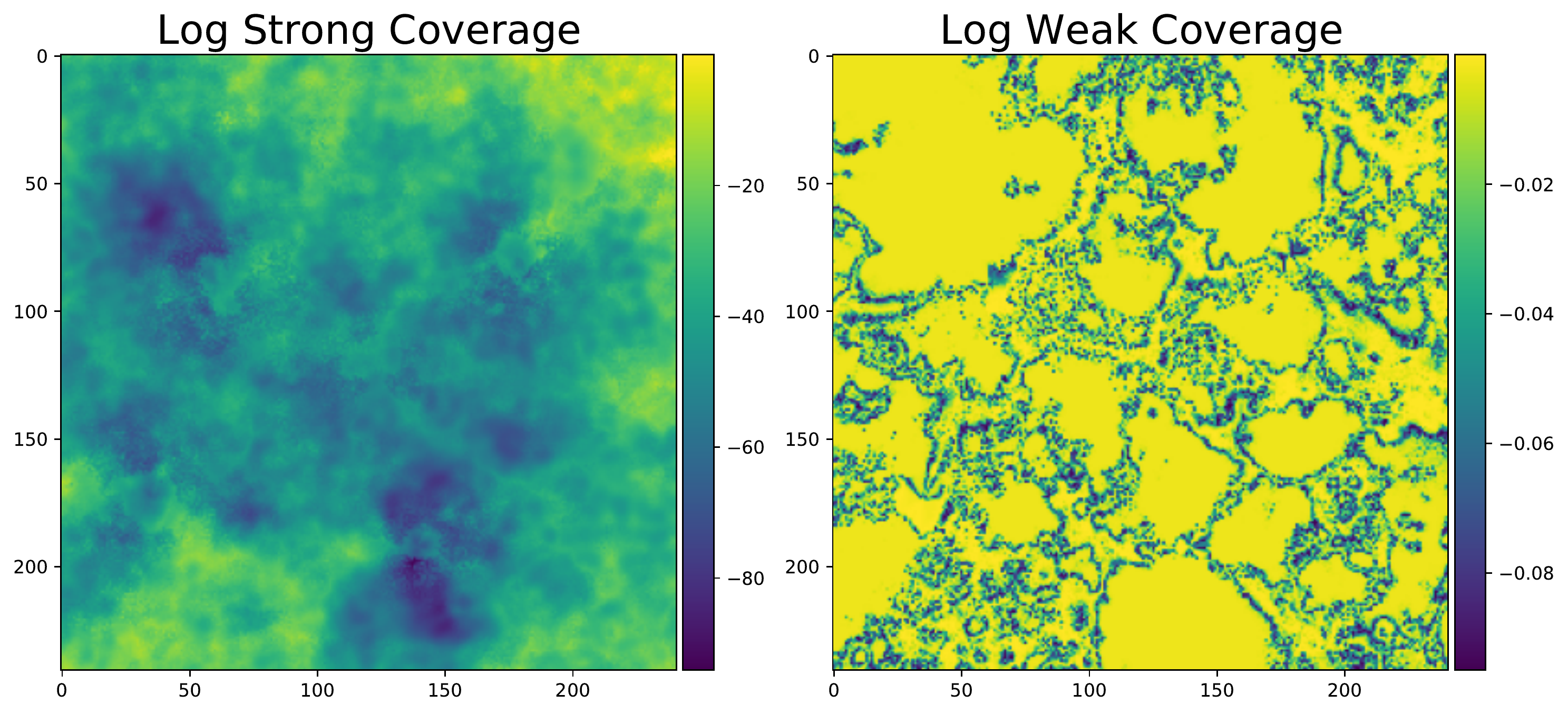}
		\caption{Coverage metrics, EI}
	\end{subfigure}
	\begin{subfigure}{0.49\textwidth}
		\includegraphics[width=\linewidth]{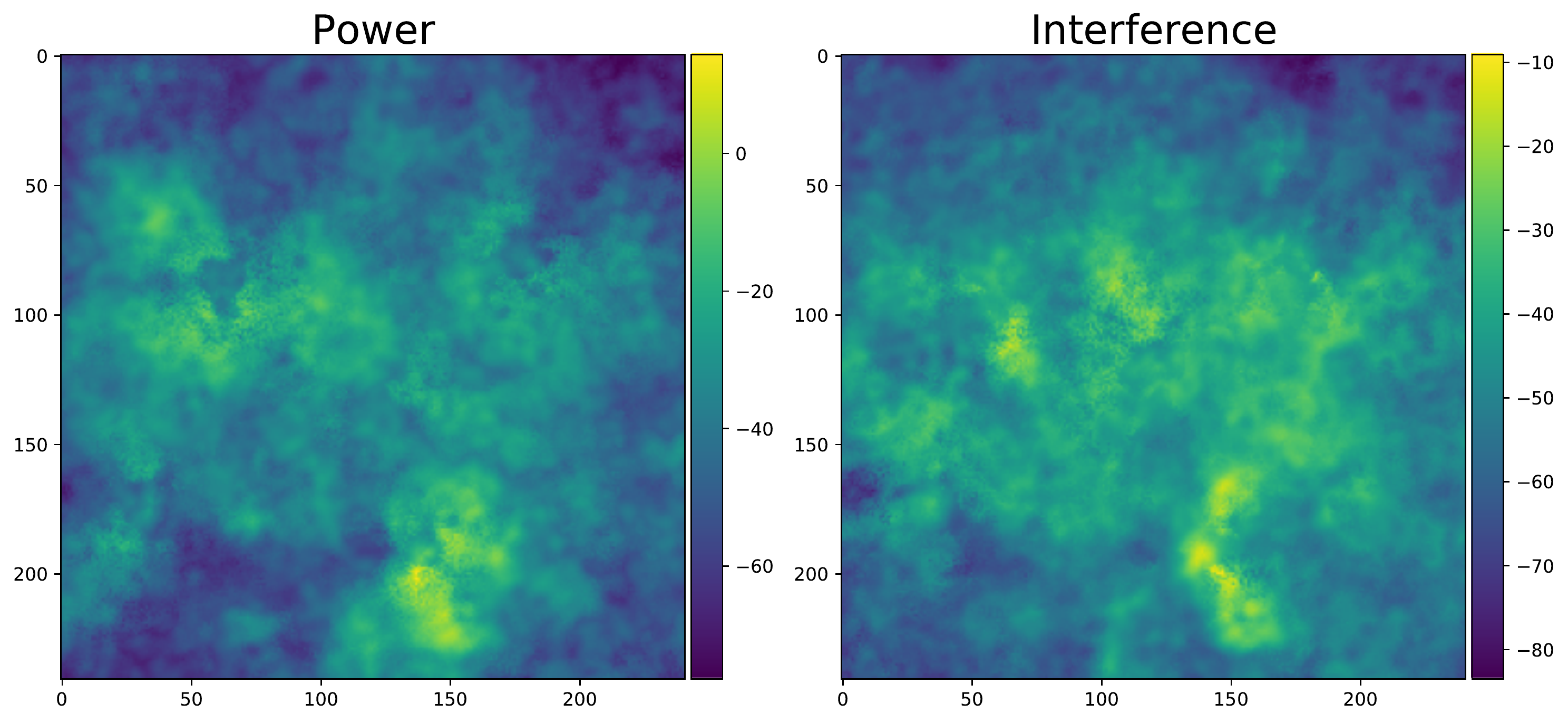}
		\caption{Coverage simulator, composite EI, HOGP + GP}
	\end{subfigure}
	\begin{subfigure}{0.49\textwidth}
		\includegraphics[width=\linewidth]{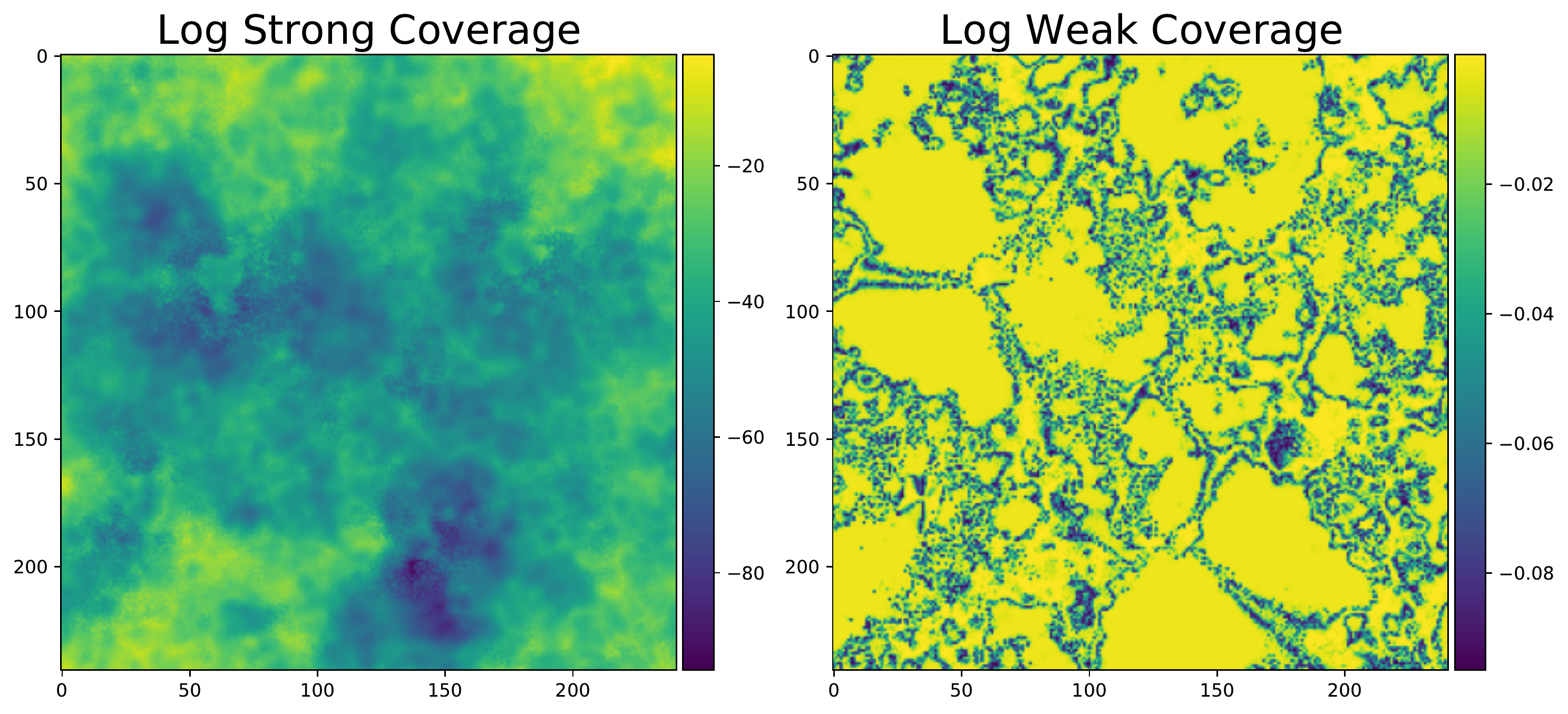}
		\caption{Coverage metrics, composite EI, HOGP + GP}
	\end{subfigure}
	\caption{\textbf{(a)} Coverage map obtained by optimizing EI on the aggregate metric, which yields the weak and strong coverage metrics (shown on log scale) \textbf{(b)}. \textbf{(c)} Coverage map obtained by optimizing EI using a composite objective with HOGP+GP, which yields the weak and strong coverage metrics (log scale) \textbf{(d)}. Composite BO with the HOGP yields distinctive differences between the best patterns found on the coverage metrics.}
	\label{fig:hogp_coverage_maps}
\end{figure}

This code was provided to us on request by the authors of \citet{dreifuerst2020optimizing}.

On the optics problem, we used the simulator of \citet{sorokin2020interferobot}, initialized with $20$ samples and ran for $115$ steps with $64$ MC samples, $64$ raw samples for initialization with a batch limit of $1.$ 
We used the same computing infrastructure as on the coverage problem above. 
to convert the problem of optimizing visibility into a BO problem rather than a reinforcement learning one, we reset the simulator to $(1e-4, 1e-4, 1e-4, 1e-4)$ each time we queried the problem and optimized the log visibility.

\begin{figure}[h!]
	\begin{subfigure}{\textwidth}
		\includegraphics[width=\linewidth]{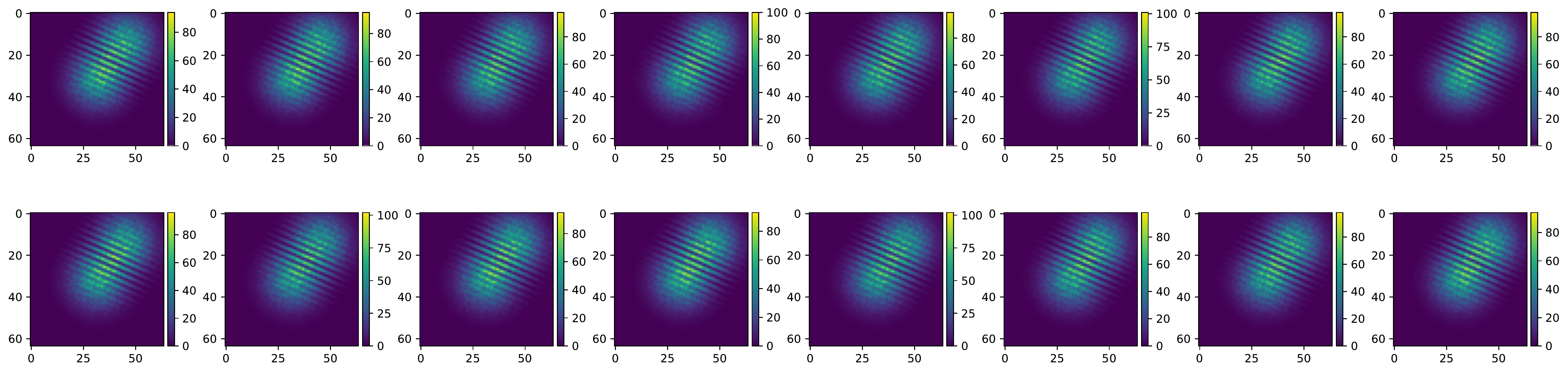}
		\caption{Random solution (visibility $\approx 3e^{-5}$)}
	\end{subfigure}
	\begin{subfigure}{\textwidth}
		\includegraphics[width=\linewidth]{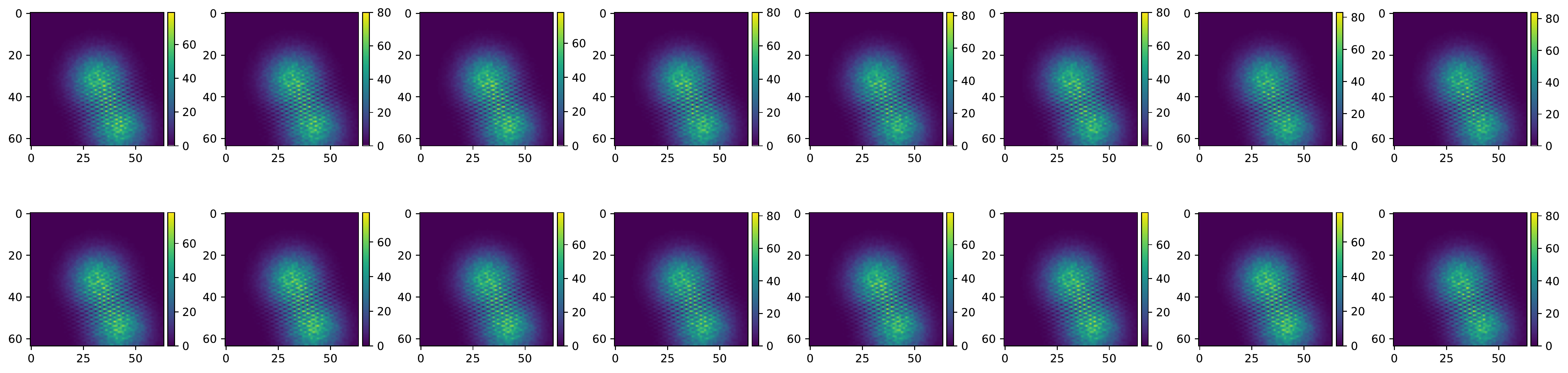}
		\caption{Solution found from EI on the metric (visibility $\approx 0.44)$)}
	\end{subfigure}
	\begin{subfigure}{\textwidth}
		\includegraphics[width=\linewidth]{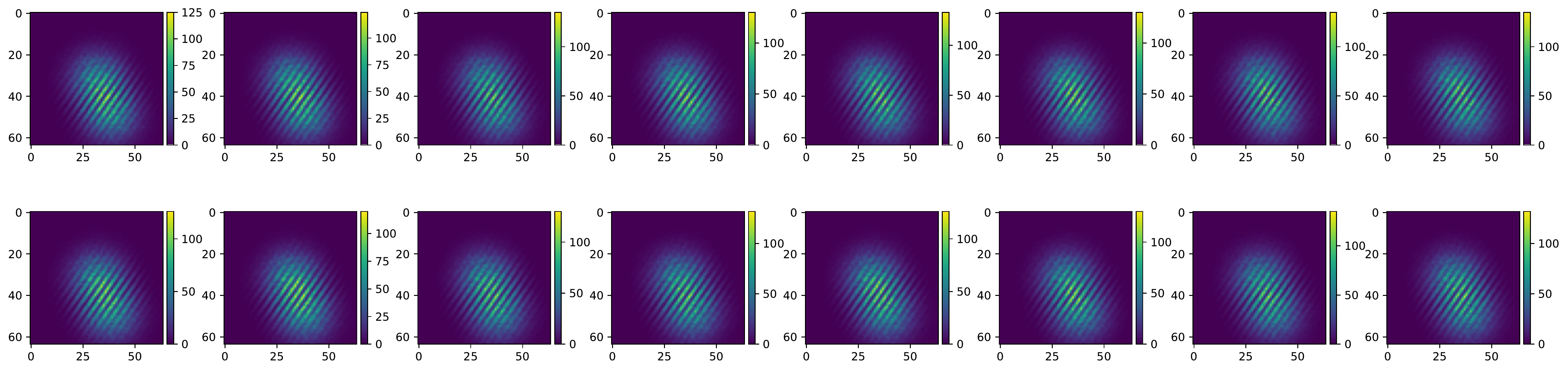}
		\caption{Solution found from composite EI with the HOGP (visibility $\approx 0.94$)}
	\end{subfigure}
	\caption{Example simulator outputs for the optics problem. \textbf{(a)} a random movement produces very unaligned lights, while \textbf{(b)} EI on the metric itself somewhat aligns the two light sources. \textbf{(c)} Running composite BO with the HOGP model produces much more aligned light sources that are presented as much brighter on the scales.}
	\label{fig:optics_outputs}
\end{figure}

To increase signal, we up-weighted the center of the image as 
\begin{align*}
	\text{Intensity}_t &:= \sum_{i,j} \exp\{-(i/64 - 0.5)^2 -(j/64 - 0.5)^2\} * I_t \\
	I_{\text{max}} &= \text{LogSumExp}(\text{Intensity}_t) \\
	I_{\text{min}} &= -\text{LogSumExp}(-\text{Intensity}_t) \\
	V &= (I_{\text{max}} - I_{\text{min}}) / (I_{\text{max}} + I_{\text{min}})
\end{align*}
and maximized the logarithm of the visibility (V), where $I_t$ is the $t$th output of the model (there are $16$ outputs, each is of shape $64 \times 64$).

We show several results from a single run in Figure \ref{fig:optics_outputs}, where we see that only EI on the HOGP is able to at least partially align the two sets of mirrors; a random solution and EI on the metric keep the light coming from the two mirrors apart.
The simulator itself comes from \url{https://github.com/dmitrySorokin/interferobotProject} (MIT License).

\end{document}